\makeatletter
\def\input@path{{packages-modified/}}
\makeatother

\pdfoutput=1

\documentclass{setup/ucbthesis}

\usepackage[round]{natbib}

\usepackage{bm}
\usepackage[labelformat=empty, position=top]{subcaption}
\usepackage[export]{adjustbox}
\usepackage{algorithm}
\usepackage{algorithmic}

\usepackage{microtype}
\usepackage{graphicx}
\usepackage{bm}
\usepackage{booktabs} 

\usepackage{hyperref}

\usepackage{caption}

\usepackage{xcolor}

\usepackage{amsmath,amsfonts,amsthm,amssymb}
\usepackage{mathtools}

\DeclareMathOperator*{\argmax}{arg\,max}
\DeclareMathOperator*{\argmin}{arg\,min}

\usepackage{babel}
\usepackage{titlesec}
\usepackage{tikz}
\usetikzlibrary{shapes.geometric}
\usetikzlibrary{arrows.meta,arrows}
\usepackage{tabularx}
\usepackage{stackengine}
\usepackage{xspace}
\usepackage{stfloats}
\usepackage[hang,flushmargin]{footmisc}
\usepackage{bbm}
\usepackage{siunitx}

\usepackage{fancyhdr-modified}
\usepackage[T1]{fontenc}
\usepackage{microtype}

\newcommand{\fig}[1]{Figure~\ref{#1}}

\newcommand{\myparagraph}[1]{{\textbf{#1}}\quad}

\DeclarePairedDelimiterX{\infdivx}[2]{(}{)}{%
  #1\;\delimsize|\delimsize|\;#2%
}

\newcommand{\method}{Diffuser\xspace}
\definecolor{cred}{HTML}{D62728}
\definecolor{cblue}{HTML}{1F77B4}
\definecolor{cgreen}{HTML}{79AB76}
\definecolor{cgrey}{rgb}{0.6,0.6,0.6}

\definecolor{gamma_red}{HTML}{D62728}
\definecolor{gamma_blue}{HTML}{1F77B4}

\definecolor{highlight}{rgb}{0,0,0}

\newcommand{\st}{\mathbf{s}_t}
\newcommand{\stp}{\mathbf{s}_{t+1}}

\newcommand{\at}{\mathbf{a}_t}
\newcommand{\rt}{\mathbf{r}_t}
\newcommand{\atp}{\mathbf{a}_{t+1}}

\newcommand{\se}{\mathbf{s}_{e}}
\renewcommand{\bs}{\mathbf{s}}
\newcommand{\ba}{\mathbf{a}}

\newcommand{\btau}[1]{\bm{\tau}^{#1}}
\newcommand{\opt}{\mathcal{O}}

\newcommand{\expect}[2]{\mathbb{E}_{#1} \left[ #2 \right] }

\newcommand{\fdiv}[2]{D_f({#1} \mid\mid {#2})}

\newcommand{\te}[1]{\texttt{#1}}

\def\vs{{\bm{s}}}

\newcommand{\bg}{\bm{g}}

\newcommand{\highlight}[1]{\textbf{#1}}

\newcommand\blfootnote[1]{%
  \begingroup
  \renewcommand\thefootnote{}\footnote{#1}%
  \addtocounter{footnote}{-1}%
  \endgroup
}

\newcommand{\ptd}[1]{p_\text{targ}({#1} \mid \st, \at)}

\newcommand{\disc}{D_\phi}

\newcommand{\gammam}{\gamma}
\newcommand{\gammav}{\tilde{\gamma}}
\newcommand{\occupancy}{\mu}
\newcommand{\boldgamma}{\boldsymbol{\gamma}}
\newcommand{\psingle}{p}

\newcommand{\deltat}{\Delta t}

\newcommand{\model}{\mu_\theta}
\newcommand{\targmodel}{\mu_{\widebar{\theta}}}

\newcommand{\grad}{\nabla}

\newcommand{\indicator}[1]{\mathbbm{1}\left[#1\right]}

\newtheorem{thm}{Theorem}
\newtheorem{prop}{Proposition}
\newtheorem{lem}{Lemma}

\newcommand{\dd}[1]{\mathrm{d}#1}

\newcommand{\widebar}[1]{\bar{#1}}

\let\ltxcup\cup

\newcommand{\chaptertitle}{}

\let\cite\citep

\setcitestyle{authoryear,round,citesep={;},aysep={,},yysep={;}}

\definecolor{mydarkblue}{rgb}{0,0.08,0.45}
\hypersetup{ %
    pdftitle={},
    pdfauthor={},
    pdfkeywords={},
    pdfborder=0 0 0,
    pdfpagemode=UseNone,
    colorlinks=true,
    linkcolor=mydarkblue,
    citecolor=mydarkblue,
    filecolor=mydarkblue,
    urlcolor=mydarkblue,
}

\begin{document}

\pagestyle{headings}


\titleformat{\chapter}[display]
    {\Huge\bfseries\centering}
    {
        \raisebox{0.175em}{\rule{0.45\linewidth}{3pt}}
        \parbox[t]{0.05\linewidth}{\centering\Huge\thechapter}
        \raisebox{0.175em}{\rule{0.46\linewidth}{3pt}}
    }
    {-1.1em}
    {}
    [\vspace{0.65em}{\titlerule[1pt]}\vspace{-.25em}]


\pagestyle{fancy}
\fancyhead{}
\fancyfoot{}



\fancyhead[C]{\small{\textsc{\textls[100]{\chaptertitle}}\hspace{1.97cm}}}
\fancyfoot[C]{\small\thepage\hspace{1.94cm}}

\renewcommand{\headrule}{}


\title{Deep Generative Models for Decision-Making and Control}
\author{Michael Janner}
\degreesemester{Spring}
\degreeyear{2023}
\degree{Doctor of Philosophy}
\chair{Professor Sergey Levine}
\othermembers{
    Professor Anca Dragan \\
    Professor Jacob Steinhardt \\
    Professor Karthik Narasimhan
}
\numberofmembers{4}
\field{Computer Science}

\maketitle

\copyrightpage

\newpage
\begin{abstract}

Deep model-based reinforcement learning methods offer a conceptually simple approach to the decision-making and control problem: use learning for the purpose of estimating an approximate dynamics model, and offload the rest of the work to classical trajectory optimization.
However, this combination has a number of empirical shortcomings, limiting the usefulness of model-based methods in practice.
The dual purpose of this thesis is to study the reasons for these shortcomings and to propose solutions for the uncovered problems.
We begin by generalizing the dynamics model itself, replacing the standard single-step formulation with a model that predicts over probabilistic latent horizons.
The resulting model, trained with a generative reinterpretation of temporal difference learning, leads to infinite-horizon variants of the procedures central to model-based control, including the model rollout and model-based value estimation.

Next, we show that poor predictive accuracy of commonly-used deep dynamics models is a major bottleneck to effective planning, and describe how to use high-capacity sequence models to overcome this limitation.
Framing reinforcement learning as sequence modeling simplifies a range of design decisions, allowing us to dispense with many of the components normally integral to reinforcement learning algorithms.
However, despite their predictive accuracy, such sequence models are limited by the search algorithms in which they are embedded.
As such, we demonstrate how to fold the entire trajectory optimization pipeline into the generative model itself, such that sampling from the model and planning with it become nearly identical.
The culmination of this endeavor is a method that improves its planning capabilities, and not just its predictive accuracy, with more data and experience.
Along the way, we highlight how inference techniques from the contemporary generative modeling toolbox, including beam search, classifier-guided sampling, and image inpainting, can be reinterpreted as viable planning strategies for reinforcement learning problems.

\end{abstract}

\newpage

\begin{frontmatter}



\tableofcontents
\clearpage
\listoffigures
\clearpage
\listoftables

\begin{acknowledgements}
    I am luckier than I deserve in having so many people to thank.

    First and foremost, I thank my advisor, Sergey Levine, for sharpening my thinking and helping me develop a research taste during the past five years.
    Sergey has a remarkable ability to always find time to chat about research regardless of how busy he is, and my work benefitted immensely from those many conversations.
    I am also grateful to my committee, Anca Dragan, Jacob Steinhardt, and Karthik Narasimhan, for their feedback and perspectives that have improved this dissertation.

    I decided to pursue a PhD in artificial intelligence due to an overwhelmingly positive research experience as an undergrad at MIT.
    Working with Josh Tenenbaum, Regina Barzilay, and Bill Freeman as well as their (at the time) students and postdocs Tejas Kulkarni, Karthik Narasimhan, Jiajun Wu, Ilker Yildirim, Pedro Tsividis, and Max Kleiman-Weiner was the best introduction to the field imaginable.

    Before starting at MIT, I spent a few formative years in Yadong Yin's materials science group at UC Riverside.
    Here I met my first research mentors: Qiao Zhang, Le He, and Mingsheng Wang.
    Looking back, I only grow more impressed at their patience and generosity.
    I am sure they had much more pressing things to attend to than showing a high school student the research ropes, but they did so anyway.
    It is because of Yadong and his students that I am a scientist today.

    Over the course of my PhD, I have been fortunate to work with many collaborators.
    Igor Mordatch, Colin Li, Yilun Du, Kevin Black, Chelsea Finn, Justin Fu, JD Co-Reyes, Rishi Veerapaneni, Katie Kang, Ilya Kostrikov, Philippe Hansen-Estruch, Zhengyao Jiang, Tianjun Zhang, Yueying Li, and Yuandong Tian have all taught me so much and expanded my research horizons.

    The Berkeley AI Research lab was a wonderful place to study.
    Aviral Kumar, Young Geng, Dibya Ghosh, Laura Smith, Philip Ball, Manan Tomar, Oleh Rybkin, Marwa Abdulhai, Kuba Grudzien, Dhruv Shah, Charlie Snell, Homer Walke, Simon Zhai, Coline Devin, Abhishek Gupta, Anusha Nagabandi, Natasha Jaques, Dinesh Jarayaman, Rowan McAllister, Vitchyr Pong, Kelvin Xu, Amy Zhang, Glen Berseth, Aurick Zhou, Avi Singh, Ashvin Nair, Allan Jabri, Sasha Sax, Vickie Ye, and the rest of the BAIR community made me look forward to coming into the lab every day.
    I am especially grateful to Michael Chang, for being my sounding board on nearly everything, and Marvin Zhang, for keeping me sane during a pandemic. 

    During the summer before my final year at Berkeley, I took a research detour and worked on reinforcement learning for program synthesis applications at Google.
    I would like to thank my host Alex Polozov, for a gracious introduction to a new subfield, as well as my residency collaborators Rishabh Singh, Charles Sutton, Abhishek Rao, Jacob Austin, David Bieber, Kensen Shi, Aitor Lewkowycz, and Vedant Misra. I also owe much to Michele Catasta and Kefan Xiao for their holistic mentorship.

    My research was generously supported by Open Philanthropy, a group of some of the most thoughtful people I have met.
    Daniel Dewey and Catherine Olsson, the early program managers of the AI Fellowship, were instrumental in pushing me to think about the larger impacts of my work.

    Finally, I would like to thank my parents, for their unwavering support from the beginning, and Emma, for more than I know how to express.




\end{acknowledgements}

\end{frontmatter}

\chapter{Introduction}
\renewcommand{\chaptertitle}{Introduction}
\label{ch:introduction}

This thesis examines one of the simplest conceivable strategies for data-driven decision-making and robotic control problems.
The abstract procedure consists of two interleaved steps:
\begin{enumerate}
\item Use data to fit a parametric model used to predict the future given the past.
\item Use the model to predict the outcomes of a candidate set of action sequences, selecting that which produces the most desirable result.
\end{enumerate}
\vspace{-.5em}
This high-level description outlines a type of model-predictive control algorithm that uses ``planning in the now" \cite{kaelbling2011hierarchical,hasselt2019parametric}, meaning that the model is used to predict into the future while making decisions as opposed to other ways of using model-generated data.
It leaves much to be specified:
how does one choose the candidate actions for evaluation?
How should the model be structured?
What constitutes useful data?

Regardless, this specification is already sufficient to suggest why it might be a good approach.
Step 1 amounts to supervised learning, which now often works reliably given enough data and high-capacity function approximators like neural networks \cite{krizhevsky2012imagenet,zhang2017understanding,kaplan2020scaling}.
In control contexts, step 2 can (in principle) be offloaded to trajectory optimization algorithms, which have been the subject of much study and are similarly well-understood in their original context when the ground-truth dynamics are known \cite{diehl2009efficient,tassa2012synthesis,matthew2017trajectory}.
It would appear that this approach combines two fairly reliable puzzle pieces.

Moreover, the separation between model-learning and decision-making has a number of appealing properties.
Most obviously, it allows for reuse of the learned model, allowing it to be deployed for a variety of tasks in the same setting.
This level of reuse is not as straightforward with model-free approaches because the reward function cannot be separated from the implicit dynamics knowledge encoded in a learned policy or value function.
This property also allows for a model to be trained from data that is not explicitly labeled with rewards, which can be useful in situations where rewards are difficult to define but experience is plentiful.
Empirically, dynamics models are found to be easier to train than value functions, allowing for better sample efficiency and generalization of learned models \cite{janner2019mbpo}; this can be viewed as a consequence of either differences between the types of algorithms used to train value functions versus dynamics models \cite{kumar2022dr3} or the simplicity of the dynamics itself relative to the optimal value function \cite{dong2020expressivity}.
Finally, this separation provides a convenient way to interpret the learned model: for any decision a planning routine produces, one can inspect the model-expected outcome that caused that decision to be selected.

Unfortunately, employing this strategy is not as straightforward as it might seem, nor do these purported benefits always translate to practice.
While there have been successful demonstrations of the combination \cite{chua2018pets,argenson2020model}, it is surprisingly difficult to extract a set of design principles from these successes that allow for the approach to be effectively applied to new problems without extensive problem-specific tuning.
As a result, the contemporary frontier of deep model-based reinforcement learning consists largely of algorithms that pull extensively from the model-\emph{free} reinforcement learning toolbox.
By contrast, conventional planning in the now with deep neural networks is rare.

This state of affairs should be surprising.
The dual purpose of the following chapters is to explain why this is the case and to suggest a way forward.
After a brief description of the problem setting and review of technical background in Chapter~\ref{ch:preliminaries}, we proceed to three primary ideas:
\begin{itemize}
    \item In Chapter~\ref{ch:gamma}, we reconsider the role of state-space prediction in reinforcement learning.
        The result is a model that amortizes the work of prediction during training time, much like a value function, as opposed to relying on model-based rollouts.
        As a result, the model can predict over infinite probabilistic horizons without sequential rollouts, blurring the line between model-based and model-free mechanisms.
        This investigation underscores a particular drawback: representing high-dimensional joint distributions over future trajectories is a difficult generative modeling problem.
    \item In Chapter~\ref{ch:transformer}, we ask whether the quality of the predictive model is the bottleneck.
        We appeal to recent successes in generative modeling and replace the conventional single-step dynamics model with a long-horizon Transformer.
        In the process, we show how to reinterpret algorithms from the sequence modeling toolbox as viable planning algorithms.
    \item Transformers largely address the predictive quality bottleneck, but are still limited by the quality of the planning routine in which they are embedded.
        In Chapter~\ref{ch:diffuser}, we discuss a way of incorporating both the prediction and the planning into a generative model, such that the line between sampling from the model and planning with it becomes blurred.
        The end result is a method that improves its planning capabilities, and not just its predictive accuracy, with more data and experience.
\end{itemize}
We conclude by discussing the lessons learned from these investigations and their implications for future model-based reinforcement learning algorithms in Chapter~\ref{ch:conclusion}.





\chapter{Preliminaries}
\renewcommand{\chaptertitle}{Preliminaries}
\label{ch:preliminaries}

\newcommand{\delt}{\Delta t}

This brief chapter introduces the problem setting studied by this thesis and defines notation.

\paragraph{The reinforcement learning problem.}
We consider an infinite-horizon Markov decision process (MDP) defined by the tuple $(\mathcal{S}, \mathcal{A}, p, r, \gamma, \rho_0)$, with state space $\mathcal{S}$ and action space $\mathcal{A}$.
The transition distribution and reward function are given by $p: \mathcal{S} \times \mathcal{A} \times \mathcal{S} \to \mathbb{R}^{+}$ and $r : \mathcal{S} \to \mathbb{R}$, respectively.
The discount is denoted by $\gamma \in [0,1)$ and the initial state distribution by $\rho_0 : \mathcal{S} \to \mathbb{R}^{+}$.
A policy $\pi : \mathcal{S} \times \mathcal{A} \to \mathbb{R}^{+}$ describes the distribution over actions taken at a particular state.
The goal of the reinforcement learning problem is to find the optimal policy $\pi^*$ that maximizes the expected sum of discounted rewards:

\begin{equation}
\label{eq:rl_objective}
\pi^*
        = \argmax_{\pi}\
        \,\mathbb{E}_{\pi}\left[\sum_{t=0}^\infty \gamma^{t} r(\st, \at)\right]
        \,.
\end{equation} 

\paragraph{The discounted occupancy.}
A policy $\pi$ induces a conditional occupancy $\occupancy^\pi(\bs \mid \st, \at)$ over future states:

\begin{equation}
\label{eq:occupancy}
    \occupancy^\pi(\bs \mid \st, \at) = (1-\gamma) \sum_{\delt=1}^{\infty} \gammam^{\delt-1} p(\mathbf{s}_{t+\delt} = \bs \mid \st, \at, \pi).
\end{equation}

The discounted occupancy is a distribution over states encountered by the policy when using a geometric weighting over future timesteps, analogous to the geometric weighting of the reinforcement learning objective in Equation~\ref{eq:rl_objective}.
Unlike the single-step transition distribution $p$, the discounted occupancy is policy-conditioned because it marginalizes over future action distributions.
When $\gamma=0$, the discounted occupancy $\mu^\pi$ becomes policy-agnostic and identical to the single-step transition distribution $p$.
For brevity, we omit the policy superscript $\pi$ in the discounted occupancy $\mu^\pi$ when it is otherwise clear from context.

The optimization objective Equation~\ref{eq:rl_objective} can be reformulated as the expected reward over the policy-conditioned discounted occupancy:
\begin{equation}
    \mathbb{E}_{\pi}\left[\sum_{t=0}^\infty \gamma^{t} r(\st, \at)\right] =
        \mathbb{E}_{
            \substack{
                \bs_0 \sim \rho_0(\cdot) \\
                \st \sim \occupancy(\cdot \mid \bs_0, \ba_0) \\
                \at \sim \pi(\cdot \mid \st)
            }
        } \left[ r(\st, \at) \right]
\end{equation}

\paragraph{Function approximation.}
In the reinforcement learning problem, we assume only the ability to interact in the environment, which provides data streams in the form of trajectories.
We do not assume query access to the functional form of, for example, the transition distribution $p$ or reward function $r$.
Instead, if these are used by an algorithm, they must be approximated from data.
We denote parametric approximations of $p$ (or $\occupancy$) as $p_\theta$ (or $\model$), in which the subscripts denote model parameters.

\paragraph{Trajectory optimization.}
Trajectory optimization \citep{witkin1988spacetime, tassa2012synthesis} refers to finding a sequence of actions $\ba_{0:T}^*$ that maximizes (or minimizes) an objective $\mathcal{J}$ factorized over per-timestep rewards (or costs) $r(\st, \at)$.

\begin{equation}
    \label{eq:trajopt_objective}
    \ba_{0:T}^* = \argmax_{\ba_{0:T}} \mathcal{J}(\bs_0, \ba_{0:T}) = \argmax_{\ba_{0:T}} \sum_{t=0}^{T} r(\st, \at)
\end{equation}

where $T$ is the planning horizon.
We use the abbreviation $\btau{} = (\bs_0, \ba_0, \bs_1, \ba_1, \ldots, \bs_T, \ba_T)$ to refer to a trajectory of interleaved states and actions and $\mathcal{J}(\btau{})$ to denote the objective value of that trajectory.

This problem statement is similar to that in Equation~\ref{eq:rl_objective}, and it is commonplace to use trajectory optimization algorithms to address problems formulated in the language of reinforcement learning.
However, there are two differences worth discussing.
The most apparent distinction is that the trajectory optimization objective in Equation~\ref{eq:trajopt_objective} considers a finite-horizon decision-making problem, though in practice the use of terminal value functions trained via reinforcement learning can allow for the consideration of infinite horizons as well.
The second distinction is that the optimization variables are of a different data type: instead of functions that output actions, they are now primitive actions themselves.
Optimizing over actions directly allows for the representation of non-Markovian policies.
For example, the solution to Equation~\ref{eq:trajopt_objective} could prescribe a different action to be taken every time a particular state is encountered; the solution to Equation~\ref{eq:rl_objective} could not without changing the definition of a policy.

\chapter{Infinite-Horizon Prediction}
\renewcommand{\chaptertitle}{Infinite-Horizon Prediction}
\label{ch:gamma}

\section{Introduction}
\label{sec:introduction}

The common ingredient in all of model-based reinforcement learning is the dynamics model: a function used for predicting future states.
The choice of the model's prediction horizon constitutes a delicate trade-off.
Shorter horizons make the prediction problem easier, as the near-term future increasingly begins to look like the present, but may not provide sufficient information for decision-making.
Longer horizons carry more information, but present a more difficult prediction problem, as errors accumulate rapidly when a model is applied to its own previous outputs
\citep{talvitie2014hallcuinated}.

Can we avoid choosing a prediction horizon altogether?
Value functions already do so by modeling the cumulative return over a discounted long-term future instead of an immediate reward, circumventing the need to commit to any single finite horizon.
However, value prediction folds two problems into one by entangling environment dynamics with reward structure, making value functions less readily adaptable to new tasks in known settings than their model-based counterparts.

In this chapter,
we propose a model that predicts over an infinite horizon with a geometrically-distributed timestep weighting (Figure~\ref{fig:teaser}).
This $\gamma$-model, named for the dependence of its probabilistic horizon on a discount factor $\gamma$, is trained with a generative analogue of temporal difference learning suitable for continuous spaces.
The $\gamma$-model bridges the gap between canonical model-based and model-free mechanisms.
Like a value function, it is policy-conditioned and contains information about the distant future; like a conventional dynamics model, it is independent of reward and may be reused for new tasks within the same environment.
The $\gamma$-model may be instantiated as both a generative adversarial network \citep{goodfellow2014gan} and a normalizing flow \citep{rezende15variational}.

The shift from standard single-step models to infinite-horizon $\gamma$-models carries several advantages:

\textbf{Constant-time prediction~~~~} Single-step models must perform an $\mathcal{O}(n)$ operation to predict $n$ steps into the future; $\gamma$-models amortize the work of predicting over extended horizons during training such that long-horizon prediction occurs with a single feedforward pass of the model.

\textbf{Generalized rollouts and value estimation~~~~}
Probabilistic prediction horizons lead to generalizations of the core procedures of model-based reinforcement learning.
For example, generalized rollouts allow for fine-grained interpolation between training-time and testing-time compounding error.
Similarly, terminal value functions appended to truncated $\gamma$-model rollouts allow for a gradual transition between model-based and model-free value estimation.

\textbf{Omission of unnecessary information~~~~}
The predictions of a $\gamma$-model do not come paired with an associated timestep. While on the surface a limitation, we show why knowing precisely {\emph{when}} a state will be encountered is not necessary for decision-making.
Infinite-horizon $\gamma$-model prediction selectively discards the unnecessary information from a standard model-based rollout.

\begin{figure}
\label{fig:teaser}
    \vspace{0.5cm}
    \begin{flushleft}
        $~\textbf{single-step model}\!: \Delta t = 1
            \hspace{3.8cm}
        \boldsymbol{\gamma}\textbf{-model}\!: \Delta t \sim \text{Geom}(1-\gamma)$
    \end{flushleft}
    \includegraphics[width=\linewidth]{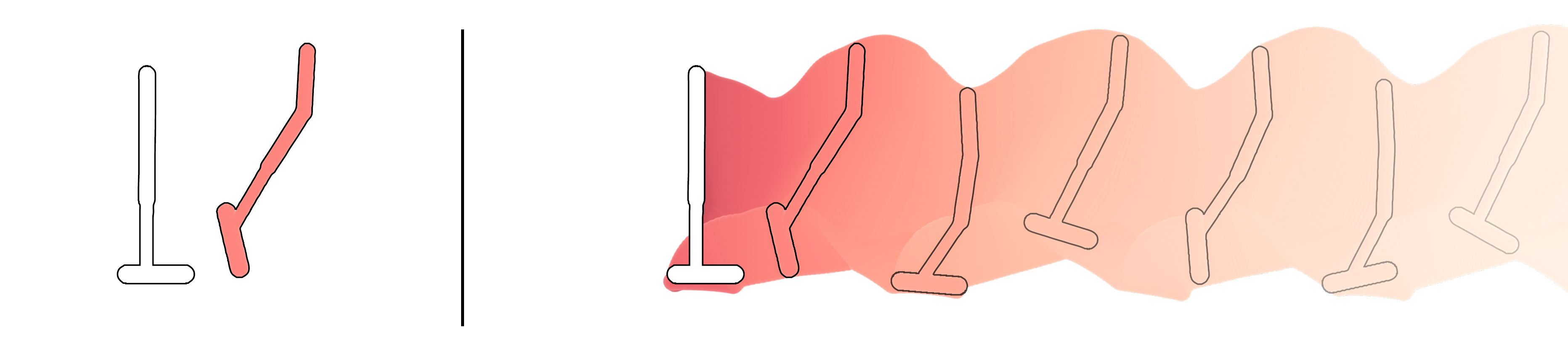}
    \vspace{-1cm}
    \begin{flushleft}
        current state \hspace{.1cm} \!\!{\color{gamma_red}prediction}
    \end{flushleft}
    \caption{
        \textbf{(Prediction with probabilistic horizons)}
        Conventional predictive models trained via maximum likelihood have a horizon of one. By interpreting temporal difference learning as a training algorithm for generative models, it is possible to predict with a probabilistic horizon governed by a geometric distribution.
        In the spirit of infinite-horizon control in model-free reinforcement learning, we refer to this formulation as infinite-horizon prediction.
    }
    \label{fig:teaser}
\end{figure}

\section{Related Work}
\label{sec:related}

The complementary strengths and weaknesses of model-based and model-free reinforcement learning have led to a number of works that attempt to combine these approaches.
Common strategies include initializing a model-free algorithm with the solution found by a model-based planner \citep{levine2013gps,farshidian2014learning,nagabandi2018mbmf},
feeding model-generated data into an otherwise model-free optimizer \citep{sutton1990dyna,silver2008dyna2,lampe2014modelnfq,kalweit2017uncertainty,luo2018algorithmic}, using model predictions to improve the quality of target values for temporal difference learning \citep{buckman2018steve,feinberg2018mve},
leveraging model gradients for backpropagation \citep{nguyen1990neural,jordan1992forward,heess2015svg}, and incorporating model-based planning without explicitly predicting future observations \citep{tamar2016vin,silver2017predictron,oh2017vpn,kahn2018gcg,amos2018differentiable,schrittwieser2019muzero}.
In contrast to combining independent model-free and model-based components, we describe a framework for training a new class of predictive model with a generative, model-based reinterpretation of model-free tools.

Temporal difference models (TDMs) \cite{pong2018tdms} provide an alternative method of training models with what are normally considered to be model-free algorithms.
TDMs interpret models as a special case of goal-conditioned value functions \citep{kaelbling1993learning,foster2002structure,schaul2015uvfa,andrychowicz2017hindsight}, though the TDM is constrained to predict at a fixed horizon and is limited to tasks for which the reward depends only on the last state.
In contrast, the $\gamma$-model predicts over a discounted infinite-horizon future and accommodates arbitrary rewards.

The most closely related prior work to $\gamma$-models is the successor representation \citep{dayan1993successor}, a formulation of long-horizon prediction that has been influential in both cognitive science \citep{momennejad2017successor,gershman2018successor} and machine learning \citep{kulkarni2016deep,ma2018universal}.
In its original form, the successor representation is tractable only in tabular domains.
Prior continuous variants have focused on policy evaluation based on expected state featurizations \citep{barreto2017successor,barreto2018transfer,hansen2020fast}, forgoing an interpretation as a probabilistic model suitable for state prediction.
Converting the tabular successor representation into a continuous generative model is non-trivial because the successor representation implicitly assumes the ability to normalize over a finite state space for interpretation as a predictive model.

Because of the discounted state occupancy's central role in reinforcement learning, its approximation by Bellman equations has been the focus of multiple lines of work.
\sloppy
Generalizations include $\beta$-models \citep{sutton1995beta}, allowing for arbitrary mixture distributions over time, and option models \citep{sutton1999semimdps}, allowing for state-dependent termination conditions.
While our focus is on generative models featuring the state-independent geometric timestep weighting of the successor representation, we are hopeful that the tools developed in this paper could also be applicable in the design of continuous analogues of these generalizations.

\section{Generative Temporal Difference Learning}

Our goal is to make long-horizon predictions without the need to repeatedly apply a single-step model.
Instead of modeling states at a particular instant in time by approximating the environment transition distribution $\psingle(\stp \mid \st, \at)$, we aim to predict a weighted distribution over all possible future states according to $\occupancy(\bs \mid \st, \at)$.
In principle, this can be posed as a conventional maximum likelihood problem:

\begin{equation*}
\max_{\theta}~ \expect{
    \st,\at,\bs \sim \mu(\cdot \mid \st, \at)
}{
    \log \model(\bs \mid \st, \at)
}.
\end{equation*}

However, doing so would require collecting samples from the occupancy $\occupancy$ independently for each policy of interest.
Forgoing the ability to re-use data from multiple policies when training dynamics models would sacrifice the sample efficiency that often makes model usage compelling in the first place, so we instead aim to design an off-policy algorithm for training $\model$.
We accomplish this by reinterpreting temporal difference learning as a method for training generative models.

Instead of collecting only on-policy samples from $\occupancy(\bs \mid \st, \at)$, we observe that $\occupancy$ admits a convenient recursive form.
Consider a modified MDP in which there is a $1-\gamma$ probability of terminating at each timestep.
The distribution over the state at termination, denoted as the exit state $\se$, corresponds to first sampling from a termination timestep $\deltat \sim \text{Geom}(1-\gamma)$ and then sampling from the per-timestep distribution $p(\bs_{t+\deltat} \mid \st, \at, \pi)$.
The distribution over $\se$ corresponds exactly to that in the definition of the occupancy $\occupancy$ in Equation~\ref{eq:occupancy}, but also lends itself to an interpretation as a mixture over only two components: the distribution at the immediate next timestep, in the event of termination, and that over all subsequent timesteps, in the event of non-termination.
This mixture yields the following target distribution:
\begin{equation}
\label{eq:gamma_targp}
    \ptd{\se} = \underset{\text{single-step distribution}}{\underbrace{(1-\gamma) p(\se \mid \st, \at)}} + \underset{\text{model bootstrap}}{\underbrace{\gamma \expect{\stp \sim p(\cdot \mid \st, \at)}{\model(\se \mid \stp)}}}.
\end{equation}
We use the shorthand $\model(\se \mid \stp) = \expect{\atp \sim \pi(\cdot \mid \stp)}{\model(\se \mid \stp, \atp)}$.
The target distribution $p_\text{targ}$ is reminiscent of a temporal difference target value: the state-action conditioned occupancy ${\model(\se \mid \st, \at)}$ acts as a $Q$-function, the state-conditioned occupancy $\model(\se \mid \stp)$ acts as a value function, and the single-step distribution $\psingle(\stp \mid \st, \at)$ acts as a reward function.
However, instead of representing a scalar target value, $p_\text{targ}$ is a distribution from which we may sample future states $\se$.
We can use this target distribution in place of samples from the true discounted occupancy $\occupancy$:
\begin{equation*}
\max_{\theta}~ \expect{
    \st, \at, \se \sim (1-\gamma)p(\cdot \mid \st, \at) + \gamma \expect{}{\model(\cdot \mid \stp)}
}{
    \log \model(\se \mid \st, \at)
}.
\end{equation*}

This formulation differs from a standard maximum likelihood learning problem in that the target distribution depends on the current model.
By bootstrapping the target distribution in this manner, we are able to use only empirical $(\st, \at, \stp)$ transitions from one policy in order to train an infinite-horizon predictive model $\model$ for any other policy.
Because the horizon is governed by the discount $\gamma$, we refer to such a model as a $\gamma$-model.

This bootstrapped model training may be incorporated into a number of different generative modeling frameworks.
We discuss two cases here.
(1) When the model $\model$ permits only sampling, we may train $\model$ by minimizing an $f$-divergence from samples:
\begin{equation}
\label{eq:objective_fdiv}
\mathcal{L}_1(\st, \at, \stp) = \fdiv{\model(\cdot \mid \st, \at)}{(1-\gamma)p(\cdot \mid \st, \at) + \gamma \model(\cdot \mid \stp)}.
\end{equation}
This objective leads naturally to an adversarially-trained $\gamma$-model.
(2)~When the model $\model$ permits density evaluation, we may minimize an error defined on log-densities directly:
\begin{equation}
\label{eq:objective_logp}
    \mathcal{L}_2(\st, \at, \stp) = \mathbb{E}_{\se} \Big[
    \big\lVert \log \model(\se \mid \st, \at) - \log\big( 
    (1-\gamma)p(\se \mid \st, \at) + \gamma \model(\se \mid \stp)
    \big)
    \big\lVert_2^2 \Big].
\end{equation}
This objective is suitable for $\gamma$-models instantiated as normalizing flows.
Due to the approximation of a target log-density $\log\left((1-\gamma) p(\cdot \mid \st, \at) + \gamma \expect{\stp}{\model(\cdot \mid \stp)}\right)$ using a single next state $\stp$, $\mathcal{L}_2$ is unbiased for deterministic dynamics and a bound in the case of stochastic dynamics.
We provide complete algorithmic descriptions of both variants and highlight practical training considerations in Section~\ref{sec:practical}.

\section{Analysis and Applications of \texorpdfstring{$\boldsymbol{\gammam}$}{Gamma}-Models}

Using the $\gamma$-model for prediction and control requires us to generalize procedures common in model-based reinforcement learning. In this section, we derive the $\gamma$-model rollout and show how it can be incorporated into a reinforcement learning procedure that hybridizes model-based and model-free value estimation. First, however, we show that the $\gamma$-model is a continuous, generative counterpart to another type of long-horizon model: the successor representation.

\subsection{\texorpdfstring{$\boldsymbol{\gammam}$}{Gamma}-Models as a Continuous Successor Representation}

The successor representation $M$ is a prediction of expected visitation counts \citep{dayan1993successor}. It has a recurrence relation making it amenable to tabular temporal difference algorithms: 
\begin{equation}
\label{eq:sr}
    M(\se \mid \st, \at) = \expect{\stp \sim p(\cdot \mid \st, \at)}{\indicator{\se = \stp} + \gamma M(\se \mid \stp)}.
\end{equation}
Adapting the successor representation to continuous state spaces in a way that retains an interpretation as a probabilistic model has proven challenging.
However, variants that forego the ability to sample in favor of estimating expected state features have been developed \citep{barreto2017successor}.

The form of the successor recurrence relation bears a striking resemblance to that of the target distribution in Equation~\ref{eq:gamma_targp}, suggesting a connection between the generative, continuous $\gamma$-model and the discriminative, tabular successor representation.
We now make this connection precise.

\begin{prop}
\label{thm:equivalence}
The global minimum of both $\mathcal{L}_1$ and $\mathcal{L}_2$ is achieved if and only if the resulting $\gamma$-model produces samples according to the normalized successor representation:
\begin{equation*}
    \model(\se \mid \st, \at) = (1-\gamma) M(\se \mid \st, \at).
\end{equation*}
\end{prop}

\begin{proof}
In the case of either objective, the global minimum is achieved only when
\begin{equation*}
    \model(\se \mid \st, \at) = (1-\gammam) p(\se \mid \st, \at) + \gammam \expect{\stp \sim p(\cdot \mid \st,\at)}{\model(\se \mid \stp)}
\end{equation*}
for all $\st, \at$.
We recognize this optimality condition exactly as the recurrence defining the successor representation $M$ (Equation~\ref{eq:sr}), scaled by $(1-\gamma)$ such that $\model$ integrates to $1$ over $\se$.
\end{proof}

\subsection{\texorpdfstring{$\boldsymbol{\gammam}$}{Gamma}-Model Rollouts}
\label{sec:gamma_rollouts}

Standard single-step models, which correspond to $\gamma$-models with $\gamma=0$, can predict multiple steps into the future by making iterated autoregressive predictions, conditioning each step on their own output from the previous step.
These sequential rollouts form the foundation of most model-based reinforcement learning algorithms.
We now generalize these rollouts to $\gamma$-models for $\gamma > 0$, allowing us to decouple the discount used during model training from the desired horizon in control.
When working with multiple discount factors, we explicitly condition an occupancy on its discount as $\occupancy(\se \mid \st; \gamma)$.
In the results below, we omit the model parameterization $\theta$ whenever a statement applies to both a discounted occupancy $\occupancy$ and a parametric $\gamma$-model $\model$.

\begin{thm}
\label{thm:weights}
Let $\occupancy_{n}(\se \mid \st; \gammam)$ denote the distribution over states at the $n^{\text{th}}$ sequential step of a $\gammam$-model rollout beginning from state $\st$. For any desired discount $\gammav \in [\gammam, 1)$, we may reweight the samples from these model rollouts according to the weights

\begin{equation*}
    \alpha_n = \frac{(1-\gammav)(\gammav - \gammam)^{n-1}}{(1-\gammam)^n}
\end{equation*}

to obtain the state distribution drawn from $\occupancy_{1}(\se \mid \st; \gammav) = \occupancy(\se \mid \st; \gammav)$.
That is, we may reweight the steps of a $\gammam$-model rollout so as to match the distribution of a $\gammav$-model with larger discount:

\begin{equation*}
    \occupancy(\se \mid \st; \gammav) = \sum_{n=1}^{\infty} \alpha_n \occupancy_{n}(\se \mid \st; \gamma). \\
\end{equation*}
\end{thm}
\begin{proof}
    Each step of the $\gamma$-model samples a time according to $\Delta t \sim \text{Geom}(1-\gamma)$, so the time after $n$ $\gamma$-model steps is distributed according to the sum of $n$ independent geometric random variables with identical parameters.
    This sum corresponds to a negative binomial random variable, $\text{NB}(n, 1-\gamma)$, with the following pmf:

    \begin{equation}
    \label{eq:nb_pmf}
    p_n(t) = {t-1 \choose t-n} \gammam^{(t-n)} (1-\gammam)^n.
    \end{equation}

    Equation~\ref{eq:nb_pmf} is mildly different from the textbook pmf because we want a distribution over the total number of trials (in our case, cumulative timesteps $t$) instead of the number of successes before the $n^{\text{th}}$ failure.
    The latter is more commonly used because it gives the random variable the same support, $t \ge 0$, for all $n$. The form in Equation~\ref{eq:nb_pmf} only has support for $t \ge n$, which substantially simplifies the following analysis.

    The distributions $q(t)$ expressible as a mixture over the per-timestep negative binomial distributions $p_n$ are given by:
    \begin{equation*}
        q(t) = \sum_{n=1}^{t} \alpha_n p_n(t),
    \end{equation*}
    in which $\alpha_n$ are the mixture weights. Because $p_n$ only has support for $t \ge n$, it suffices to only consider the first $t$ $\gamma$-model steps when solving for $q(t)$.

    We are interested in the scenario in which $q(t)$ is also a geometric random variable with smaller parameter, corresponding to a larger discount $\gammav$.
    We proceed by setting $q(t)=\text{Geom}(1-\gammav)$ and solving for the mixture weights $\alpha_n$ by induction.

    \paragraph{Base case.}
    Let $n=1$. Because $p_1$ is the only mixture component with support at $t=1$, $\alpha_1$ is determined by $q(1)$:
    \begin{align*}
        1-\gammav &= \alpha_1 {t-1 \choose t-1} \gammam^{t-1} (1-\gammam)^t \\
        &= \alpha_1 (1-\gammam).
    \end{align*}
    Solving for $\alpha_1$ gives:
    \begin{align*}
        \alpha_1 = \frac{1-\gammav}{1-\gammam}.
    \end{align*}

    \paragraph{Induction step.}
    We now assume the form of $\alpha_k$ for $k = 1, \ldots, n-1$ and solve for $\alpha_n$ using $q(n)$.
    \begin{align*}
        (1-\gammav) \gammav^{n-1} &= \sum_{k=1}^{n} \alpha_k {n-1 \choose n-k} \gammam^{n-k} (1-\gammam)^{k} \\
        &= \left\{ \sum_{k=1}^{n-1} \frac{(1-\gammav)(\gammav-\gammam)^{k-1}}{(1-\gammam)^k} {n-1 \choose n-k} \gammam^{n-k} (1-\gamma)^k \right\}
        + \alpha_n (1-\gammam)^n \\
        &= (1-\gammav) \left\{\sum_{k=1}^{n-1}  {n-1 \choose n-k} (\gammav-\gammam)^{k-1} \gammam^{n-k} \right\}
        + \alpha_n (1-\gammam)^n \\
        &= (1-\gammav) \left\{\sum_{k=1}^{n}  {n-1 \choose n-k} (\gammav-\gammam)^{k-1} \gammam^{n-k} \right\} - (1-\gammav) (\gammav - \gammam)^{n-1}
        + \alpha_n (1-\gammam)^n \\
        &= (1-\gammav) \gammav^{n-1} - (1-\gammav) (\gammav - \gammam)^{n-1}
        + \alpha_n (1-\gammam)^n \\
    \end{align*}
    Solving for $\alpha_n$ gives
    \begin{align*}
        \alpha_n &= \frac{(1-\gammav)(\gammav-\gammam)^{n-1}}{(1-\gammam)^n}
    \end{align*}
    as desired.
\end{proof}


This reweighting scheme has two special cases of interest.
A standard single-step model, with $\gammam=0$, yields $\alpha_n = (1-\gammav) \gammav^{n-1}$. These weights are familiar from the definition of the discounted state occupancy in terms of a per-timestep mixture (Equation~\ref{eq:occupancy}).
Setting $\gammam=\gammav$ yields $\alpha_n = 0^{n-1}$, or a weight of 1 on the first step and $0$ on all subsequent steps.\footnote{We define $0^0$ as $\lim_{x \to 0} x^x = 1$.} This result is also expected: when the model discount matches the target discount, only a single forward pass of the model is required.

\begin{figure}
\centering
\includegraphics[width=0.495\linewidth]{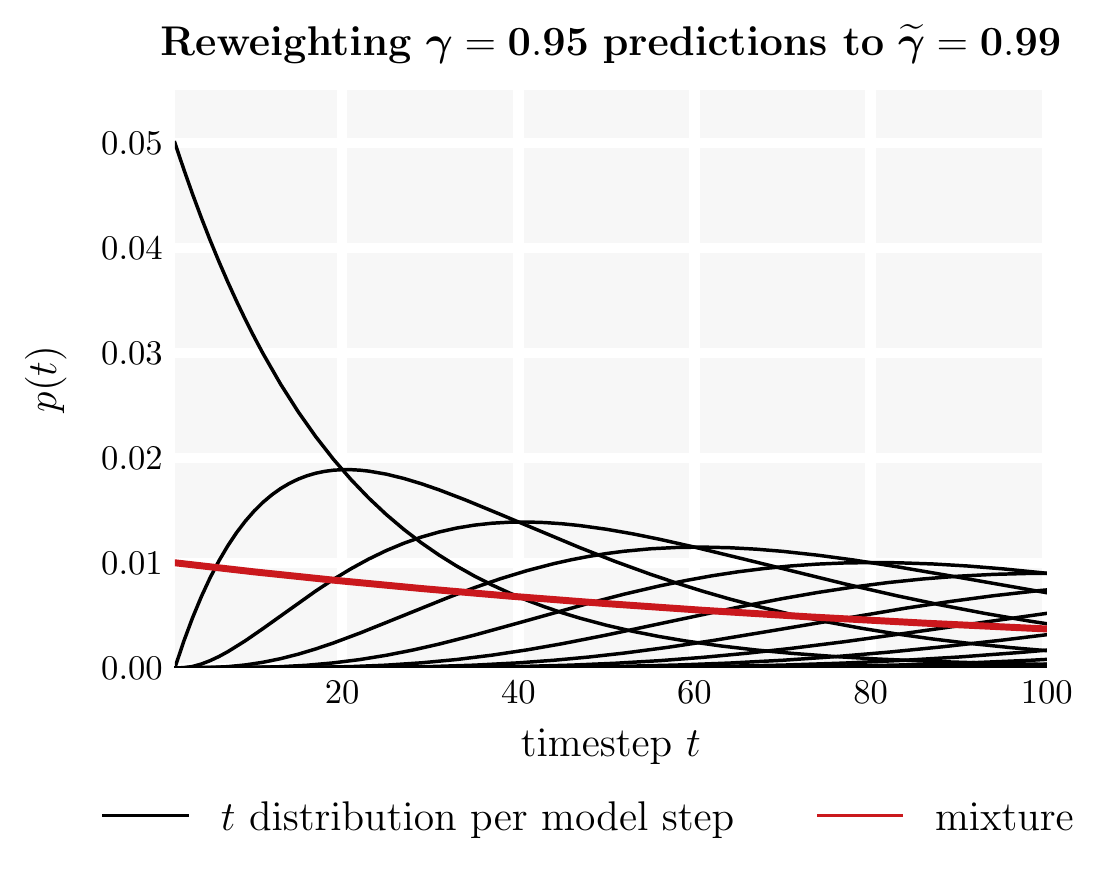}
\includegraphics[width=0.495\linewidth]{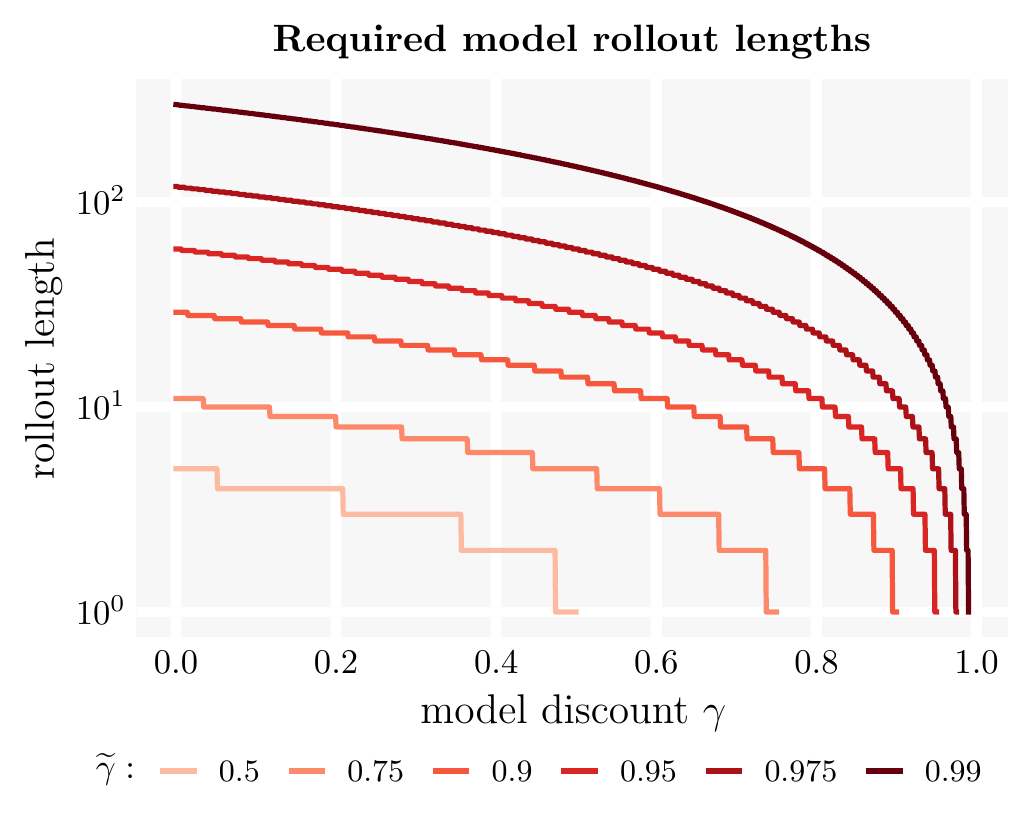} \\
\vspace{-6.46cm}
\footnotesize
\hspace{-0.1cm} \textbf{(a)}
\hspace{7.5cm} \textbf{(b)}
\hspace{7.5cm}~ \\
\normalsize
\vspace{5.8cm}
\caption{
    \textbf{(Rollouts with probabilistic horizons)}
    \textbf{(a)} The first step from a $\gamma$-model samples states at timesteps distributed according to a geometric distribution with parameter $1-\gamma$; all subsequent steps have a negative binomial timestep distribution stemming from the sum of independent geometric random variables.
    When these steps are reweighted according to Theorem~\ref{thm:weights}, the resulting distribution follows a geometric distribution with smaller parameter (corresponding to a larger discount value $\gammav$).
    \textbf{(b)} The number of steps needed to recover $95\%$ of the probability mass from distributions induced by various target discounts $\gammav$ for all valid model discounts $\gammam$.
    When using a standard single-step model, corresponding to the case of $\gammam=0$, a $299$-step model rollout is required to reweight to a discount of $\gammav=0.99$. 
}
\label{fig:sr_composition}
\end{figure}

Figure~\ref{fig:sr_composition} visually depicts the reweighting scheme and the number of steps required for truncated model rollouts to approximate the distribution induced by a larger discount.
There is a natural tradeoff with $\gamma$-models: the higher $\gammam$ is, the fewer model steps are needed to make long-horizon predictions, reducing model-based compounding prediction errors \citep{asadi2019combating,janner2019mbpo}.
However, increasing $\gammam$ transforms what would normally be a standard maximum likelihood problem (in the case of single-step models) into one resembling approximate dynamic programming (with a model bootstrap), leading to model-free bootstrap error accumulation \citep{kumar2019stabilizing}.
The primary distinction is whether this accumulation occurs during training, when the work of sampling from the occupancy $\occupancy$
is being amortized, or during ``testing'', when the model is being used for rollouts.
While this horizon-based error compounding cannot be eliminated entirely, $\gamma$-models allow for a continuous interpolation between the two extremes.

\vspace{0.2cm}
\subsection{\texorpdfstring{$\boldsymbol{\gammam}$}{Gamma}-Model-Based Value Expansion}
\label{sec:gamma_mve}
\vspace{0.5cm}

We now turn our attention from prediction with $\gamma$-models to value estimation for control.
In tabular domains, the state-action value function can be decomposed as the inner product between the successor representation $M$
and the vector of per-state rewards \citep{gershman2018successor}.
Taking care to account for the normalization from the equivalence in Proposition~\ref{thm:equivalence}, we can similarly estimate the $Q$ function as the expectation of reward under states sampled from the $\gamma$-model:






\begin{align}
    Q(\st, \at; \gamma) &= \sum_{\deltat = 1}^{\infty} \gamma^{\deltat - 1} \int_{\mathcal{S}} r(\se) p(\bs_{t+\deltat} = \se \mid \st, \at, \pi) \dd{\se} \nonumber \\
    &= \int_\mathcal{S} r(\se) \sum_{\deltat = 1}^{\infty} \gamma^{\deltat - 1} p(\bs_{t+\deltat} = \se \mid \st, \at, \pi) \dd{\se} \nonumber \\
    &= \frac{1}{1-\gamma} \expect{\se \sim \occupancy(\cdot \mid \st, \at; \gamma)}{r(\se)}
\label{eq:q_estimation}
\end{align}
 
This relation suggests a model-based reinforcement learning algorithm in which $Q$-values are estimated by a $\gamma$-model without the need for sequential model-based rollouts.
However, in some cases it may be practically difficult to train a generative $\gamma$-model with discount as large as that of a discriminative $Q$-function.
While one option is to chain together $\gamma$-model steps as in Section~\ref{sec:gamma_rollouts}, an alternative solution often effective with single-step models is to combine short-term value estimates from a truncated model rollout with a terminal model-free value prediction:

\begin{equation*}
\label{eq:mve}
    V_\text{MVE}(\st; \gammav) = \sum_{n=1}^{H} \gammav^{n - 1} r(\bs_{t+n}) + \gammav^{H} V(\bs_{t+H}; \gammav).
\end{equation*}

This hybrid estimator is referred to as a model-based value expansion (MVE; \citealt{feinberg2018mve}).
There is a hard transition between the model-based and model-free value estimation in MVE, occuring at the model horizon $H$.
We may replace the single-step model with a $\gamma$-model for a similar estimator in which there is a probabilistic prediction horizon, and as a result a gradual transition.

\vspace{.2cm}
{
\newcommand*\circled[1]{\tikz[baseline=(char.base)]{
            \node[shape=circle,draw,inner sep=2pt] (char) {#1};}}
\begin{thm}
\label{thm:gamma_mve}
For $\gammav > \gammam$, $V(\st; \gammav)$ may be decomposed as a weighted average of $H$ $\gammam$-model steps and a terminal value estimation. We denote this the $\boldgamma$\textbf{-MVE estimator}:


\begin{equation*}
    \hat{V}_{\gamma\text{\normalfont-MVE}}(\st; \gammav) = \frac{1}{1-\gammav} \sum_{n=1}^{H} \alpha_n \expect{\se \sim \occupancy_{n}(\cdot \mid \st; \gamma)}{r(\se)}
    + \left(\frac{\gammav - \gammam}{1-\gammam}\right)^{\!H} \expect{\se \sim \occupancy_{H}(\cdot \mid \st; \gamma)}{V(\se; \gammav)}.    
\end{equation*}
\end{thm}
\begin{proof}
\begin{align}
V(\st; \gammav) &= \frac{1}{1-\gammav} \expect{\se \sim \occupancy(\cdot \mid \st; \gammav)}{r(\se)} \nonumber \\
&= \frac{1}{1-\gammav} \sum_{n=1}^{\infty} \alpha_n \expect{\se \sim \occupancy_{n}(\cdot \mid \st; \gamma)}{r(\se)} \nonumber \\
&= \frac{1}{1-\gammav}
\underset{\circled{1}}{\underbrace{\sum_{n=1}^{H} \alpha_n \expect{\se \sim \occupancy_{n}(\cdot \mid \st; \gamma)}{r(\se)}}} 
+ \frac{1}{1-\gammav}
\underset{\circled{2}}{\underbrace{\sum_{n=H+1}^{\infty} \alpha_n \expect{\se \sim \occupancy_{n}(\cdot \mid \st; \gamma)}{r(\se)}}}
. \label{eq:gamma_mve_decomp}
\end{align}
The second equality rewrites an expectation over a $\gammav$-model as an expectation over a rollout of a $\gamma$-model using step weights $\alpha_n$ from Theorem~\ref{thm:weights}.
We recognize $\circled{1}$ as the model-based component of the value estimation in $\gamma$-MVE.
All that remains is to write $\circled{2}$ using a terminal value function.
\begin{align}
\sum_{n=H+1}^{\infty} \alpha_n \expect{\se \sim \occupancy_{n}(\cdot \mid \st; \gamma)}{r(\se)}
&= \sum_{n=1}^{\infty} \alpha_{H+n} \expect{\se \sim \occupancy_{H+n}(\cdot \mid \st; \gammam)}{r(\se)} \nonumber \\
&= \left( \frac{\gammav-\gammam}{1-\gammam} \right)^{\!H} \expect{\bs_{H} \sim \occupancy_{H}(\cdot \mid \st; \gammam)}{\sum_{n=1}^{\infty} \alpha_{n} \expect{\se \sim \occupancy_{n}(\cdot \mid \bs_{H}; \gammam)}{r(\se)}} \nonumber \\
&= \left( \frac{\gammav-\gammam}{1-\gammam} \right)^{\!H} \expect{\bs_{H} \sim \occupancy_{H}(\cdot \mid \st; \gammam)}{\expect{\se \sim \occupancy(\cdot \mid \bs_{H}; \gammav)}{r(\se)}} \nonumber \\
&= (1-\gammav) \left( \frac{\gammav-\gammam}{1-\gammam} \right)^{\!H} \expect{\se \sim \occupancy_{H}(\cdot \mid \st; \gammam)}{V(\se; \gammav)} \label{eq:gamma_mve_value}
\end{align}
The second equality uses $\alpha_{H+n} = \left( \frac{\gammav-\gammam}{1-\gammam} \right)^{\!H} \alpha_n$ and the time-invariance of $G^{(n)}$ with respect to its conditioning state. Plugging Equation~\ref{eq:gamma_mve_value} into Equation~\ref{eq:gamma_mve_decomp} gives:
\begin{align*}
V(\st; \gammav) &=
\frac{1}{1-\gammav} \sum_{n=1}^{H} \alpha_n \expect{\se \sim \occupancy_{n}(\cdot \mid \st; \gamma)}{r(\se)}
+ \left( \frac{\gammav-\gammam}{1-\gammam} \right)^{\!H} \expect{\se \sim \occupancy_{H}(\cdot \mid \st; \gammam)}{V(\se; \gammav)}.
\end{align*}
\end{proof}

\paragraph{Remark 1.}
Using Lemma~\ref{lem:gamma_mve_weights} from Appendix~\ref{app:geometric_lemma} to substitute $1 - \sum_{n=1}^{H}\alpha_n$ in place of $\left( \frac{\gammav-\gammam}{1-\gammam} \right)^{\!H}$ clarifies the interpretation of $V(\st; \gammav)$ as a weighted average over $H$ $\gamma$-model steps and a terminal value function.
Because the mixture weights must sum to 1, it is unsurprising that the weight on the terminal value function turned out to be $\left( \frac{\gammav-\gammam}{1-\gammam} \right)^{\!H} = 1-\sum_{n=1}^{H} \alpha_n$.

\paragraph{Remark 2.}
Setting $\gamma=0$ recovers standard MVE with a single-step model, as the weights on the model steps simplify to $\alpha_n = (1-\gammav)(\gammav-\gammam)^{n-1}$ and the weight on the terminal value function simplifies to $\gammav^H$.
}

{
\setlength{\textfloatsep}{10pt}
\newcommand*{\medcup}{\mathbin{\scalebox{1.5}{\ensuremath{\cup}}}}%

\begin{algorithm}[t]
\caption{~~\textbf{$\boldsymbol{\gamma}$-model based value expansion}}
\label{alg:gamma_mve}
\begin{algorithmic}[1]
\STATE \textbf{Input} $\gammam$: model discount, $\gammav$: value discount, $\lambda:$ step size
\STATE \textbf{Initialize} $\model:$ $\gamma$-model generator
\STATE \textbf{Initialize} $Q_\omega:$ $Q$-function, $V_\xi:$ value function, $\pi_\psi:$ policy, $\mathcal{D}:$ replay buffer
\FOR{each iteration}
\FOR{each environment step}
    \STATE $\at \sim \pi_\psi(\cdot \mid \st)$
    \STATE $\stp \sim p(\cdot \mid \st, \at)$
    \STATE $\rt = r(\st, \at)$
    \STATE $\mathcal{D} \gets \mathcal{D} \ltxcup \left\{ \st, \at, \rt, \stp \right\}$
\ENDFOR
\FOR{each gradient step}
    \STATE Sample transitions $(\st, \at, \rt, \stp)$ from $\mathcal{D}$
    \STATE Update $\model$ to Algorithm~\ref{alg:practical_samples}~or~\ref{alg:practical_logp} 
    \STATE Compute $V_{\gamma-\text{MVE}}(\stp)$ according to Theorem~\ref{thm:gamma_mve}
    \STATE Update $Q$-function parameters: \\
    \hspace{1cm} $\omega \gets \omega - \lambda \grad_\omega \frac{1}{2} \left(Q_\omega(\st, \at) - \left(\rt + \gammav V_{\gamma-\text{MVE}}(\stp) \right)\right)^2$
    \STATE Update value function parameters: \\
    \hspace{1cm} $\xi \gets \xi - \lambda \grad_\xi \frac{1}{2} \left( V_\xi(\st) - \expect{\ba \sim \pi_\psi(\cdot \mid \st)}{Q_\omega(\st, \ba) - \log \pi_\psi(\ba \mid \st)} \right)^2$
    \STATE Update policy parameters: \\
    \hspace{1cm} $\psi \gets \psi - \lambda \grad_\psi \expect{\ba \sim \pi_\psi(\cdot \mid \st)}{\log \pi_\psi(\ba \mid \st) - Q_\omega(\st, \ba)}$
\ENDFOR
\ENDFOR
\end{algorithmic}
\end{algorithm}
}







The $\gamma$-MVE estimator allows us to perform $\gamma$-model-based rollouts with horizon $H$, reweight the samples from this rollout by solving for weights $\alpha_n$ given a desired discount $\gammav > \gammam$, and correct for the truncation error stemming from the finite rollout length using a terminal value function with discount $\gammav$.
As expected, MVE is a special case of $\gamma$-MVE, as can be verified by considering the weights corresponding to $\gamma=0$ described in Section~\ref{sec:gamma_rollouts}.
This estimator, along with the simpler value estimation in Equation~\ref{eq:q_estimation}, highlights the fact that it is not necessary to have timesteps associated with states in order to use predictions for decision-making.
Pseudocode for an actor-critic algorithm using the $\gamma$-MVE estimator is provided Algorithm~\ref{alg:gamma_mve}.

\section{Practical Training of \texorpdfstring{$\boldsymbol{\gammam}$}{Gamma}-Models}
\label{sec:practical}

\sloppy
Because $\gamma$-model training differs from standard dynamics modeling primarily in the bootstrapped target distribution and not in the model parameterization, $\gamma$-models are in principle compatible with any generative modeling framework.
We focus on two representative scenarios, differing in whether the generative model class used to instantiate the $\gamma$-model allows for tractable density evaluation.

\paragraph{{\color{gamma_red}Training without density evaluation.}}
When the $\gamma$-model parameterization does not allow for tractable density evaluation,
we minimize a bootstrapped $f$-divergence according to $\mathcal{L}_1$ (Equation~\ref{eq:objective_fdiv}) using only samples from the model.
The generative adversarial framework provides a convenient way to train a parametric generator by minimizing an $f$-divergence of choice given only samples from a target distribution $p_\text{targ}$ and the ability to sample from the generator \citep{goodfellow2014gan,nowozin2016fgan}.
In the case of bootstrapped maximum likelihood problems, our target distribution is induced by the model itself (alongside a single-step transition distribution), meaning that we only need sample access to our $\gamma$-model in order to train $\model$ as a generative adversarial network (GAN).

Introducing an auxiliary discriminator $\disc$ and selecting the Jensen-Shannon divergence as our $f$-divergence, we can reformulate minimization of the original objective $\mathcal{L}_1$ as a saddle-point \nobreak{optimization} over the following objective:

\begin{equation*}
    \hat{\mathcal{L}}_1(\st,\at)
    = \expect{\se^{+} \sim p_\text{targ}(\cdot \mid \st, \at)}{\log \disc(\se^{+} \mid \st, \at)}
    + \expect{\se^{-} \sim \model(\cdot \mid \st, \at)}{\log(1-\disc(\se^{-} \mid \st, \at))},
\end{equation*}
which is minimized over $\model$ and maximized over $\disc$.
As in $\mathcal{L}_1$, $p_\text{targ}$ refers to the bootstrapped target distribution in Equation~\ref{eq:gamma_targp}.
In this formulation, $\model$ produces samples by virtue of a deterministic mapping of a random input vector $\mathbf{z} \sim \mathcal{N}(0,I)$ and conditioning information $(\st, \at)$.
Other choices of $f$-divergence may be instantiated by different choices of activation function \citep{nowozin2016fgan}.

\paragraph{{\color{gamma_blue}Training with density evaluation.}}
When the $\gamma$-model permits density evaluation, we may bypass saddle point approximations to an $f$-divergence and directly regress to target density values, as in objective $\mathcal{L}_2$ (Equation~\ref{eq:objective_logp}).
This is a natural choice when the $\gamma$-model is instantiated as a conditional normalizing flow \citep{rezende15variational}.
Evaluating target values of the form
\begin{equation*}
    T(\st, \at, \stp, \se) = \log\big( 
    (1-\gamma)p(\se \mid \st, \at) + \gamma \model(\se \mid \stp)
    \big)
\end{equation*}
requires density evaluation of not only our $\gamma$-model, but also the single-step transition distribution.
There are two options for estimating the single-step densities: (1) a single-step model $p_\theta$ may be trained alongside the $\gamma$-model $\model$ for the purposes of constructing targets $T(\st, \at, \stp, \se)$, or (2) a simple approximate model may be constructed on the fly from $(\st, \at, \stp)$ transitions.
We found $p_\theta = \mathcal{N}(\stp, \sigma^2)$, with $\sigma^2$ a constant hyperparameter, to be sufficient.

{


\begin{figure}[t]
\begin{minipage}{\linewidth}
\begin{algorithm}[H]
\caption{~~$\gammam$-model training {\color{gamma_red}without density evaluation}}
\label{alg:practical_samples}
\begin{algorithmic}[1]
\STATE \textbf{Input} $\mathcal{D}:$ dataset of transitions, $\pi:$ policy, $\lambda:$ step size, $\tau:$ delay parameter
\STATE Initialize parameter vectors $\theta, \widebar{\theta}, {\color{gamma_red}\phi}$
\WHILE{not converged}
    \STATE Sample transitions $(\st, \at, \stp)$ from $\mathcal{D}$ and actions $\atp \sim \pi(\cdot \mid \stp)$
    {\color{gamma_red}
    \STATE Sample from bootstrapped target~$\se^{+} \sim (1-\gamma) \delta_{\stp} + \gamma \targmodel(\cdot \mid \stp, \atp)$
    \STATE Sample from current model~$\se^{-} \sim \model(\cdot \mid \st, \at)$
    \STATE Evaluate objective~$\mathcal{L} = \log \disc(\se^{+} \mid \st, \at) + \log(1 - \disc(\se^{-} \mid \st, \at))$
    }
    \STATE Update model parameters~$\theta \gets \theta - \lambda \nabla_{\theta} \mathcal{L}$;~
    {\color{gamma_red}$\phi \gets \phi + \lambda \nabla_{\phi} \mathcal{L}$}
    \STATE Update target parameters $\widebar{\theta} \gets \tau \theta + (1-\tau) \widebar{\theta}$
\ENDWHILE
\end{algorithmic}
\end{algorithm}
\end{minipage}


\begin{minipage}{\linewidth}
\begin{algorithm}[H]
\caption{~~$\gammam$-model training {\color{gamma_blue}with density evaluation}}
\label{alg:practical_logp}
\begin{algorithmic}[1]
\STATE \textbf{Input} $\mathcal{D}:$ dataset of transitions, $\pi:$ policy, $\lambda:$ step size, $\tau:$ delay parameter, {\color{gamma_blue}$\sigma^2:$ variance}
\STATE Initialize parameter vectors $\theta, \widebar{\theta}$; {\color{gamma_blue}let $f$ denote the Gaussian pdf}
\WHILE{not converged}
    \STATE Sample transitions $(\st, \at, \stp)$ from $\mathcal{D}$ and actions $\atp \sim \pi(\cdot \mid \stp)$
    {\color{gamma_blue}
    \STATE Sample from bootstrapped target~$\se \sim (1-\gamma) \mathcal{N}(\stp, \sigma^2) + \gamma \targmodel(\cdot \mid \stp, \atp)$
    \STATE Construct target values $T = \log\big( 
    (1-\gamma)f(\se \mid \stp, \sigma^2) + \gamma \targmodel(\se \mid \stp, \atp)
    \big)$
    \STATE Evaluate objective~$\mathcal{L} =
    \lVert \log \model(\se \mid \st, \at) - T\lVert_2^2$
    }
    \STATE Update model parameters $\theta \gets \theta - \lambda \nabla_{\theta} \mathcal{L}$
    \STATE Update target parameters $\widebar{\theta} \gets \tau \theta + (1-\tau) \widebar{\theta}$
\ENDWHILE
\end{algorithmic}
\end{algorithm}
\end{minipage}
\end{figure}

}

\paragraph{Stability considerations.}
To alleviate the instability caused by bootstrapping, we appeal to the standard solution employed in model-free reinforcement learning: decoupling the regression targets from the current model by way of a ``delayed" target network \citep{mnih2015humanlevel}.
In particular, we use a delayed $\gamma$-model $\targmodel$ in the bootstrapped target distribution $p_\text{targ}$,
with the parameters $\widebar{\theta}$ given by an exponentially-moving average of previous parameters $\theta$.

We summarize the above scenarios in
Algorithms~\ref{alg:practical_samples}~and~\ref{alg:practical_logp}.
We isolate model training from data collection and focus on a setting in which a static dataset is provided, but this algorithm may also be used in a data-collection loop for policy improvement.
Further implementation details, including all hyperparameter settings and network architectures, are included in Appendix~\ref{app:implementation}.

\begin{figure}[!b]
    \centering
    \includegraphics[width=1.0\linewidth]{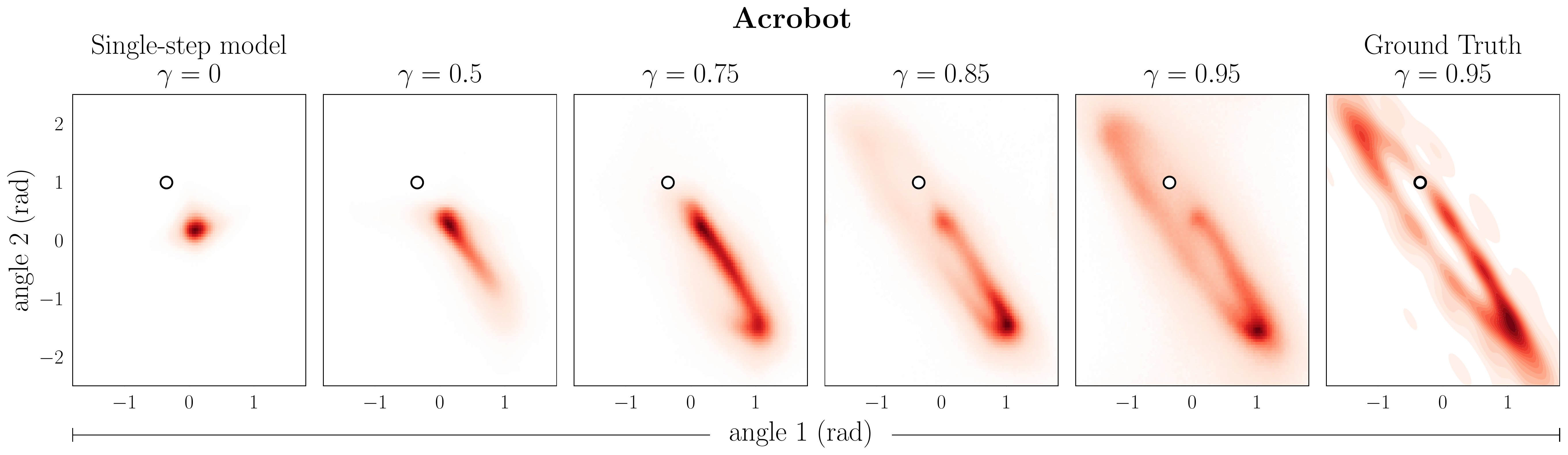} \\
    \vspace{0.2cm}
    \includegraphics[width=1.0\linewidth]{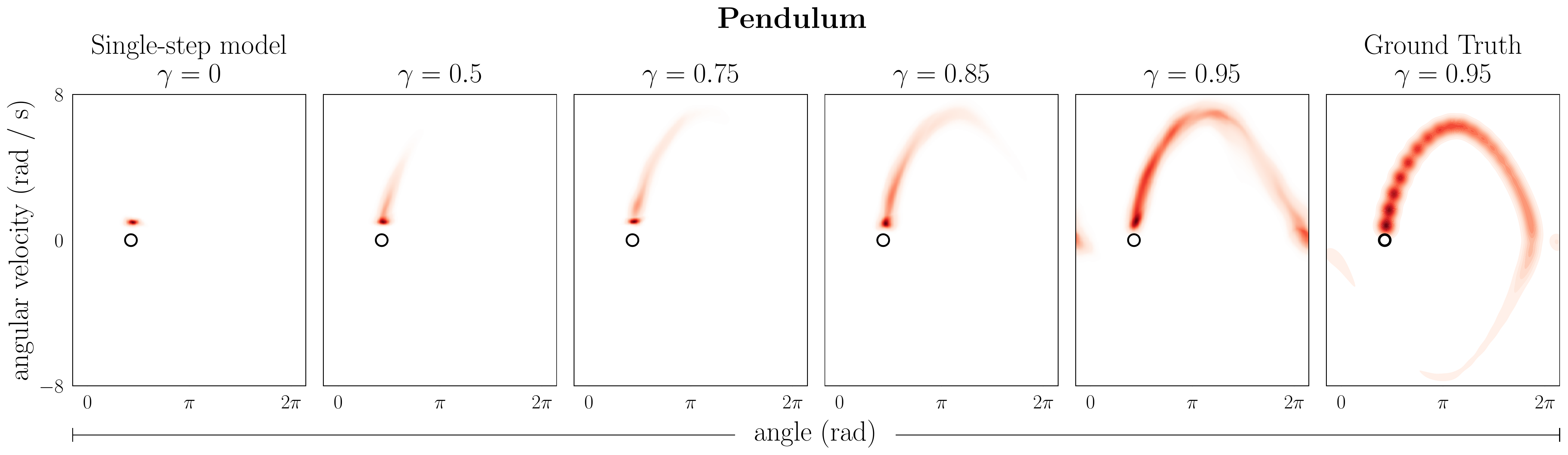}
    \caption{
    \textbf{($\boldgamma$-model predictions)}
    Visualization of the predicted distribution from a \textbf{single} feedforward pass of normalizing flow $\gamma$-models trained with varying discounts $\gamma$.
    The conditioning state $\st$ is denoted by
    \protect\includegraphics[height=1.5ex]{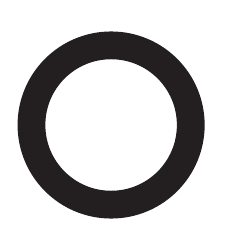}.
    The leftmost plots, with $\gamma=0$, correspond to a single-step model.
    For comparison, the rightmost plots show a Monte Carlo estimation of the discounted occupancy from $100$ environment trajectories.
    }
    \label{fig:density}
\end{figure}

\section{Experimental Evaluation}
\label{sec:experiments}

Our experimental evaluation is designed to study the viability of $\gamma$-models as a replacement of conventional single-step models for long-horizon state prediction and model-based control. 

\subsection{Prediction}

We investigate $\gamma$-model predictions as a function of discount in continuous-action versions of two benchmark environments suitable for visualization: acrobot \citep{sutton1996generalization} and pendulum.
The training data come from a mixture distribution over all intermediate policies of 200 epochs of optimization with soft-actor critic (SAC; \citealt{haarnoja18sac}).
The final converged policy is used for $\gamma$-model training.
We refer to Appendix~\ref{app:implementation} for implementation and experiment details.

Figure~\ref{fig:density} shows the predictions of a $\gamma$-model trained as a normalizing flow according to Algorithm~\ref{alg:practical_logp} for five different discounts, ranging from $\gamma=0$ (a single-step model) to $\gamma=0.95$.
The rightmost column shows the ground truth discounted occupancy corresponding to $\gamma=0.95$, estimated with Monte Carlo rollouts of the policy.
Increasing the discount $\gamma$ during training has the expected effect of qualitatively increasing the predictive lookahead of a single feedforward pass of the $\gamma$-model.
We found flow-based $\gamma$-models to be more reliable than GAN parameterizations, especially at higher discounts.
Corresponding GAN $\gamma$-model visualizations can be found in Appendix~\ref{app:gan_predictions} for comparison.

\begin{figure}[!b]
    $$\hspace{1.1cm}
    \textbf{Rewards}
        \hspace{1.4cm}
    \boldsymbol{\gamma}\textbf{-model predictions}
        \hspace{1.2cm}
    \textbf{Value estimates}
        \hspace{0.9cm}
    \textbf{Ground truth}$$
    \vspace{-0.8cm}
    \centering
    \includegraphics[width=1.0\linewidth]{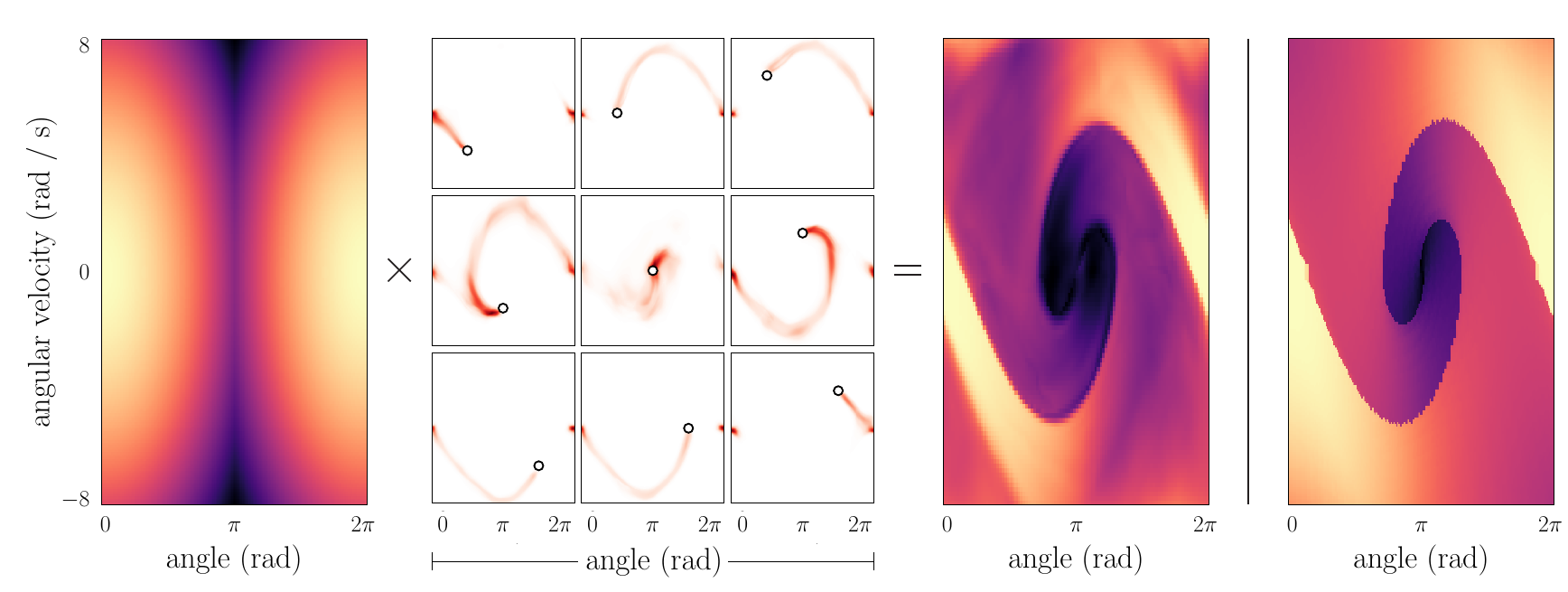} \\
    \caption{
    \textbf{($\boldgamma$-model value estimation)}
    Values are expectations of reward over a single feedforward pass of a $\gamma$-model (Equation~\ref{eq:q_estimation}).
    We visualize $\gamma$-model predictions ($\gamma=0.99$) from nine starting states, denoted by
    \protect\includegraphics[height=1.5ex]{gamma/images/circle.pdf},
    in the pendulum benchmark environment.
    Taking the expectation of reward over each of these predicted distributions yields a value estimate for the corresponding conditioning state.
    The rightmost plot depicts the value map produced by value iteration on a discretization of the same environment for reference.
    }
    \label{fig:value_estimation}
\end{figure}

Equation~\ref{eq:q_estimation} expresses values as an expectation over a single feedforward pass of a $\gamma$-model.
We visualize this relation in Figure~\ref{fig:value_estimation}, which depicts $\gamma$-model predictions on the pendulum environment for a discount of $\gamma=0.99$ and the resulting value map estimated by taking expectations over these predicted state distributions.
In comparison, value estimation for the same discount using a single-step model would require 299-step rollouts in order to recover $95\%$ of the probability mass (see Figure~\ref{fig:sr_composition}).

\subsection{Control}

To study the utility of the $\gamma$-model for model-based reinforcement learning, we use the $\gamma$-MVE estimator from Section~\ref{sec:gamma_mve} as a drop-in replacement for value estimation in SAC.
We compare this approach to the state-of-the-art in model-based and model-free methods, with representative algorithms consisting of SAC, PPO \citep{schulman2017ppo}, MBPO \citep{janner2019mbpo}, and MVE~\citep{feinberg2018mve}.
In $\gamma$-MVE, we use a model discount of $\gamma=0.8$, a value discount of $\gammav=0.99$ and a single model step ($n=1$).
We use a model rollout length of $5$ in MVE such that it has an effective horizon identical to that of $\gamma$-MVE.
Other hyperparameter settings can once again be found in Appendix~\ref{app:implementation}; details regarding the evaluation environments can be found in Appendix~\ref{app:environments}.
Figure~\ref{fig:gamma_mve} shows learning curves for all methods.
We find that $\gamma$-MVE converges faster than prior algorithms, twice as quickly as SAC, while retaining their asymptotic performance.

\begin{figure}[t!]
    \centering
    \includegraphics[width=1.0\linewidth]{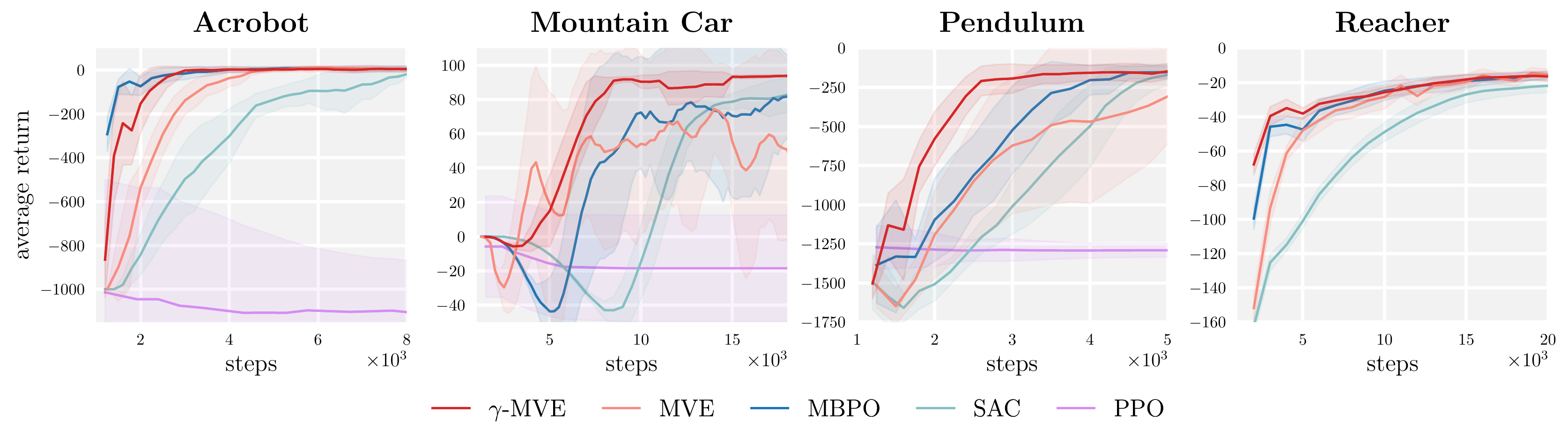}
    \vspace{-.8cm}
    \caption{
    \textbf{($\boldgamma$-MVE control performance)}
    Comparative performance of $\gamma$-MVE and four prior reinforcement learning algorithms on continuous control benchmark tasks. $\gamma$-MVE retains the asymptotic performance of SAC with sample-efficiency matching that of MBPO.
    Shaded regions depict standard deviation among $5$ seeds.
    }
    \label{fig:gamma_mve}
\end{figure}

\section{Discussion}
\label{sec:discussion}

We have introduced a new class of predictive model, a $\gamma$-model, that is a hybrid between standard model-free and model-based mechanisms. It is policy-conditioned and infinite-horizon, like a value function, but independent of reward, like a standard single-step model.
This new formulation of infinite-horizon prediction allows us to generalize the procedures integral to model-based control, yielding new variants of model rollouts and model-based value estimation.
Our experimental evaluation shows that, on tasks with low to moderate dimensionality, our method learns accurate long-horizon predictive distributions without sequential rollouts and can be incorporated into standard model-based reinforcement learning methods to produce results that are competitive with state-of-the-art algorithms.

However, scaling up this framework to more complex tasks, including high-dimensional continuous control problems and tasks with image observations, presents a number of additional challenges.
These challenges are largely those of generative modeling;
whereas temporal differences algorithms are conventionally used to estimate expectations of scalar random variables, we are here employing them to estimate high-dimensional joint distributions.
Unsurprisingly, this approach can fall short above a threshold dimensionality (in our experience, above $10$) or discount factor (above $0.99$), as these both increase the complexity of the distribution in question.
In the next chapter, we investigate whether the language modeling toolbox can be used to address this limitation directly and provide effective generative modeling solutions for the reinforcement learning problem setting.

\chapter[Reinforcement Learning as Sequence Modeling]{\fontsize{22}{14}\selectfont Reinforcement Learning as Sequence Modeling}
\renewcommand{\chaptertitle}{Reinforcement Learning as Sequence Modeling}
\label{ch:transformer}

\section{Introduction}
\label{sec:intro}

The standard treatment of reinforcement learning relies on decomposing a long-horizon problem into smaller, more local subproblems.
In model-free algorithms, this takes the form of the principle of optimality \citep{bellman1957dynamic}, a recursion that leads naturally to the class of dynamic programming methods like $Q$-learning.
In model-based algorithms, this decomposition takes the form of single-step predictive models, which reduce the problem of predicting high-dimensional, policy-dependent state trajectories to that of estimating a comparatively simpler, policy-agnostic transition distribution.
As seen in Chapter~\ref{ch:gamma}, these two approaches constitute the endpoints of a spectrum, and it is possible to design methods that interpolate between them.

However, we can also view reinforcement learning as analogous to a sequence generation problem, with the goal being to produce a sequence of actions that, when enacted in an environment, will yield a sequence of high rewards.
In this chapter, we consider the logical extreme of this analogy: does the toolbox of contemporary sequence modeling itself provide a viable reinforcement learning algorithm?
We investigate this question by treating trajectories as unstructured sequences of states, actions, and rewards.
We model the distribution of these trajectories using a Transformer architecture \citep{vaswani2017attention}, the current tool of choice for capturing long-horizon dependencies.
In place of the trajectory optimizers common in model-based control, we use beam search \citep{reddy1997beam}, a heuristic decoding scheme ubiquitous in natural language processing, as a planning algorithm.

Posing reinforcement learning, and more broadly data-driven control, as a sequence modeling problem handles many of the considerations that typically require distinct solutions: actor-critic algorithms require separate actors and critics, model-based algorithms require predictive dynamics models, and offline reinforcement learning methods often require estimation of the behavior policy \citep{fujimoto2019off}.
These components estimate different densities or distributions, such as that over actions in the case of actors and behavior policies, or that over states in the case of dynamics models.
Even value functions can be viewed as performing inference in a graphical model with auxiliary optimality variables, amounting to estimation of the distribution over future rewards~\citep{levine2018reinforcement}.
All of these problems can be unified under a single sequence model, which treats states, actions, and rewards as simply a stream of data.
The advantage of this perspective is that high-capacity sequence model architectures can be brought to bear on the problem, resulting in a more streamlined approach that could benefit from the same scalability underlying large-scale unsupervised learning results \citep{brown2020gpt3}.

We refer to our model as the Trajectory Transformer (\autoref{fig:architecture}) and evaluate it in the offline regime so as to be able to make use of large amounts of prior interaction data.
The Trajectory Transformer is a substantially more reliable long-horizon predictor than conventional dynamics models, even in Markovian environments for which the standard model parameterization is in principle sufficient.
When decoded with a modified beam search procedure that biases trajectory samples according to their cumulative reward, the Trajectory Transformer attains results on offline reinforcement learning benchmarks that are competitive with the best prior methods designed specifically for that setting.
Additionally, we describe how variations of the same decoding procedure yield a model-based imitation learning method, a goal-reaching method, and, when combined with dynamic programming, a state-of-the-art planner for sparse-reward, long-horizon tasks.
Our results suggest that the algorithms and architectural motifs that have been widely applicable in unsupervised learning carry similar benefits in reinforcement learning.

\begin{figure}
    \centering
    \includegraphics[width=1.0\linewidth]{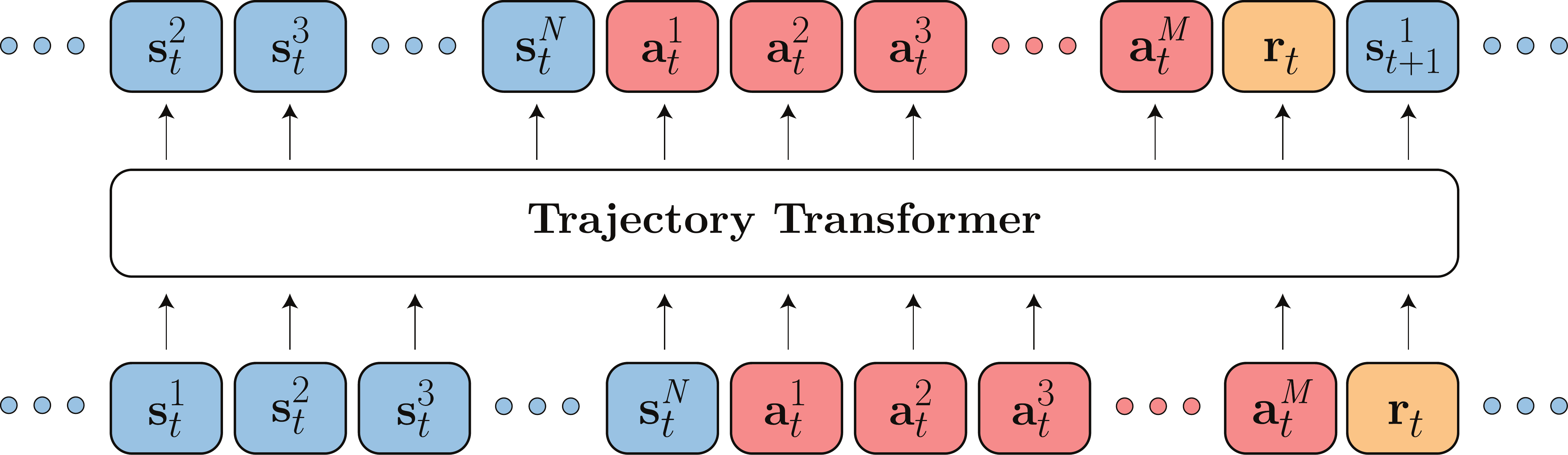}
    \vspace{.25cm}
    \caption{
    \textbf{(Trajectory Transformer architecture)}
    The Trajectory Transformer trains on sequences of (autoregressively discretized) states, actions, and rewards.
    Planning with the Trajectory Transformer mirrors the sampling procedure used to generate sequences from a language model.
    }
    \label{fig:architecture}
\end{figure}

\section{Related Work}
\label{sec:related}

Recent advances in sequence modeling with deep networks have led to rapid improvement in the effectiveness of such models, from LSTMs and sequence-to-sequence models~\citep{hochreiter1997long, NIPS2014_a14ac55a} to Transformer architectures with self-attention \citep{vaswani2017attention}.
In light of this, it is tempting to consider how such sequence models can lead to improved performance in reinforcement learning, which is also concerned with sequential processes~\citep{sutton1988learning}. Indeed, a number of prior works have studied applying sequence models of various types to represent components in standard reinforcement learning algorithms, such as policies, value functions, and models~\citep{bakker2002reinforcement, Heess2015MemorybasedCW, chiappa2017recurrent,parisotto2020stabilizing,parisotto2021efficient,kumar2020adaptive}. While such works demonstrate the importance of such models for representing memory~\citep{oh2016control}, they still rely on standard reinforcement learning algorithmic advances to improve performance. The goal in our work is different: we aim to replace as much of the reinforcement learning pipeline as possible with sequence modeling, so as to produce a simpler method whose effectiveness is determined by the representational capacity of the sequence model rather than algorithmic sophistication.

Estimation of probability distributions and densities arises in many places in learning-based control. This is most obvious in model-based reinforcement learning, where it is used to train predictive models that can then be used for planning or policy learning~\citep{sutton1990dyna,silver2008dyna2,fairbank2008reinforcement,deisenroth2011pilco,lampe2014modelnfq,heess2015svg,janner2020gamma,amos2020model}. However, it also figures heavily in offline reinforcement learning, where it is used to estimate conditional distributions over actions that serve to constrain the learned policy to avoid out-of-distribution behavior that is not supported under the dataset~\citep{fujimoto2019off,kumar2019stabilizing,ghasemipour2020emaq}; imitation learning, where it is used to fit an expert's actions to obtain a policy~\citep{ross2010efficient,ross2011reduction}; and other areas such as hierarchical reinforcement learning~\citep{peng2017deeploco,co2018self,jiang2019language}.
In our method, we train a single high-capacity sequence model to represent the joint distribution over sequences of states, actions, and rewards.
This serves as \emph{both} a predictive model \emph{and} a behavior policy (for imitation) or behavior constraint (for offline reinforcement learning).

Our approach to reinforcement learning is most closely related to prior model-based methods that plan with a learned model~\citep{chua2018pets,wang2019exploring}.
However, while these prior methods typically require additional machinery to work well, such as ensembles in the online setting~\citep{kurutach2018model,buckman2018sample,malik2019calibrated} or conservatism mechanisms in the offline setting~\citep{yu2020mopo,kidambi2020morel,argenson2020model}, our method does not require explicit handling of these components.
Modeling the states and actions jointly already provides a bias toward generating in-distribution actions, which avoids the need for explicit pessimism~\citep{fujimoto2019off,kumar2019stabilizing,ghasemipour2020emaq,nair2020accelerating,jin2020pessimism,yin2021near,dadashi2021offline}.
Our method also differs from most prior model-based algorithms in the dynamics model architecture used, with fully-connected networks parameterizing diagonal-covariance Gaussian distributions being a common choice \citep{chua2018pets}, though recent work has highlighted the effectiveness of autoregressive state prediction \citep{zhang2021autoregressive} like that used by the Trajectory Transformer.
In the context of recently proposed offline reinforcement learning algorithms, our method can be interpreted as a combination of model-based reinforcement learning and policy constraints~\citep{kumar2019stabilizing,wu2019behavior}, though our approach does not require introducing such constraints explicitly.
In the context of model-free reinforcement learning, our method also resembles recently proposed work on goal relabeling \citep{andrychowicz2017hindsight,rauber2017hindsight,ghosh2019gcsl,paster2021planning} and reward conditioning \citep{schmidhuber2019reinforcement,srivastava2019training,kumar2019reward} to reinterpret all past experience as useful demonstrations with proper contextualization.

Concurrently with our work, \citet{chen2021decision} also proposed a reinforcement learning approach centered around sequence prediction, focusing on reward conditioning as opposed to the beam-search-based planning used by the Trajectory Transformer.
Their work further supports the possibility that a high-capacity sequence model can be applied to reinforcement learning problems without the need for the components usually associated with reinforcement learning algorithms.

\section{Reinforcement Learning as Sequence Modeling}
\label{sec:method}

In this section, we describe the training procedure for our sequence model and discuss how it can be used for control.
We refer to the model as the Trajectory Transformer for brevity, but emphasize that at the implementation level, both our model and search strategy are nearly identical to those common in natural language processing.
As a result, modeling considerations are concerned less with architecture design and more with how to represent trajectory data -- potentially consisting of continuous states and actions -- for processing by a discrete-token architecture \citep{radford2018improving}.

\subsection{Trajectory Transformer}

At the core of our approach is the treatment of trajectory data as an unstructured sequence for modeling by a Transformer architecture.
A trajectory $\bm{\tau}$ consists of $T$ states, actions, and scalar rewards:
\[
    \bm{\tau} = \big( \bs_1, \ba_1, r_1, \bs_2, \ba_2, r_2, \ldots, \bs_T, \ba_T, r_T \big).
\]
In the event of continuous states and actions, we discretize each dimension independently. Assuming $N$-dimensional states and $M$-dimensional actions, this turns $\tau$ into sequence of length $T(N+M+1)$:
\[
\bm{\tau} = \big(\ldots, s_t^1, s_t^2, \ldots, s_t^{N}, a_t^1, a_t^2, \ldots, a_t^{M}, r_t, \ldots \big)~~~~~~~~ t=1, \ldots, T.
\]
Subscripts on all tokens denote timestep and superscripts on states and actions denote dimension (\emph{i.e.}, $s_t^i$ is the $i^\text{th}$ dimension of the state at time $t$).
While this choice may seem inefficient, it allows us to model the distribution over trajectories with more expressivity without simplifying assumptions such as Gaussian transitions. 

We investigate two simple discretization approaches:

\begin{enumerate}
    \item \textbf{Uniform:} All tokens for a given dimension correspond to a fixed width of the original continuous space. Assuming a per-dimension vocabulary size of $V$, the tokens for state dimension $i$ cover uniformly-spaced intervals of width $(\max \bs^i - \min \bs^i)/ V$.
    \item \textbf{Quantile:} All tokens for a given dimension account for an equal amount of probability mass under the empirical data distribution; each token accounts for 1 out of every $V$ data points in the training set.
\end{enumerate}

Uniform discretization has the advantage that it retains information about Euclidean distance in the original continuous space, which may be more reflective of the structure of a problem than the training data distribution.
However, outliers in the data may have outsize effects on the discretization size, leaving many tokens corresponding to zero training points.
The quantile discretization scheme ensures that all tokens are represented in the data.
We compare the two empirically in Section~\ref{sec:exp_rl}.

Our model is a Transformer decoder mirroring the GPT architecture \citep{radford2018improving}.
We use a smaller architecture than those typically used in large-scale language modeling, consisting of four layers and four self-attention heads.
(A full architectural description is provided in Appendix~\ref{sec:details}.)
Training is performed with the standard teacher-forcing procedure \citep{williams1989learning} used to train sequence models.
Denoting the parameters of the Trajectory Transformer as $\theta$ and induced conditional probabilities as $P_\theta$, the objective maximized during training is:
\[
\mathcal{L}(\tau) = 
    \sum_{t=1}^{T} \Big(
    \sum_{i=1}^{N} \log P_\theta\big(s_{t}^i \mid         
        \bs_{~t}^{<i},
        \bm{\tau}_{<t}
        \big) +
    \sum_{j=1}^{M} \log P_\theta\big(a_{t}^j \mid
        \ba_{t}^{<j},
        \bs_{t},
        \bm{\tau}_{<t}
        \big) +
    \log P_\theta\big(r_t \mid
        \ba_t,
        \bs_t,
        \bm{\tau}_{<t}
    \big)
    \Big),
\]
in which we use $\bm{\tau}_{<t}$ to denote a trajectory from timesteps $0$ through $t-1$, $\st^{<i}$ to denote dimensions $0$ through $i-1$ of the state at timestep $t$, and similarly for $\at^{<j}$.
We use the Adam optimizer \citep{ba2015adam} with a learning rate of $\num{2.5e-4}$ to train parameters $\theta$.

\def\NoNumber#1{{\def\alglinenumber##1{}\STATE #1}\addtocounter{ALG@line}{-1}}

{\centering
  \begin{algorithm}[t]
    \caption{Beam search}
    \label{alg:beam}
    \begin{algorithmic}[1]
    \STATE \textbf{Require} Input sequence $\mathbf{x}$, vocabulary $\mathcal{V}$, sequence length $T$, beam width $B$ \\
    \STATE \textbf{Initialize} $Y_0 = \{~(~)~\}$ \\
    \FOR{$t = 1, \ldots, T$}
        \STATE $\mathcal{C}_t \gets \{
            \mathbf{y}_{t-1} \circ y \mid 
            \mathbf{y}_{t-1} \in Y_{t-1} \text{~and~} y \in \mathcal{V}
        \}$ \hspace{0.75cm}\small{\color{gray}\te{// candidate single-token extensions}}
        \STATE $Y_{t} \gets \argmax{
            Y \subseteq \mathcal{C}_t, ~|Y| = B
        }{
            \log P_\theta(Y \mid \mathbf{x})
        }$ \hspace{1.5cm}\small\hspace{-.0cm}{\color{gray}\te{// $B$ most likely sequences from candidates}} \\
    \ENDFOR \\
    \STATE \textbf{Return} $\argmax{\mathbf{y} \in Y_T}{\log P_\theta(\mathbf{y} \mid \mathbf{x})}$
    \end{algorithmic}
  \end{algorithm}
}

\subsection{Planning with Beam Search}
\label{sec:tto}

We now describe how sequence generation with the Trajectory Transformer can be repurposed for control, focusing on three settings: imitation learning, goal-conditioned reinforcement learning, and offline reinforcement learning.
These settings are listed in increasing amount of required modification on top of the sequence model decoding approach routinely used in natural language processing.

The core algorithm providing the foundation of our planning techniques, beam search, is described in Algorithm~\ref{alg:beam} for generic sequences.
Following the presentation in \cite{meister2020beam}, we have overloaded $\log P_\theta(\cdot \mid \mathbf{x})$ to define the likelihood of a set of sequences in addition to that of a single sequence: $\log P_\theta(Y \mid x) = \sum_{\mathbf{y} \in Y} \log P_\theta (\mathbf{y} \mid \mathbf{x})$.
We use $(~)$ to denote the empty sequence and $\circ$ to represent concatenation.

\paragraph{Imitation learning.}
When the goal is to reproduce the distribution of trajectories in the training data, we can optimize directly for the probability of a trajectory $\bm{\tau}$.
This situation matches the goal of sequence modeling exactly and as such we may use Algorithm~\ref{alg:beam} without modification by setting the conditioning input $\mathbf{x}$ to the current state $\st$ (and optionally previous history $\bm{\tau}_{<t}$).

The result of this procedure is a tokenized trajectory $\bm{\tau}$, beginning from a current state $\st$, that has high probability under the data distribution.
If the first action $\at$ in the sequence is enacted and beam search is repeated, we have a receding horizon-controller.
This approach resembles a long-horizon model-based variant of behavior cloning, in which entire trajectories are optimized to match those of a reference behavior instead of only immediate state-conditioned actions.
If we set the predicted sequence length to be the action dimension, our approach corresponds exactly to the simplest form of behavior cloning with an autoregressive policy.

\paragraph{Goal-conditioned reinforcement learning.}
Transformer architectures feature a ``causal'' attention mask to ensure that predictions only depend on previous tokens in a sequence.
In the context of natural language, this design corresponds to generating sentences in the linear order in which they are spoken as opposed to an ordering reflecting their hierarchical syntactic structure
(see, however, \citealt{gu2019insertion} for a discussion of non-left-to-right sentence generation with autoregressive models).
In the context of trajectory prediction, this choice instead reflects physical causality, disallowing future events to affect the past.
However, the conditional probabilities of the past given the future are still well-defined, allowing us to condition samples not only on the preceding states, actions, and rewards that have already been observed, but also any future context that we wish to occur.
If the future context is a state at the end of a trajectory, we decode trajectories with probabilities of the form:
\[
    P_\theta(s_{t}^i \mid \bs_t^{<i}, \bm{\tau}_{<t}, {\color{cblue}\bs_T})
\]
We can use this directly as a goal-reaching method by conditioning on a desired final state $\color{cblue}\bs_T$.
If we always condition sequences on a final goal state, we may leave the lower-diagonal attention mask intact and simply permute the input trajectory to $\{ {\color{cblue}\bs_{T}}, \bs_1, \bs_2, \ldots, \bs_{T-1} \}$.
By prepending the goal state to the beginning of a sequence, we ensure that all other predictions may attend to it without modifying the standard attention implementation.
This procedure for conditioning resembles prior methods that use supervised learning to train goal-conditioned policies \citep{ghosh2019gcsl} and is also related to relabeling techniques in model-free reinforcement learning \citep{andrychowicz2017hindsight}.
In our framework, it is identical to the standard subroutine in sequence modeling: inferring the most likely sequence given available evidence.

\paragraph{Offline reinforcement learning.}

The beam search method described in Algorithm~\ref{alg:beam} optimizes sequences for their probability under the data distribution.
By replacing the log-probabilities of transitions with the predicted reward signal, we can use the same Trajectory Transformer and search strategy for reward-maximizing behavior.
Appealing to the control as inference graphical model~\citep{levine2018reinforcement}, we are in effect replacing a transition's log-probability in beam search with its log-probability of optimality.

Using beam-search as a reward-maximizing procedure has the risk of leading to myopic behavior.
To address this issue, we augment each transition in the training trajectories with reward-to-go:
$R_t = \sum_{t'=t}^{T} \gamma^{t'-t} r_{t'}$
and include it as an additional quantity, discretized identically to the others, to be predicted after immediate rewards $r_t$.
During planning, we then have access to value estimates from our model to add to cumulative rewards.
While acting greedily with respect to such Monte Carlo value estimates is known to suffer from poor sample complexity and convergence to suboptimal behavior when online data collection is not allowed, we only use this reward-to-go estimate as a heuristic to guide beam search, and hence our method does not require the estimated values to be as accurate as in methods that rely solely on the value estimates to select actions.

In offline RL, reward-to-go estimates are functions of the \emph{behavior} policy that collected the training data and do not, in general, correspond to the values achieved by the Trajectory Transformer-derived policy.
Of course, it is much simpler to learn the value function of the behavior policy than that of the optimal policy, since we can simply use Monte Carlo estimates without relying on Bellman updates.
A value function for an improved policy would provide a better search heuristic, though requires invoking the tools of dynamic programming.
In Section~\ref{sec:exp_rl} we show that the simple reward-to-go estimates are sufficient for planning with the Trajectory Transformer in many environments, but that improved value functions are useful in the most challenging settings, such as sparse-reward tasks.

Because the Trajectory Transformer predicts reward and reward-to-go only every $N+M+1$ tokens, we sample all intermediate tokens according to model log-probabilities, as in the imitation learning and goal-reaching settings.
More specifically, we sample full transitions $(\bs_t, \ba_t, r_t, R_t)$ using likelihood-maximizing beam search, treat these transitions as our vocabulary, and filter sequences of transitions by those with the highest cumulative reward plus reward-to-go estimate.

We have taken a sequence-modeling route to what could be described as a fairly simple-looking model-based planning algorithm, in that we sample candidate action sequences, evaluate their effects using a predictive model, and select the reward-maximizing trajectory.
This conclusion is in part due to the close relation between sequence modeling and trajectory optimization.
There is one dissimilarity, however, that is worth highlighting: by modeling actions jointly with states and sampling them using the same procedure, we can prevent the model from being queried on out-of-distribution actions.
The alternative, of treating action sequences as unconstrained optimization variables that do not depend on state \citep{nagabandi2018mbmf}, can more readily lead to model exploitation, as the problem of maximizing reward under a learned model closely resembles that of finding adversarial examples for a classifier \citep{goodfellow2014explaining}.

\begin{figure}[tb]
  \centering
  \begin{subfigure}[t]{0.02\textwidth}
    \raisebox{-1.4cm}{\rotatebox[origin=t]{90}{\textbf{Reference}}}
  \end{subfigure}
  \begin{subfigure}[t]{0.96\textwidth}
    \includegraphics[width=\linewidth,valign=t]{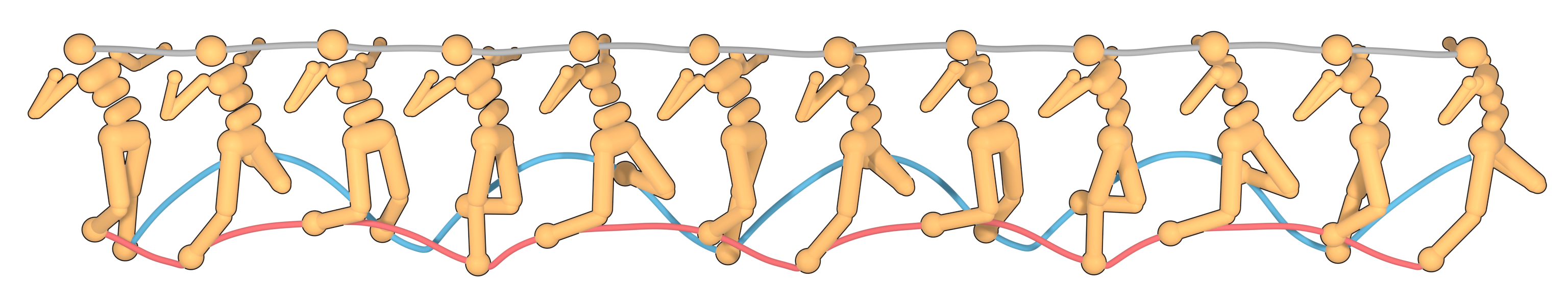} \\
  \end{subfigure}\hfill
  \vspace{-.5cm}
  \begin{subfigure}[t]{0.02\textwidth}
    \raisebox{-1.4cm}{\rotatebox[origin=t]{90}{\textbf{Transformer}}}
  \end{subfigure}
  \begin{subfigure}[t]{0.96\textwidth}
    \includegraphics[width=\linewidth,valign=t]{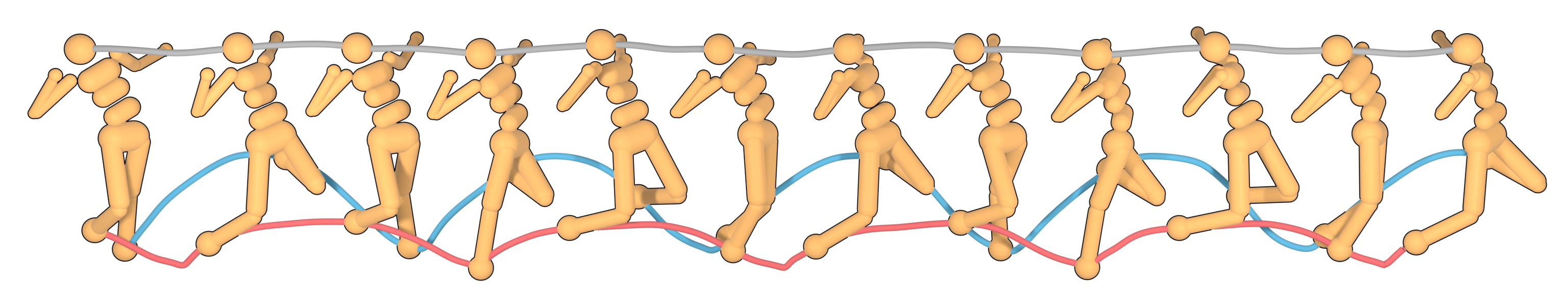} \\
  \end{subfigure}\hfill
  \vspace{-.5cm}
  \begin{subfigure}[t]{0.02\textwidth}
    \hspace{-.255cm}
    \raisebox{-1.6cm}{\rotatebox[origin=t]{90}{\textbf{Feedforward}}}
  \end{subfigure}
  \begin{subfigure}[t]{0.96\textwidth}
    \includegraphics[width=\linewidth,valign=t]{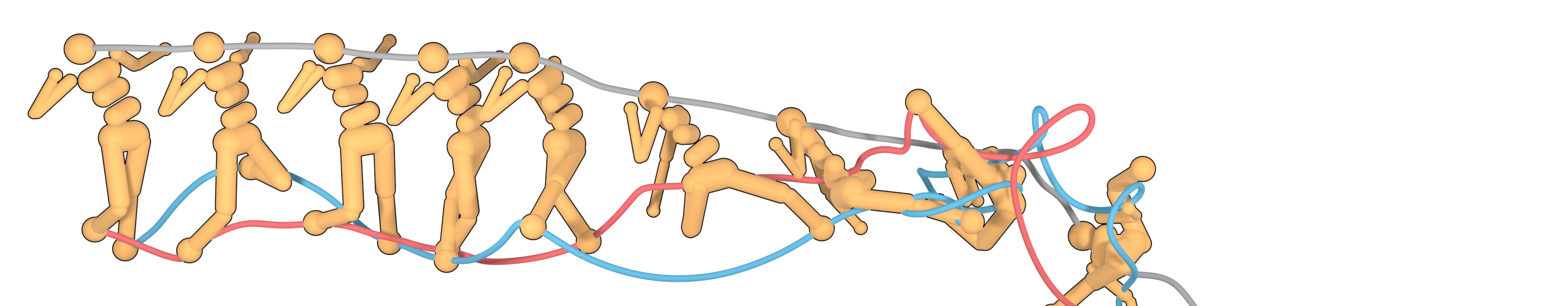} \\
  \end{subfigure}\hfill
    \caption{
    \textbf{(Transformer prediction visualization)}
    A qualitative comparison of length-100 trajectories generated by the Trajectory Transformer and a feedforward Gaussian dynamics model from PETS, a state-of-the-art planning algorithm \cite{chua2018pets}.
    Both models were trained on trajectories collected by a single policy, for which a true trajectory is shown for reference.
    Compounding errors in the single-step model lead to physically implausible predictions, whereas the Transformer-generated trajectory is visually indistinguishable from those produced by the policy acting in the actual environment.
    The paths of the feet and head are traced through space for depiction of the movement between rendered frames.
    }
    \label{fig:humanoid}
\end{figure}

\section{Experimental Evaluation}
\label{sec:experiments}

Our experimental evaluation focuses on (1) the accuracy of the Trajectory Transformer as a long-horizon predictor compared to standard dynamics model parameterizations and (2) the utility of sequence modeling tools -- namely beam search -- as a control algorithm in the context of offline reinforcement learning, imitation learning, and goal-reaching.

\subsection{Model Analysis}

We begin by evaluating the Trajectory Transformer
as a long-horizon policy-conditioned predictive model.
The usual strategy for predicting trajectories given a policy is to rollout with a single-step model, with actions supplied by the policy.
Our protocol differs from the standard approach not only in that the model is not Markovian, but also in that it does not require access to a policy to make predictions -- the outputs of the policy are modeled alongside the states encountered by that policy.
Here, we focus only on the quality of the model's predictions; we use actions predicted by the model for an imitation learning method in the next subsection.

\begin{figure}
    \centering
    \includegraphics[width=1.0\linewidth]{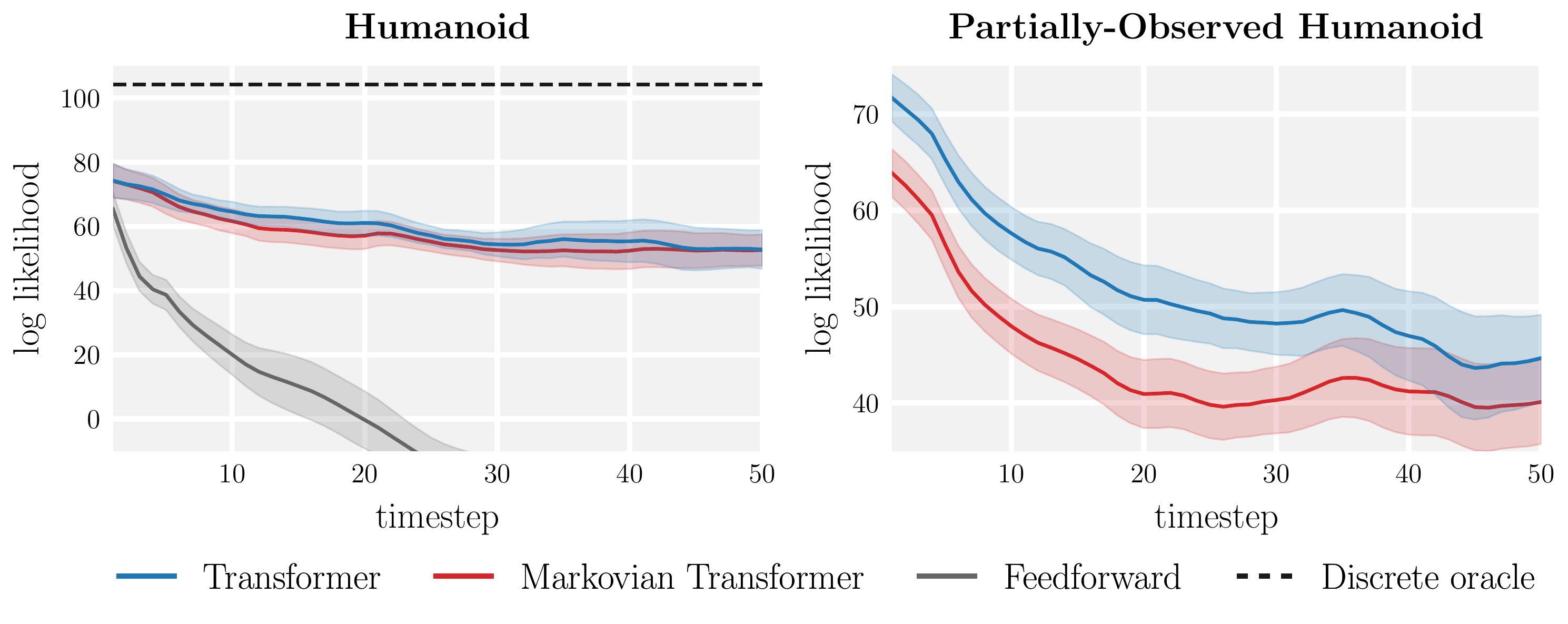}
    \caption{
    \textbf{(Compounding model errors)}
    We compare the accuracy of the Trajectory Transformer (with uniform discretization) to that of the probabilistic feedforward model ensemble \citep{chua2018pets} over the course of a planning horizon in the humanoid environment, corresponding to the trajectories visualized in Figure~\ref{fig:humanoid}.
    The Trajectory Transformer has substantially better error compounding with respect to prediction horizon than the feedforward model.
    The discrete oracle is the maximum log likelihood attainable given the discretization size; see Appendix~\ref{app:oracle} for a discussion.
    }
    \label{fig:model_error}
\end{figure}

\paragraph{Trajectory predictions.}
Figure~\ref{fig:humanoid} depicts a visualization of predicted 100-timestep trajectories from our model after having trained on a dataset collected by a trained humanoid policy.
Though model-based methods have been applied to the humanoid task, prior works tend to keep the horizon intentionally short to prevent the accumulation of model errors \citep{janner2019mbpo,amos2020model}.
The reference model is the probabilistic ensemble implementation of PETS \citep{chua2018pets}; we tuned the number of models within the ensemble, the number of layers, and layer sizes, but were unable to produce a model that predicted accurate sequences for more than a few dozen steps.
In contrast, we see that the Trajectory Transformer's
long-horizon predictions are substantially more accurate, remaining visually indistinguishable from the ground-truth trajectories even after 100 predicted steps.
To our knowledge, no prior model-based reinforcement learning algorithm has demonstrated predicted rollouts of such accuracy and length on tasks of comparable dimensionality.

\paragraph{Error accumulation.}
A quantitative account of the same finding is provided in Figure~\ref{fig:model_error}, in which we evaluate the model's accumulated error versus prediction horizon.
Standard predictive models tend to have excellent single-step errors but poor long-horizon accuracy, so instead of evaluating a test-set single-step likelihood, we sample 1000 trajectories from a fixed starting point to estimate the per-timestep state marginal predicted by each model.
We then report the likelihood of the states visited by the reference policy on a held-out set of trajectories under these predicted marginals.
To evaluate the likelihood under our discretized model, we treat each bin as a uniform distribution over its specified range; by construction, the model assigns zero probability outside of this range.

To better isolate the source of the Transformer's improved accuracy over standard single-step models, we also evaluate a Markovian variant of our same architecture.
This ablation has a truncated context window that prevents it from attending to more than one timestep in the past.
This model performs similarly to the trajectory Transformer on fully-observed environments, suggesting that architecture differences and increased expressivity from the autoregressive state discretization play a large role in the trajectory Transformer's long-horizon accuracy.
We construct a partially-observed version of the same humanoid environment, in which each dimension of every state is masked out with $50\%$ probability (Figure~\ref{fig:model_error} right), and find that, as expected, the long-horizon conditioning plays a larger role in the model's accuracy in this setting.

\begin{figure}
    \centering
    \includegraphics[width=0.35\linewidth]{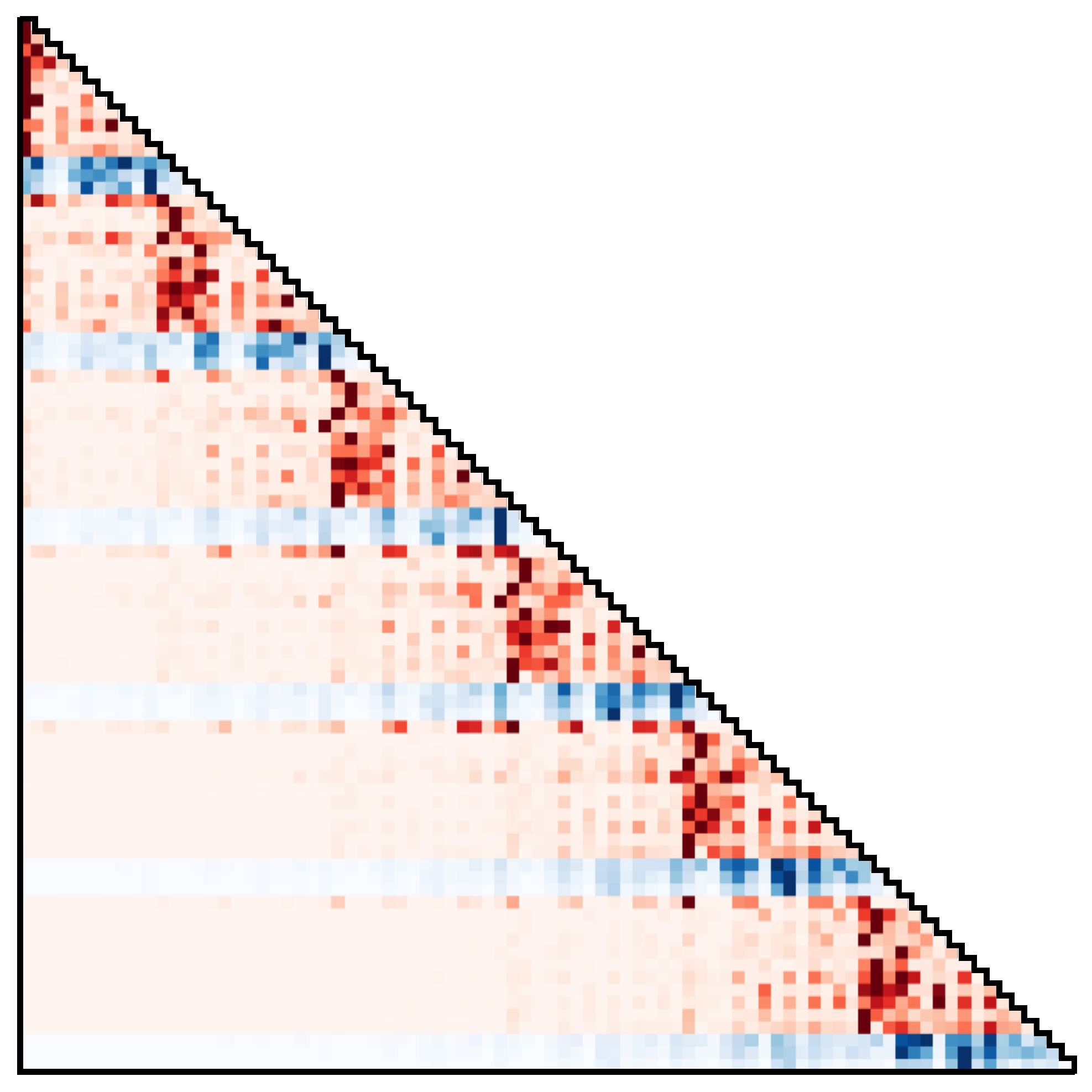}
    \hspace{0.15\linewidth}
    \includegraphics[width=0.35\linewidth]{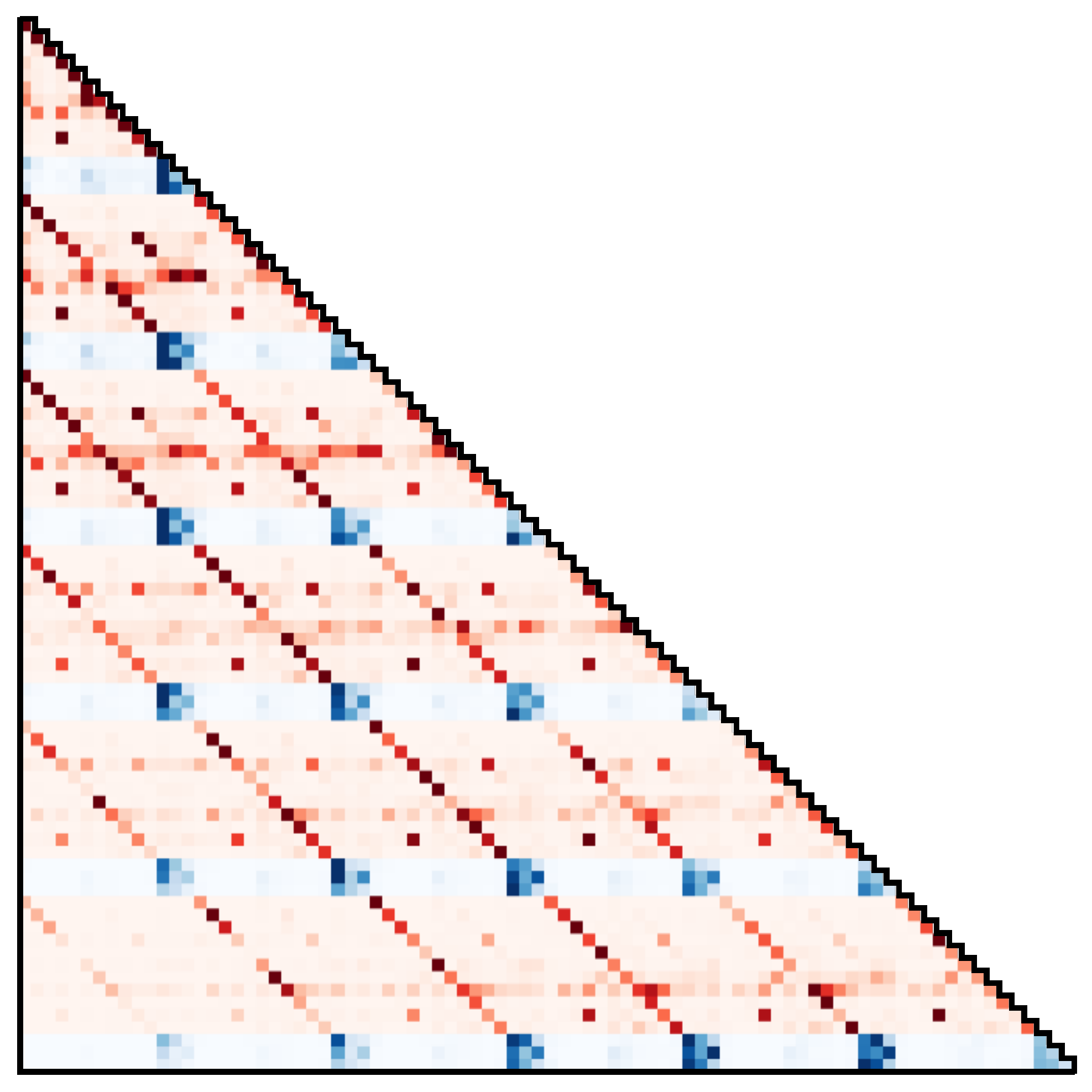}
    \begin{flushleft}
    \vspace{-6.2cm}
    ${\color{cred}\st}$ \hspace{7.8cm} ${\color{cred}\st}$ \\
    \vspace{.04cm}
    ${\color{cblue}\at}$ \hspace{7.8cm} ${\color{cblue}\at}$ \\
    \vspace{.95cm}
    {\color{cred}\Large\vdots} \hspace{8.2cm} \\
        \vspace{-0.95cm}\hspace{8.2cm}
        {\color{cred}\Large\vdots} \\
    \vspace{.1cm}
    {\color{cblue}\Large\vdots}  \hspace{8.2cm} \\
        \vspace{-0.95cm}\hspace{8.2cm}
        {\color{cblue}\Large\vdots} \\
    \vspace{.77cm}
    ${\color{cred}\bs_{t+5}}$  \hspace{7.45cm} ${\color{cred}\bs_{t+5}}$ \\
    \vspace{.02cm}
    ${\color{cblue}\ba_{t+5}}$ \hspace{7.45cm} ${\color{cblue}\ba_{t+5}}$ \\
    \vspace{.05cm}
    \vspace{.2cm}
    \end{flushleft}
    \vspace{.05cm}
    \caption{
    \textbf{(Attention patterns)}
    We observe two distinct types of attention masks during trajectory prediction.
    In the first, both {\color{cred}states} and {\color{cblue}actions} are dependent primarily on the immediately preceding transition, corresponding to a model that has learned the Markov property.
    The second strategy has a striated appearance, with state dimensions depending most strongly on the same dimension of multiple previous timesteps.
    Surprisingly, actions depend more on past actions than they do on past states, reminiscent of the action smoothing used in some trajectory optimization algorithms \citep{nagabandi2019pddm}.
    The above masks are produced by a first- and third-layer attention head during sequence prediction on the hopper benchmark; reward dimensions are omitted for this visualization.$^1$
    }
    \label{fig:attention}
\end{figure}

\paragraph{Attention patterns.}
We visualize the attention maps during model predictions in Figure~\ref{fig:attention}.
We find two primary attention patterns.
The first is a discovered Markovian strategy, in which a state prediction attends overwhelmingly to the previous transition.
The second is qualitatively striated, with the model attending to specific dimensions in multiple prior states for each state prediction.
Simultaneously, the action predictions attend to prior actions more than they do prior states.
The action dependencies contrast with the usual formulation of behavior cloning, in which actions are a function of only past states, but is reminiscent of the action filtering technique used in some planning algorithm to produce smoother action sequences \citep{nagabandi2019pddm}.

\subsection{Reinforcement Learning and Control}
\label{sec:exp_rl}

\begin{table*}[b]
\centering
\footnotesize
\begin{tabular*}{\textwidth}{@{\extracolsep{\fill}}llrrrrrrr}
\toprule
\multicolumn{1}{c}{\textbf{Dataset}} & \multicolumn{1}{c}{\textbf{Environment}} & \multicolumn{1}{c}{\textbf{BC}} & \multicolumn{1}{c}{\textbf{MBOP}} & \multicolumn{1}{c}{\textbf{BRAC}} & \multicolumn{1}{c}{\textbf{CQL}} & \multicolumn{1}{c}{\textbf{DT}} & \multicolumn{1}{c}{\textbf{TT} \scriptsize{\textcolor{cgrey}{(uniform)}}} & \multicolumn{1}{c}{\textbf{TT} \scriptsize{\textcolor{cgrey}{(quantile)}}} \\ 
\midrule
Med-Expert & HalfCheetah & $59.9$ & $105.9$ & $41.9$ & $91.6$ & $86.8$ & $40.8$ \scriptsize{\raisebox{1pt}{$\pm 2.3$}} & $95.0$ \scriptsize{\raisebox{1pt}{$\pm 0.2$}} \\ 
Med-Expert & Hopper & $79.6$ & $55.1$ & $0.9$ & $105.4$ & $107.6$ & $106.0$ \scriptsize{\raisebox{1pt}{$\pm 0.28$}} & $110.0$ \scriptsize{\raisebox{1pt}{$\pm 2.7$}} \\ 
Med-Expert & Walker2d & $36.6$ & $70.2$ & $81.6$ & $108.8$ & $108.1$ & $91.0$ \scriptsize{\raisebox{1pt}{$\pm 2.8$}} & $101.9$ \scriptsize{\raisebox{1pt}{$\pm 6.8$}} \\ 
\midrule
Medium & HalfCheetah & $43.1$ & $44.6$ & $46.3$ & $44.0$ & $42.6$ & $44.0$ \scriptsize{\raisebox{1pt}{$\pm 0.31$}} & $46.9$ \scriptsize{\raisebox{1pt}{$\pm 0.4$}} \\ 
Medium & Hopper & $63.9$ & $48.8$ & $31.3$ & $58.5$ & $67.6$ & $67.4$ \scriptsize{\raisebox{1pt}{$\pm 2.9$}} & $61.1$ \scriptsize{\raisebox{1pt}{$\pm 3.6$}} \\ 
Medium & Walker2d & $77.3$ & $41.0$ & $81.1$ & $72.5$ & $74.0$ & $81.3$ \scriptsize{\raisebox{1pt}{$\pm 2.1$}} & $79.0$ \scriptsize{\raisebox{1pt}{$\pm 2.8$}} \\ 
\midrule
Med-Replay & HalfCheetah & $4.3$ & $42.3$ & $47.7$ & $45.5$ & $36.6$ & $44.1$ \scriptsize{\raisebox{1pt}{$\pm 0.9$}} & $41.9$ \scriptsize{\raisebox{1pt}{$\pm 2.5$}} \\ 
Med-Replay & Hopper & $27.6$ & $12.4$ & $0.6$ & $95.0$ & $82.7$ & $99.4$ \scriptsize{\raisebox{1pt}{$\pm 3.2$}} & $91.5$ \scriptsize{\raisebox{1pt}{$\pm 3.6$}} \\ 
Med-Replay & Walker2d & $36.9$ & $9.7$ & $0.9$ & $77.2$ & $66.6$ & $79.4$ \scriptsize{\raisebox{1pt}{$\pm 3.3$}} & $82.6$ \scriptsize{\raisebox{1pt}{$\pm 6.9$}} \\ 
\midrule
\multicolumn{2}{c}{\textbf{Average}} & 47.7 & 47.8 & 36.9 & 77.6 & 74.7 & 72.6 \hspace{.6cm} & 78.9 \hspace{.6cm} \\ 
\bottomrule
\end{tabular*}

\caption{
\textbf{(Offline reinforcement learning)} The Trajectory Transformer (TT) performs on par with or better than the best prior offline reinforcement learning algorithms on D4RL locomotion (\texttt{v2}) tasks.
Results for TT variants correspond to the mean and standard error over 15 random seeds (5 independently trained Transformers and 3 trajectories per Transformer).
We detail the sources of the performance for other methods in Appendix~\ref{app:baselines}.
}
\label{table:d4rl}
\end{table*}

\paragraph{Offline reinforcement learning.}
We evaluate the Trajectory Transformer on a number of environments from the D4RL offline benchmark suite \citep{fu2020d4rl}, including the locomotion and AntMaze domains.
This evaluation is the most difficult of our control settings, as reward-maximizing behavior is the most qualitatively dissimilar from the types of behavior that are normally associated with unsupervised modeling -- namely, imitative behavior.
Results for the locomotion environments are shown in Table~\ref{table:d4rl}.
We compare against five other methods spanning other approaches to data-driven control: (1) behavior-regularized actor-critic (BRAC; \citealt{wu2019behavior}) and conservative $Q$-learning (CQL; \citealt{kumar2020conservative}) represent the current state-of-the-art in model-free offline RL; model-based offline planning (MBOP; \citealt{argenson2020model}) is the best-performing prior offline trajectory optimization technique; decision transformer (DT; \cite{chen2021decision}) is a concurrently-developed sequence-modeling approach that uses return-conditioning instead of planning; and behavior-cloning (BC) provides the performance of a pure imitative method.

The Trajectory Transformer performs on par with or better than all prior methods (Table~\ref{table:d4rl}).
The two discretization variants of the Trajectory Transformer, uniform and quantile, perform similarly on all environments except for HalfCheetah-Medium-Expert, where the large range of the velocities prevents the uniform discretization scheme from recovering the precise actuation required for enacting the expert policy.
As a result, the quantile discretization approach achieves a return of more than twice that of the uniform discretization.\blfootnote{
    $^1$More attention visualizations can be found at 
    \href{https://trajectory-transformer.github.io/attention}
    {\texttt{trajectory-transformer.github.io/attention}}
}

\begin{figure}[t]
\centering
\includegraphics[width=0.85\linewidth]{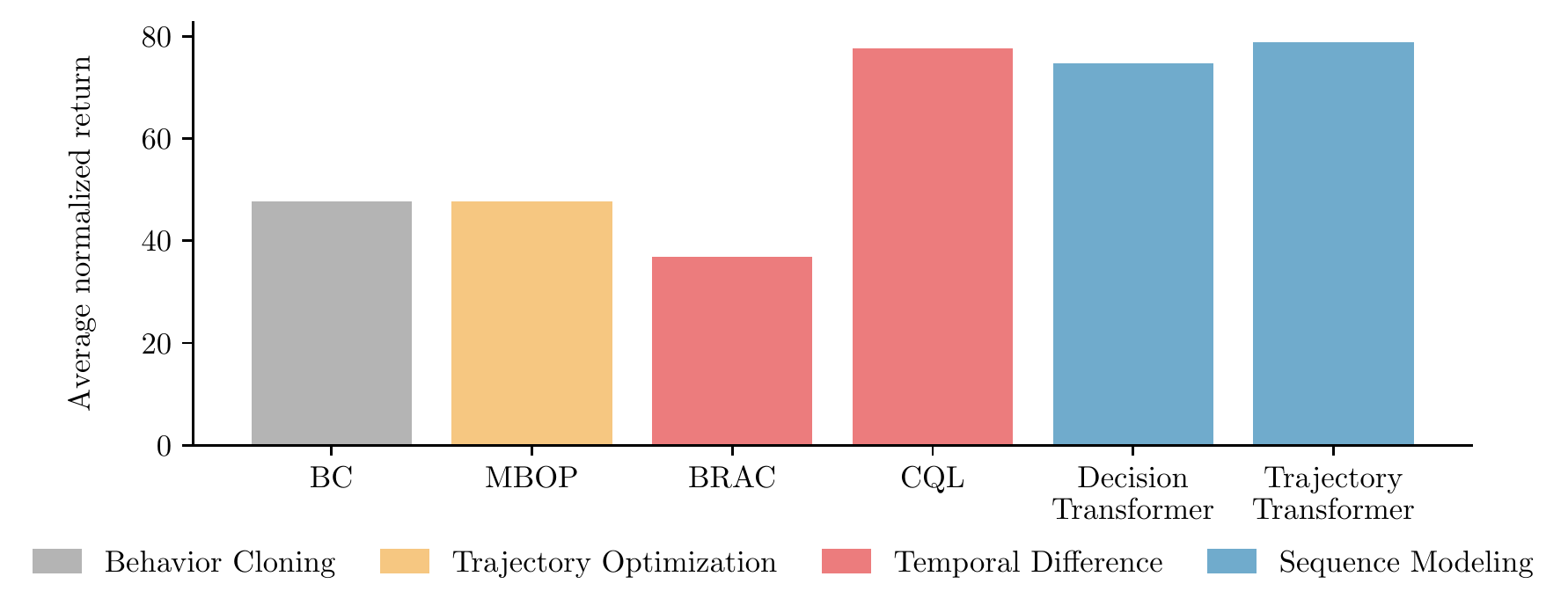}

\caption{
\textbf{(Offline locomotion performance)}
A plot showing the average per-algorithm performance in Table~\ref{table:d4rl}, with bars colored according to a crude algorithm categorization.
In this plot, ``Trajectory Transformer'' refers to the quantile discreization variant.
}

\label{fig:offline}
\end{figure}

\paragraph{Combining with $Q$-functions.}
Though Monte Carlo value estimates are sufficient for many standard offline reinforcement learning benchmarks, in sparse-reward and long-horizon settings they become too uninformative to guide the beam-search-based planning procedure.
In these problems, the value estimate from the Transformer can be replaced with a $Q$-function trained via dynamic programming.
We explore this combination by using the $Q$-function from the implicit $Q$-learning algorithm (IQL; \citealt{kostrikov2021implicit}) on the AntMaze navigation tasks \citep{fu2020d4rl}, for which there is only a sparse reward upon reaching the goal state.
These tasks evaluate temporal compositionality because they require stitching together multiple zero-reward trajectories in the dataset to reach a designated goal.

AntMaze results are provided in Table~\ref{table:antmaze}.
$Q$-guided Trajectory Transformer planning outperforms all prior methods on all maze sizes and dataset compositions.
In particular, it outperforms the IQL method from which we obtain the $Q$-function, underscoring that planning with a $Q$-function as a search heuristic can be less susceptible to errors in the $Q$-function than policy extraction.
However, because the $Q$-guided planning procedure still benefits from the temporal compositionality of both dynamic programming and planning, it outperforms return-conditioning approaches, such as the Decision Transformer, that suffer due to the lack of complete demonstrations in the AntMaze datasets.

\begin{table*}[t]
\centering
\small
\begin{tabular*}{\textwidth}{@{\extracolsep{\fill}}llrrrrrr}
\toprule
\multicolumn{1}{c}{\textbf{Dataset}} & \multicolumn{1}{c}{\textbf{Environment}} &
    \multicolumn{1}{c}{\textbf{BC}} &
    \multicolumn{1}{c}{\textbf{CQL}} &
    \multicolumn{1}{c}{\textbf{IQL}} &
    \multicolumn{1}{c}{\textbf{DT}} &
    \multicolumn{1}{c}{\textbf{TT \scriptsize{\textcolor{cgrey}{$\bm{(+Q)}$}}}} \\ 
\midrule
Umaze & AntMaze & 
    $54.6$ &
    $74.0$ &
    $87.5$  &
    $59.2$ &
    $100.0$ \scriptsize{\raisebox{1pt}{$\pm 0.0$}} \\ 
Medium-Play & AntMaze &
    $0.0$ &
    $61.2$ &
    $71.2$ &
    $0.0$ &
    $93.3$ \scriptsize{\raisebox{1pt}{$\pm 6.4$}} \\ 
Medium-Diverse & AntMaze &
    $0.0$ &
    $53.7$ &
    $70.0$ &
    $0.0$ &
    $100.0$ \scriptsize{\raisebox{1pt}{$\pm 0.0$}} \\
Large-Play & AntMaze &
    $0.0$ &
    $15.8$ &
    $39.6$ &
    $0.0$ &
    $66.7$ \scriptsize{\raisebox{1pt}{$\pm 12.2$}} \\ 
Large-Diverse & AntMaze &
    $0.0$ &
    $14.9$ &
    $47.5$ &
    $0.0$ &
    $60.0$ \scriptsize{\raisebox{1pt}{$\pm 12.7$}} \\
\midrule
\multicolumn{2}{c}{\textbf{Average}} & 10.9 & 44.9 & 63.2 & 11.8 & 84.0 \hspace{.6cm} \\ 
\bottomrule
\end{tabular*}
\caption{
\textbf{(Combining with $Q$-functions)}
Performance on the sparse-reward AntMaze (\texttt{v0}) navigation task.
Using a $Q$-function as a search heuristic with the Trajectory Transformer ({TT~\scriptsize{\textcolor{cgrey}{$\bm{(+Q)}$}}}) outperforms policy extraction from the $Q$-function (IQL) and return-conditioning approaches like the Decision Transformer (DT).
We report means and standard error over 15 random seeds for {TT \scriptsize{\textcolor{cgrey}{$\bm{(+Q)}$}}}; baseline results are taken from \cite{kostrikov2021implicit}.
}

\label{table:antmaze}
\end{table*}

\paragraph{Imitation and goal-reaching.}
We additionally plan with the Trajectory Transformer using standard likelihood-maximizing beam search, as opposed to the return-maximizing version used for offline RL.
We find that after training the model on datasets collected by expert policies \citep{fu2020d4rl}, using beam search as a receding-horizon controller achieves an average normalized return of $104\%$ and $109\%$ in the Hopper and Walker2d environments, respectively, using the same evaluation protocol of 15 runs described as in the offline reinforcement learning results.
While this result is perhaps unsurprising, as behavior cloning with standard feedforward architectures is already able to reproduce the behavior of the expert policies, it demonstrates that a decoding algorithm used for language modeling can be effectively repurposed for control.

Finally, we evaluate the goal-reaching variant of beam-search, which conditions on a future desired state alongside previously encountered states.
We use a continuous variant of the classic four rooms environment as a testbed \citep{sutton1999semimdps}.
Our training data consists of trajectories collected by a pretrained goal-reaching agent, with start and goal states sampled uniformly at random across the state space.
Figure~\ref{fig:goals} depicts routes taken by the the planner.
Anti-causal conditioning on a future state allows for beam search to be used as a goal-reaching method.
No reward shaping, or rewards of any sort, are required; the planning method relies entirely on goal relabeling.
An extension of this experiment to procedurally-generated maps is described in Appendix~\ref{app:minigrid}.

\begin{figure}[t]
    \centering
    \includegraphics[width=0.325\linewidth]{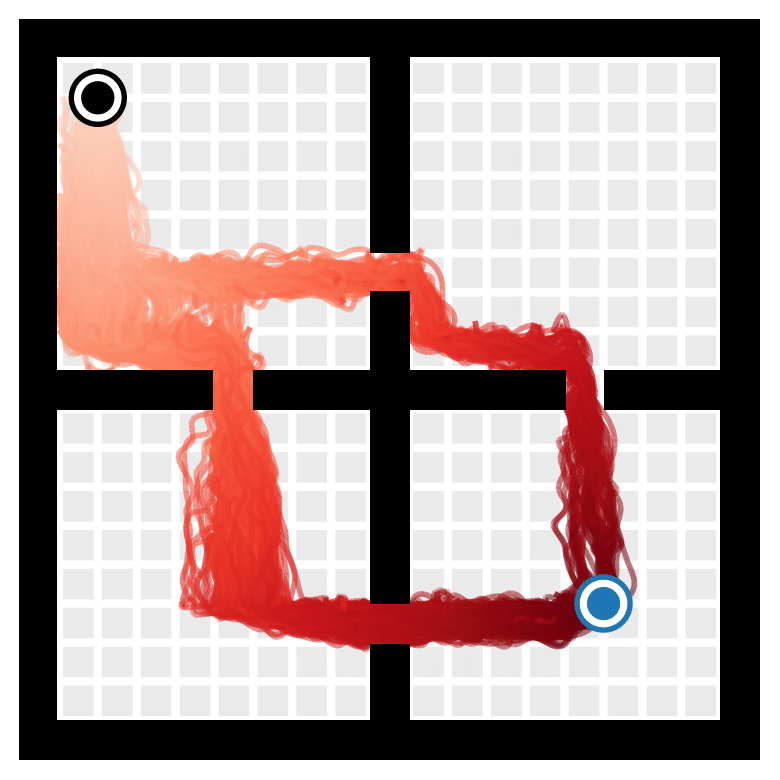}
    \includegraphics[width=0.325\linewidth]{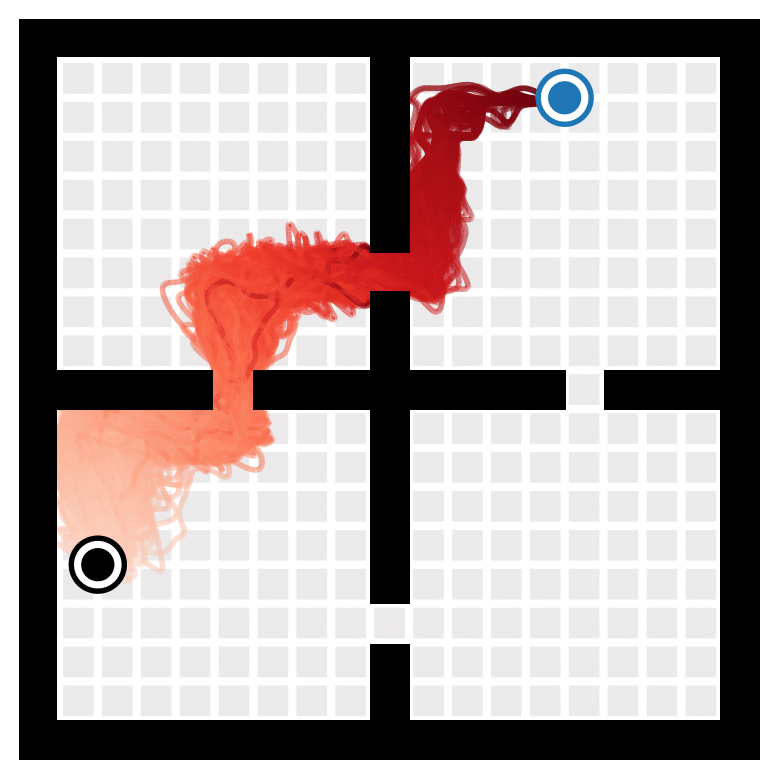}
    \includegraphics[width=0.325\linewidth]{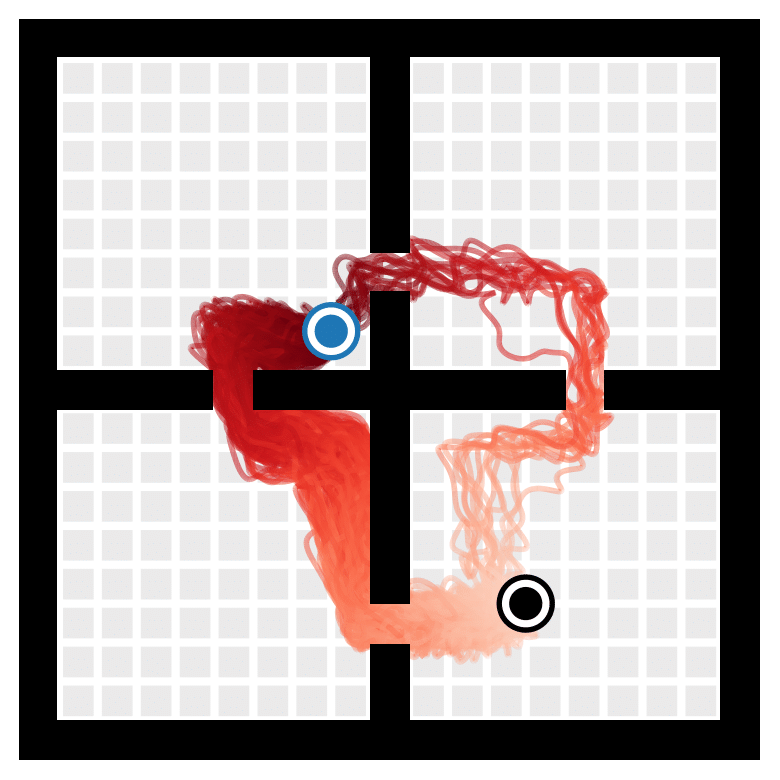}
    \caption{
    \textbf{(Goal-reaching)}
    Trajectories collected by TTO with anti-causal goal-state conditioning in a continuous variant of the four rooms environment.
    Trajectories are visualized as curves passing through all encountered states, with color becoming more saturated as time progresses.
    Note that these curves depict real trajectories collected by the controller and not sampled sequences.
    The starting state is depicted by
    \protect{\includegraphics[height=.35cm]{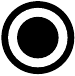}}
    and the goal state by
    \protect{\includegraphics[height=.35cm]{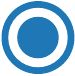}}.
    Best viewed in color.
    }
    \label{fig:goals}
\end{figure}

\section{Discussion}
\label{sec:discussion}

We have presented a sequence modeling view on reinforcement learning that enables us to derive a single algorithm for a diverse range of problem settings, unifying many of the standard components of reinforcement learning algorithms (such as policies, models, and value functions) under a single sequence model. The algorithm involves training a sequence model jointly on states, actions, and rewards and sampling from it using a minimally modified beam search.
Despite drawing from the tools of large-scale language modeling instead of those normally associated with control, we find that this approach is effective in imitation learning, goal-reaching, and offline reinforcement learning.

However, prediction with Transformers is currently slower and more resource-intensive than prediction with the types of single-step models often used in model-based control, requiring up to multiple seconds for action selection when the context window grows too large.
This precludes real-time control with standard Transformers for most dynamical systems.
While the beam-search-based planner is conceptually an instance of model-predictive control, and as such could be applicable wherever model-based RL is, in practice the slow planning also makes online RL experiments unwieldy.
(Computationally-efficient Transformer architectures \citep{tay2021long} could potentially cut runtimes down substantially.)
Further, we have chosen to discretize continuous data to fit a standard architecture instead of modifying the architecture to handle continuous inputs.
While we found this design to be much more effective than conventional continuous dynamics models, it does in principle impose an upper bound on prediction precision.

The effectiveness of the Trajectory Transformer stems largely from its improved predictive accuracy compared to dynamics models used in prior model-based reinforcement learning algorithms;
its weaknesses are largely consequences of the beam search-based planner in which it is embedded.
In the next chapter, we will consider whether it is possible to retain the strengths of the Trajectory Transformer without suffering from its limitations by building a planning algorithm from the ground up around the affordances of a specific generative model, as opposed to using an improved generative model in a generic search algorithm for model-predictive control.

\chapter{Planning with Diffusion}
\renewcommand{\chaptertitle}{Planning with Diffusion}
\label{ch:diffuser}

\section{Introduction}
\label{sec:intro}

Planning with a learned model is a conceptually simple framework for reinforcement learning and data-driven decision-making.
Its appeal comes from employing learning techniques only where they are the most mature and effective: for the approximation of unknown environment dynamics in what amounts to a supervised learning problem.
Afterwards, the learned model may be plugged into conventional search algorithms \citep{reddy1997beam,kocsis2006mcts} or trajectory optimization routines \cite{tassa2012synthesis,posa2014direct,matthew2017trajectory},
as demonstrated with the Trajectory Transformer in Chapter~\ref{ch:transformer}.

However, this combination has a number of shortcomings.
Because powerful trajectory optimizers can exploit learned models, plans generated by better optimizers often look more like adversarial examples than optimal trajectories \cite{talvitie2014hallcuinated,ke2018modeling}.
As a result, contemporary model-based reinforcement learning algorithms often inherit more from model-free methods, such as value functions and policy gradients \cite{wang2019benchmarking}, than from the trajectory optimization toolbox.
Those methods that do rely on online planning tend to use simple gradient-free trajectory optimization routines like random shooting \cite{nagabandi2018mbmf} or the cross-entropy method \cite{botev2013cem,chua2018pets} to avoid the aforementioned issues.
Because these problems stem from the limitations of the planning algorithm, they cannot be fully overcome by improved model quality.

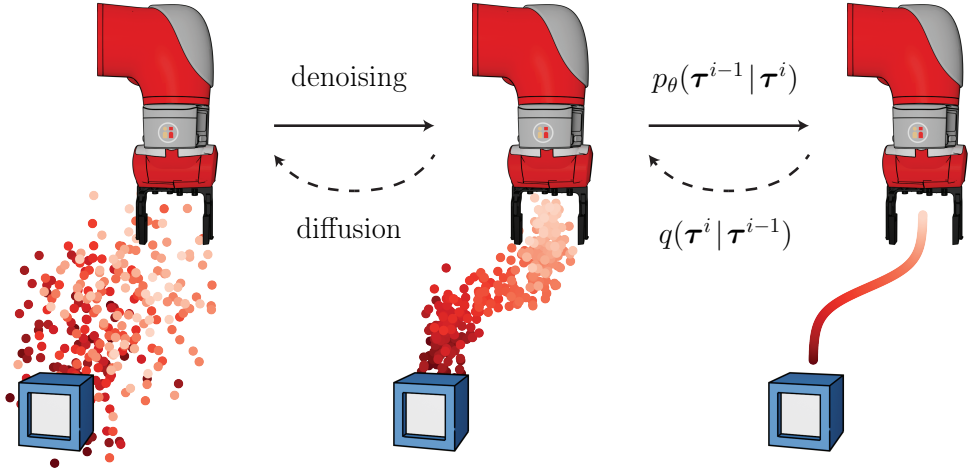
\begin{figure}
\centering
\resizebox{0.8\columnwidth}{!}{
    \begin{tikzpicture}
        \node[anchor=south west,inner sep=0] (image) at (0,0)
            {\includegraphics[width=\columnwidth]{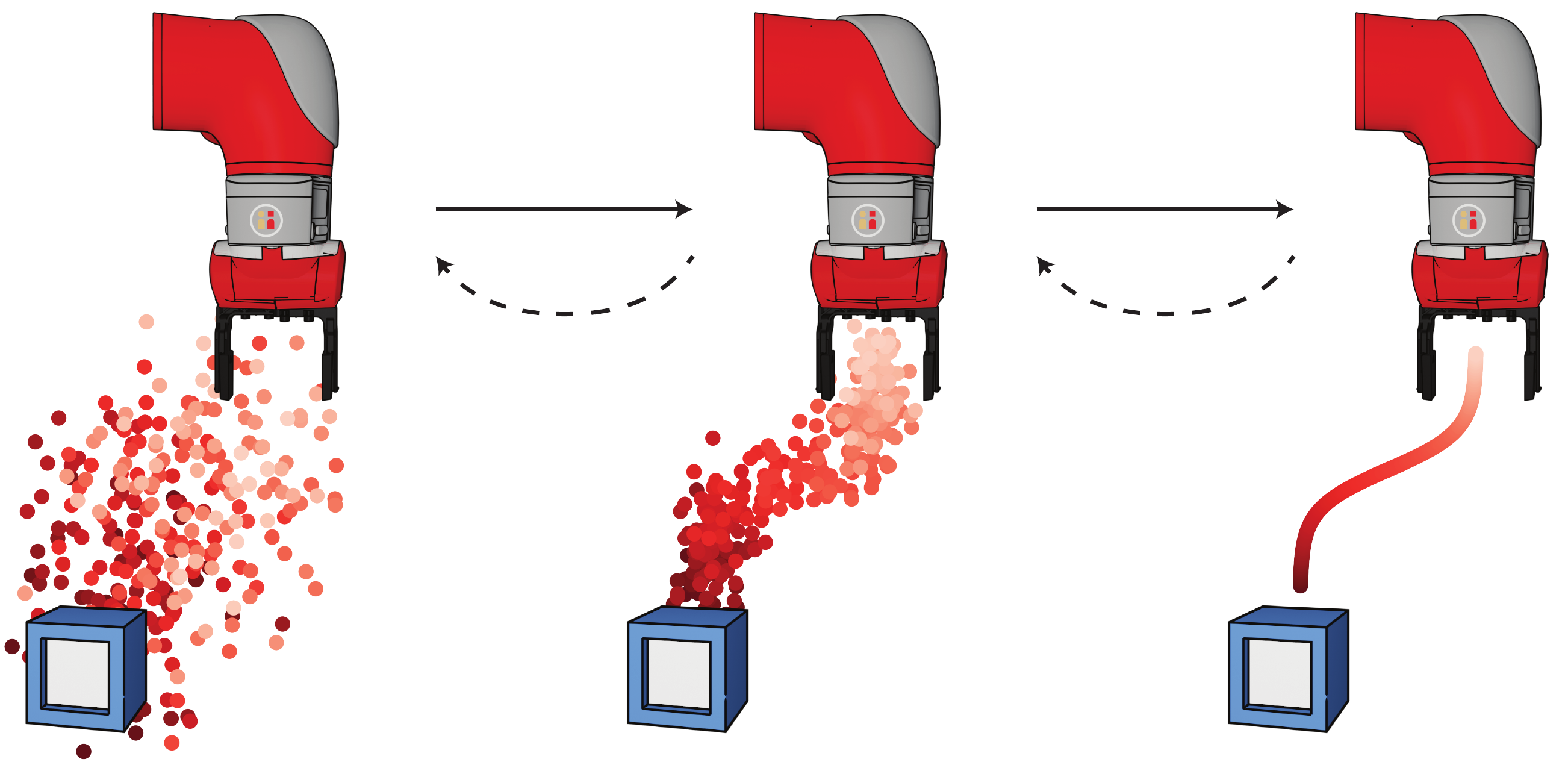}};
        \begin{scope}[x={(image.south east)},y={(image.north west)}]
            \node[] at (.74,.825) {\large $p_\theta(\btau{i-1} \!\mid\! \btau{i})$};
            \node[] at (.74,.5) {\large $q(\btau{i} \!\mid\! \btau{i-1})$};
            \node[] at (.355,.825) {\large denoising};
            \node[] at (.355,.515) {\large diffusion};
        \end{scope}
    \end{tikzpicture}
}
\caption{
    \textbf{(Planning via denoising)}
    \method{} is a diffusion probabilistic model that plans by iteratively refining trajectories.
}
\label{fig:teaser}
\end{figure}

In this chapter, we propose an alternative approach to data-driven trajectory optimization.
The core idea is to train a model that is directly amenable to trajectory optimization, in the sense that sampling from the model and planning with it become nearly identical.
This goal requires a shift in how the model is designed.
Because learned dynamics models are normally meant to be proxies
for environment dynamics, improvements are often achieved by structuring the model according to the underlying causal process \cite{bapst2019structured}.
Instead, we consider how to design a model in line with the planning problem in which it will be used.
For example, because the model will ultimately be used for planning, action distributions are just as important as state dynamics and long-horizon accuracy is more important than single-step error.
On the other hand, the model should remain agnostic to reward function so that it may be used in multiple tasks, including those unseen during training.
Finally, the model should be designed so that its plans, and not just its predictions, improve with experience and are resistant to the myopic failure modes of standard shooting-based planning algorithms.

\begin{figure}
\centering
\centering
\includegraphics[width=0.75\columnwidth]{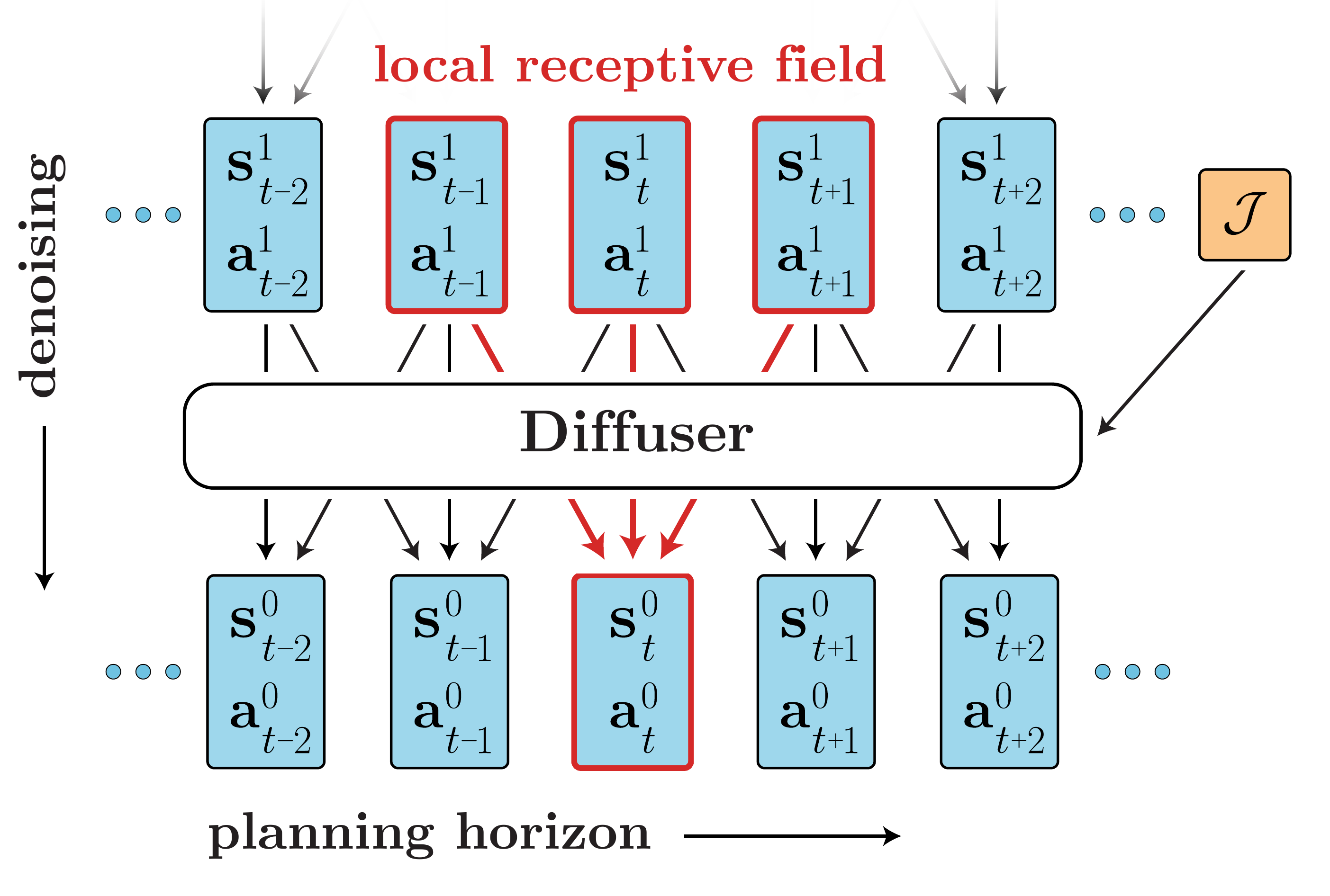}
\caption{
    \textbf{(Diffuser block architecture)}
    Diffuser samples plans by iteratively denoising two-dimensional arrays consisting of a variable number of state-action pairs.
    A small receptive field constrains the model to only enforce local consistency during a single denoising step.
    By composing many denoising steps together, local consistency can drive global coherence of a sampled plan.
    An optional guide function $\mathcal{J}$ can be used to bias plans toward those optimizing a test-time objective or satisfying a set of constraints.
}
\label{fig:architecture}
\end{figure}

We instantiate this idea as a trajectory-level diffusion probabilistic model \citep{sohldickstein2015nonequilibrium,ho2020denoising} called \method{}, visualized in Figure~\ref{fig:architecture}.
Whereas standard model-based planning techniques predict forward in time autoregressively, \method{} predicts all timesteps of a plan simultaneously.
The iterative sampling process of diffusion models leads to flexible conditioning, allowing for auxiliary guides to modify the sampling procedure to recover trajectories with high return or satisfying a set of constraints.
This formulation of data-driven trajectory optimization has several appealing properties:

\textbf{Long-horizon scalability~~}
Diffuser is trained for the accuracy of its generated trajectories rather than its single-step error, so it does not suffer from the compounding rollout errors of single-step dynamics models and scales more gracefully with respect to long planning horizon.

\textbf{Task compositionality~~}
Reward functions provide auxiliary gradients to be used while sampling a plan, allowing for a straightforward way of planning by composing multiple rewards simultaneously by adding together their gradients.

\textbf{Temporal compositionality~~}
Diffuser generates globally coherent trajectories by iteratively improving local consistency, allowing it to generalize to novel trajectories by stitching together in-distribution subsequences.

\textbf{Effective non-greedy planning~~}
By blurring the line between model and planner, the training procedure that improves the model's predictions also has the effect of improving its planning capabilities.
This design yields a learned planner that can solve the types of long-horizon, sparse-reward problems that prove difficult for many conventional planning methods.

The core contribution of this chapter is a denoising diffusion model designed for trajectory data and an associated probabilistic framework for behavior synthesis.
While unconventional compared to the types of models routinely used in deep model-based reinforcement learning, the unusual properties of \method make it particularly effective in control settings that require long-horizon reasoning and test-time flexibility.

\section{Background on Diffusion Probabilistic Models}

\label{sec:background_diffusion}
Diffusion probabilistic models \citep{sohldickstein2015nonequilibrium,ho2020denoising} pose the data-generating process as an iterative denoising procedure $p_\theta(\btau{i-1} \mid \btau{i})$.
This denoising is the reverse of a forward diffusion process $q(\btau{i} \mid \btau{i-1})$ that slowly corrupts the structure in data by adding noise.
The data distribution induced by the model is given by:
\[
p_\theta(\btau{0}) =
\int p(\btau{N}) \prod_{i=1}^{N} p_\theta(\btau{i-1} \mid \btau{i}) \mathrm{d} \btau{1:N}
\]
where $p(\btau{N})$ is a standard Gaussian prior and $\btau{0}$ denotes (noiseless) data.
Parameters $\theta$ are optimized by minimizing a variational bound on the negative log likelihood of the reverse process:
$
\theta^* = \argmin_{\theta}
-\expect{\btau{0}}{\log p_\theta(\btau{0})}.
$
The reverse process is often parameterized as Gaussian with fixed timestep-dependent covariances:
\[
p_\theta(\btau{i-1} \mid \btau{i}) = \mathcal{N}(\btau{i-1} \mid \mu_\theta(\btau{i}, i), \Sigma^i).
\]
The forward process $q(\btau{i} \mid \btau{i-1})$ is typically prespecified.

\textbf{Notation.}~~
There are two ``times" at play in this work: that of the diffusion process and that of the planning problem.
We use superscripts ($i$ when unspecified) to denote diffusion timestep and subscripts ($t$ when unspecified) to denote planning timestep.
For example, $\st^{0}$ refers to the $t^{\text{th}}$ state in a noiseless trajectory.
When it is unambiguous from context, superscripts of noiseless quantities are omitted: $\btau{} = \btau{0}$.
We overload notation slightly by referring to the $t^\text{th}$ state (or action) in a trajectory $\btau{}$ as $\btau{}_{\st}$ (or $\btau{}_{\at}$).

\section{Planning with Diffusion}
\label{sec:method}

A major obstacle to using trajectory optimization techniques is that they require knowledge of the environment dynamics.
Most learning-based methods attempt to overcome this obstacle by training an approximate dynamics model and plugging it in to a conventional planning routine.
However, learned models are often poorly suited to the types of planning algorithms designed with ground-truth models in mind, leading to planners that exploit learned models by finding adversarial examples.

We propose a tighter coupling between modeling and planning.
Instead of using a learned model in the context of a classical planner, we subsume as much of the planning process as possible into the generative modeling framework, such that planning becomes nearly identical to sampling.
We do this using a diffusion model of trajectories, $p_\theta(\btau{})$.
The iterative denoising process of a diffusion model lends itself to flexible conditioning by way of sampling from perturbed distributions of the form:

\begin{equation}
\label{eq:perturbed}
\tilde{p}_\theta(\btau{}) \propto p_\theta(\btau{}) h(\btau{}).
\end{equation}

The function $h(\btau{})$ can contain information about prior evidence (such as an observation history), desired outcomes (such as a goal to reach), or general functions to optimize (such as rewards or costs).
Performing inference in this perturbed distribution can be seen as a probabilistic analogue to the trajectory optimization problem posed in Section~\ref{ch:preliminaries}, as it requires finding trajectories that are both physically realistic under $p_\theta(\btau{})$ and high-reward (or constraint-satisfying) under $h(\btau{})$.
Because the dynamics information is separated from the perturbation distribution $h(\btau{})$, a single diffusion model $p_\theta(\btau{})$ may be reused for multiple tasks in the same environment.

In this section, we describe \method, a diffusion model designed for learned trajectory optimization.
We then discuss two specific instantiations of planning with \method, realized as reinforcement learning counterparts to classifier-guided sampling and image inpainting.

\subsection{A Generative Model for Trajectory Planning}
\label{sec:model}

\textbf{Temporal ordering.}~~
Blurring the line between sampling from a trajectory model and planning with it yields an unusual constraint: we can no longer predict states autoregressively in temporal order.
Consider the goal-conditioned inference ${p(\bs_1 \mid \bs_0, \bs_T)}$;
the next state $\bs_1$ depends on a \emph{future} state as well as a prior one.
This example is an instance of a more general principle:
while dynamics prediction is causal, in the sense that the present is determined by the past, decision-making and control can be anti-causal, in the sense that decisions in the present are conditional on the future.\footnote{
    In general reinforcement learning contexts, conditioning on the future emerges from the assumption of future optimality for the purpose of writing a dynamic programming recursion. Concretely, this appears as the future optimality variables $\mathcal{O}_{t:T}$ in the action distribution $\log p(\at \mid \st, \mathcal{O}_{t:T})$ (Levine, 2018)\nocite{levine2018reinforcement}.
}
Because we cannot use a temporal autoregressive ordering, we design \method to predict all timesteps of a plan concurrently.

\textbf{Temporal locality.}~~
Despite not being autoregressive or Markovian, \method features a relaxed form of temporal locality.
In Figure~\ref{fig:architecture}, we depict a dependency graph for a diffusion model consisting of a single temporal convolution.
The receptive field of a given prediction only consists of nearby timesteps, both in the past and the future.
As a result, each step of the denoising process can only make predictions based on local consistency of the trajectory.
By composing many of these denoising steps together, however, local consistency can drive global coherence.

\textbf{Trajectory representation.}~~
\method is a model of trajectories designed for planning, meaning that the effectiveness of the controller derived from the model is just as important as the quality of the state predictions.
As a result, states and actions in a trajectory are predicted jointly; for the purposes of prediction the actions are simply additional dimensions of the state.
Specifically, we represent inputs (and outputs) of \method as a two-dimensional array:

\begin{align}
\label{eq:trajectory_array}
\btau{} = \begin{bmatrix}
    \bs_{0} & \bs_{1} & \raisebox{-.2cm}{\ldots} & \bs_{T} \\
    \ba_{0} & \ba_{1} & & \ba_{T}
\end{bmatrix}.
\end{align}
with one column per timestep of the planning horizon.

\textbf{Architecture.}~~
We now have the ingredients needed to specify a \method architecture:
\textbf{(1)} an entire trajectory should be predicted non-autoregressively,
\textbf{(2)} each step of the denoising process should be temporally local, and
\textbf{(3)} the trajectory representation should allow for equivariance  along one dimension (the planning horizon) but not the other (the state and action features).
We satisfy these criteria with a model consisting of repeated (temporal) convolutional residual blocks.
The overall architecture resembles the types of U-Nets that have found success in image-based diffusion models, but with two-dimensional spatial convolutions replaced by one-dimensional temporal convolutions (Figure~\ref{fig:app_architecture}).
Because the model is fully convolutional, the horizon of the predictions is determined not by the model architecture, but by the input dimensionality; it can change dynamically during planning if desired.

\textbf{Training.}~~
We use \method to parameterize a learned gradient $\epsilon_\theta(\btau{i}, i)$ of the trajectory denoising process, from which the mean $\mu_\theta$ can be solved in closed form \citep{ho2020denoising}.
We use the simplified objective for training the $\epsilon$-model, given by:
\[
\mathcal{L}(\theta) = \expect{
    i,\epsilon,\btau{0}
    }{
    \lVert \epsilon - \epsilon_\theta(\btau{i}, i) \rVert^2
    },
\]
in which $i \sim \mathcal{U}\{1, 2, \ldots, N\}$ is the diffusion timestep, $\epsilon \sim \mathcal{N}(\bm{0}, \bm{I})$ is the noise target, and $\btau{i}$ is the trajectory $\btau{0}$ corrupted with noise $\epsilon$.
Reverse process covariances $\Sigma^i$ follow the cosine schedule of \citet{nichol2021improved}.

\subsection{Reinforcement Learning as Guided Sampling}
\label{sec:guided}

In order to solve reinforcement learning problems with \method, we must introduce a notion of reward.
We appeal to the control-as-inference graphical model \citep{levine2018reinforcement} to do so.
Let $\mathcal{O}_{t}$ be a binary random variable denoting the optimality of timestep $t$ of a trajectory, with ${p(\mathcal{O}_t=1) = \exp(r(\st, \at))}$. We can sample from the set of optimal trajectories by setting ${h(\btau{}) = p(\mathcal{O}_{1:T} \mid \btau{})}$ in Equation~\ref{eq:perturbed}:

\[
\tilde{p}_\theta(\btau{}) = p(\btau{} \mid \mathcal{O}_{1:T}=1) \propto p(\btau{}) p(\mathcal{O}_{1:T}=1 \mid \btau{}).
\]

We have exchanged the reinforcement learning problem for one of \emph{conditional sampling}.
Thankfully, there has been much prior work on conditional sampling with diffusion models.
While it is intractable to sample from this distribution exactly, when $p(\mathcal{O}_{1:T} \mid \btau{i})$ is sufficiently smooth, the reverse diffusion process transitions can be approximated as Gaussian \citep{sohldickstein2015nonequilibrium}:
\begin{equation}
\label{eq:guided}
p_\theta(\btau{i-1} \mid \btau{i}, \opt_{1:T}) \approx \mathcal{N}(\btau{i-1}; \mu + \Sigma g, \Sigma)
\end{equation}
where $\mu, \Sigma$ are the parameters of the original reverse process transition $p_\theta(\btau{i-1} \mid \btau{i})$ and

\begin{align*}
g &= \nabla_{\btau{}} \log p(\opt_{1:T} \mid \btau{}) |_{\btau{} = \mu} \\
&= \sum_{t=0}^{T} \nabla_{\st,\at} r(\st, \at) |_{(\st,\at)=\mu_t}
= \nabla \mathcal{J}(\mu).
\end{align*}

This relation provides a straightforward translation between classifier-guided sampling, used to generate class-conditional images \citep{dhariwal2021diffusion}, and the reinforcement learning problem setting.
We first train a diffusion model $p_\theta(\btau{})$ on the states and actions of all available trajectory data.
We then train a separate model $\mathcal{J}_\phi$ to predict the cumulative rewards of trajectory samples $\btau{i}$.
The gradients of $\mathcal{J}_\phi$ are used to guide the trajectory sampling procedure by modifying the means $\mu$ of the reverse process according to Equation~\ref{eq:guided}.
The first action of a sampled trajectory ${\btau{} \sim p(\btau{} \mid \mathcal{O}_{1:T}=1)}$ may be executed in the environment, after which the planning procedure begins again in a standard receding-horizon control loop.
Pseudocode for the guided planning method is given in Algorithm~\ref{alg:rl}.

\def\NoNumber#1{{\def\alglinenumber##1{}\STATE #1}\addtocounter{ALG@line}{-1}}

{\centering
\begin{figure}[t]
\begin{minipage}{\linewidth}
  \begin{algorithm}[H]
    \caption{Guided Diffusion Planning}
    \label{alg:rl}
    \begin{algorithmic}[1]
    \STATE \textbf{Require} Diffuser $\mu_\theta$, guide $\mathcal{J}$, scale $\alpha$, covariances $\Sigma^i$ \\
    \WHILE{not done}
        \STATE Observe state $\bs$; initialize plan  $\btau{N} \sim \mathcal{N}(\bm{0}, \bm{I})$ \\
        \FOR{$i = N, \ldots, 1$}
            \STATE \small{\color{gray}\te{// parameters of reverse transition}} \\
            \STATE $\mu \gets \mu_\theta(\btau{i})$ \\
            \STATE \small{\color{gray}\te{// guide using gradients of return}} 
            \STATE $\btau{i-1} \sim \mathcal{N}(\mu + \alpha \Sigma \nabla \mathcal{J}(\mu), \Sigma^i)$\\
            \STATE \small{\color{gray}\te{// constrain first state of plan}}
            \STATE $\btau{i-1}_{\vs_0} \gets \vs$ \hspace{0.15cm} \\
        \ENDFOR
        \STATE Execute first action of plan $\btau{0}_{\ba_0}$
    \ENDWHILE
    \end{algorithmic}
  \end{algorithm}
\end{minipage}
\end{figure}
}

\subsection{Goal-Conditioned Reinforcement Learning as Inpainting}
\label{sec:inpainting}

Some planning problems are more naturally posed as constraint satisfaction than reward maximization.
In these settings, the objective is to produce any feasible trajectory that satisfies a set of constraints, such as terminating at a goal location.
Appealing to the two-dimensional array representation of trajectories described by Equation~\ref{eq:trajectory_array}, this setting can be translated into an \emph{inpainting problem}, in which state and action constraints act analogously to observed pixels in an image \citep{sohldickstein2015nonequilibrium}.
All unobserved locations in the array must be filled in by the diffusion model in a manner consistent with the observed constraints.

The perturbation function required for this task is a Dirac delta for observed values and constant elsewhere.
Concretely, if $\mathbf{c}_t$ is state constraint at timestep $t$, then

\[
    h(\btau{}) = \delta_{\mathbf{c}_t}(\bs_0,\ba_0,\ldots,\bs_T,\ba_T) =
\begin{cases}
    +\infty      & \text{if } \mathbf{c}_t = \st\\
    ~~~~~0      & \text{otherwise}
\end{cases}
\]

The definition for action constraints is identical.
In practice, this may be implemented by sampling from the unperturbed reverse process ${\btau{i-1} \sim p_\theta(\btau{i-1} \mid \btau{i})}$ and replacing the sampled values with conditioning values $\mathbf{c}_t$ after all diffusion timesteps $i \in \{0, 1, \ldots, N\}$.

Even reward maximization problems require conditioning-by-inpainting because all sampled trajectories should begin at the current state.
This conditioning is described by line 10 in Algorithm~\ref{alg:rl}.

\section{Properties of Diffusion Planners}
\label{sec:properties}

We discuss a number of \method's important properties, focusing on those that are are either distinct from standard dynamics models or unusual for non-autoregressive trajectory prediction.

\textbf{Learned long-horizon planning.}~~
Single-step models are typically used as proxies for ground-truth environment dynamics $p(\stp \mid \st, \at)$, and as such are not tied to any planning algorithm in particular.
In contrast, the planning routine in Algorithm~\ref{alg:rl} is closely tied to the specific affordances of diffusion models.
Because our planning method is nearly identical to sampling (with the only difference being guidance by a perturbation function $h(\btau{})$), \method's effectiveness as a long-horizon predictor directly translates to effective long-horizon planning.
We demonstrate the benefits of learned planning in a goal-reaching setting in \textbf{Figure~\ref{fig:properties}a}, showing that \method is able to generate feasible trajectories in the types of sparse reward settings where shooting-based approaches are known to struggle.
We explore a more quantitative version of this problem setting in Section~\ref{sec:maze2d}.

\vspace{.2em}
\textbf{Temporal compositionality.}~~
Single-step models are often motivated using the Markov property, allowing them to compose in-distribution transitions to generalize to out-of-distribution trajectories.
Because \method generates globally coherent trajectories by iteratively improving local consistency (Section~\ref{sec:model}), it can also stitch together familiar subsequences in novel ways.
In \textbf{Figure~\ref{fig:properties}b}, we train \method on trajectories that only travel in a straight line, and show that it can generalize to v-shaped trajectories by composing trajectories at their point of intersection.

\newcommand{\undersetimage}[3]{
    \begin{tikzpicture}
        \node[anchor=south west,inner sep=0] (image) at (0, 0)
            {\includegraphics[width=0.55cm]{#1}};
        \begin{scope}[x={(image.south east)},y={(image.north west)}]
            \node[] at (0.5, #3) {#2};
        \end{scope}
    \end{tikzpicture}
}

\begin{figure*}
\centering
\begin{minipage}[t]{0.65\textwidth}
    \raisebox{2cm}{\textbf{a}}~~~
    \hspace{.2cm}
    \includegraphics[width=2.6cm]{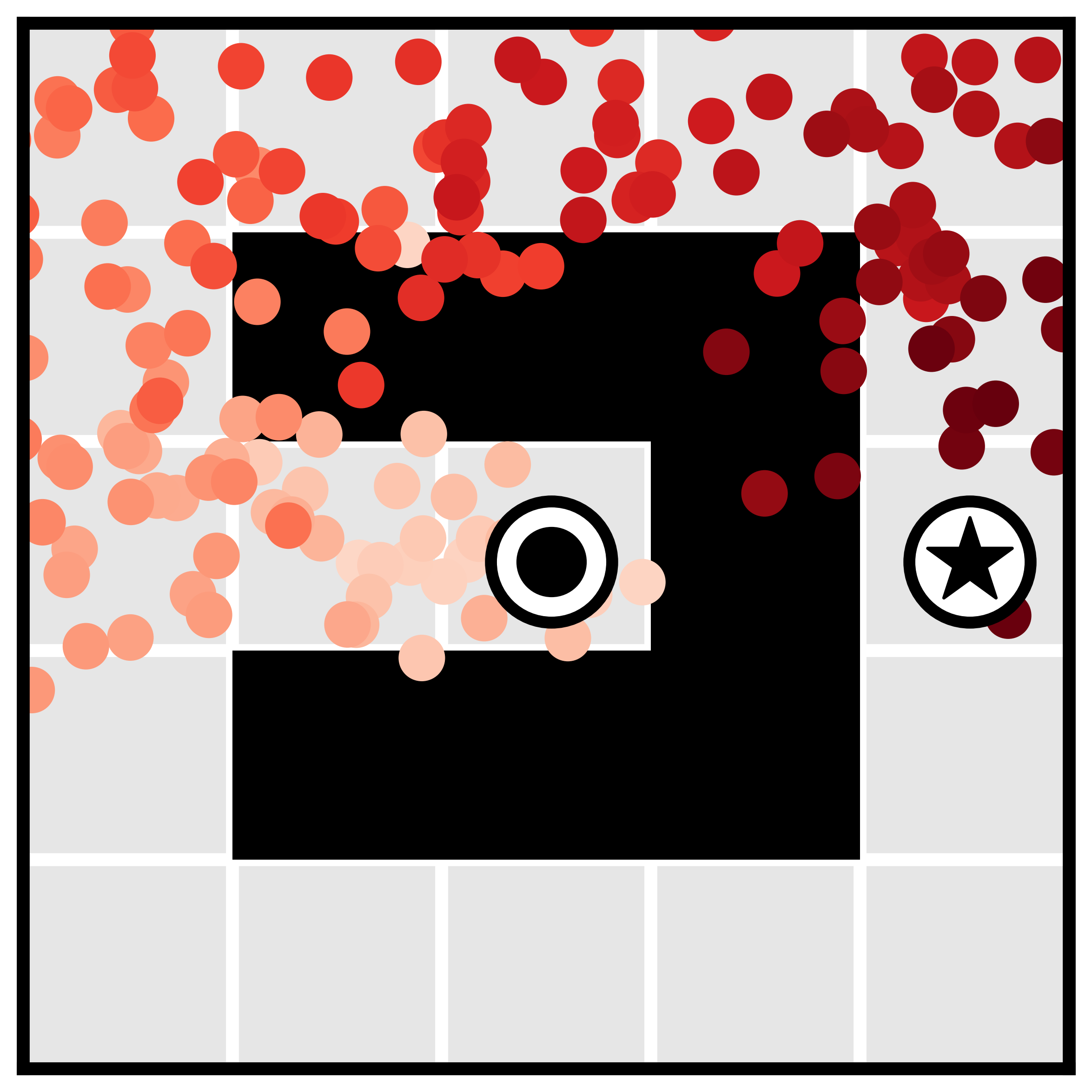}~
    \includegraphics[width=2.6cm]{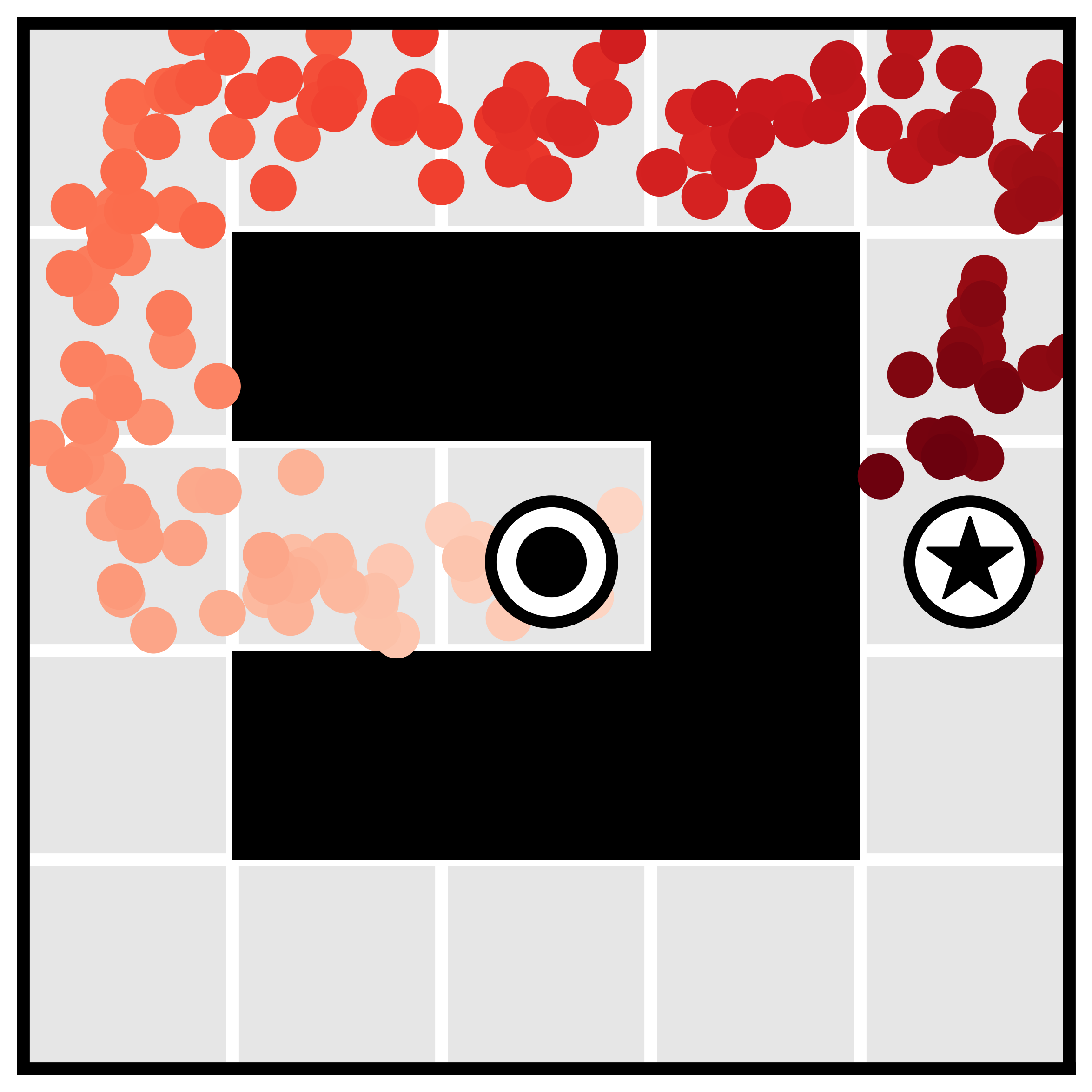}~
    \includegraphics[width=2.6cm]{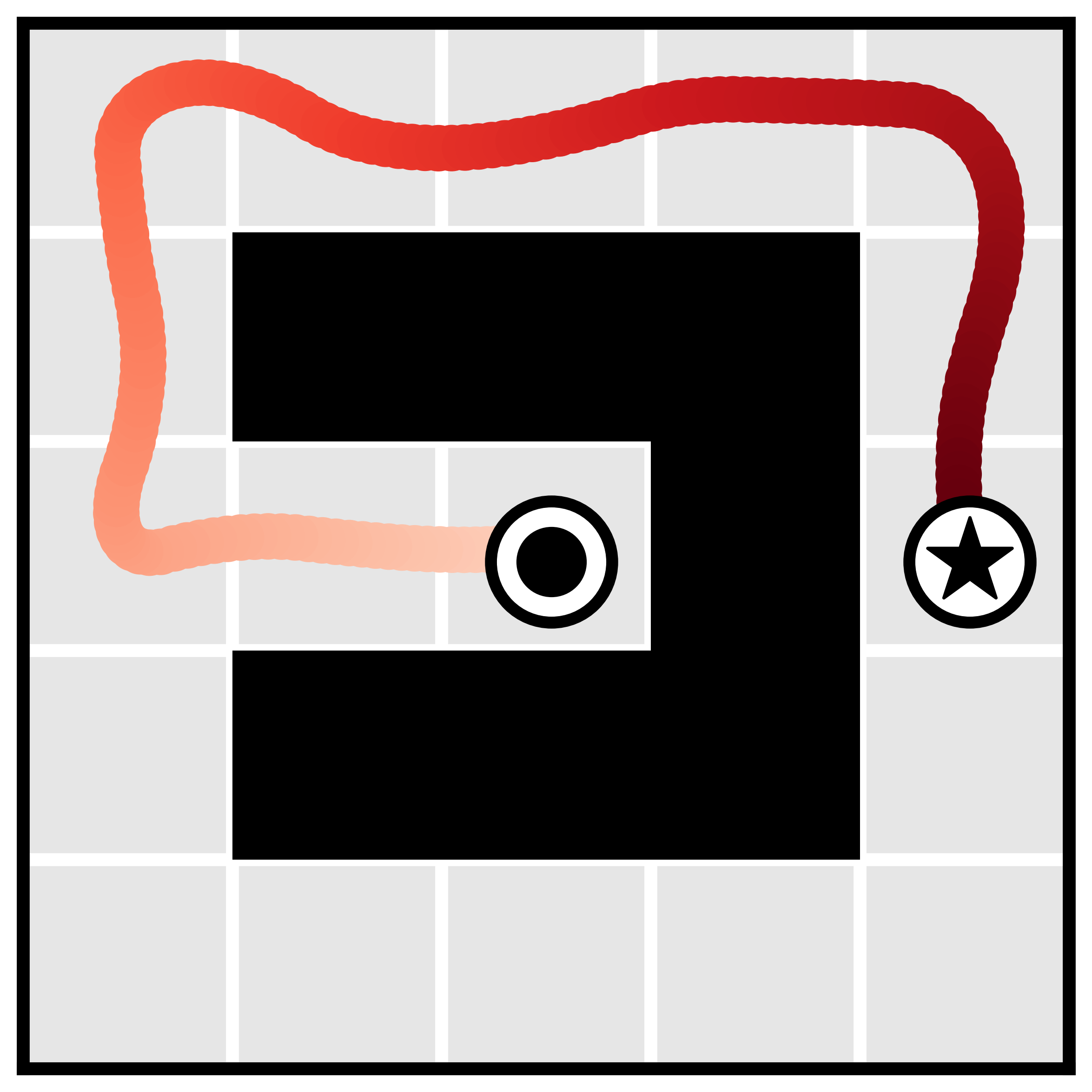} \\
    \vspace{-.4cm}
    \begin{center}
        \raisebox{.3\height}{denoising}~~
        \tikz \draw [-{Computer Modern Rightarrow[length=1.33mm, width=2mm]}, line width=.2mm] (0,0) -- (3,0); \\
    \end{center}
\end{minipage}%
\hfill
\begin{minipage}[t]{0.35\textwidth}
    \raisebox{2cm}{\textbf{b}}~~~
    \includegraphics[width=2.6cm]{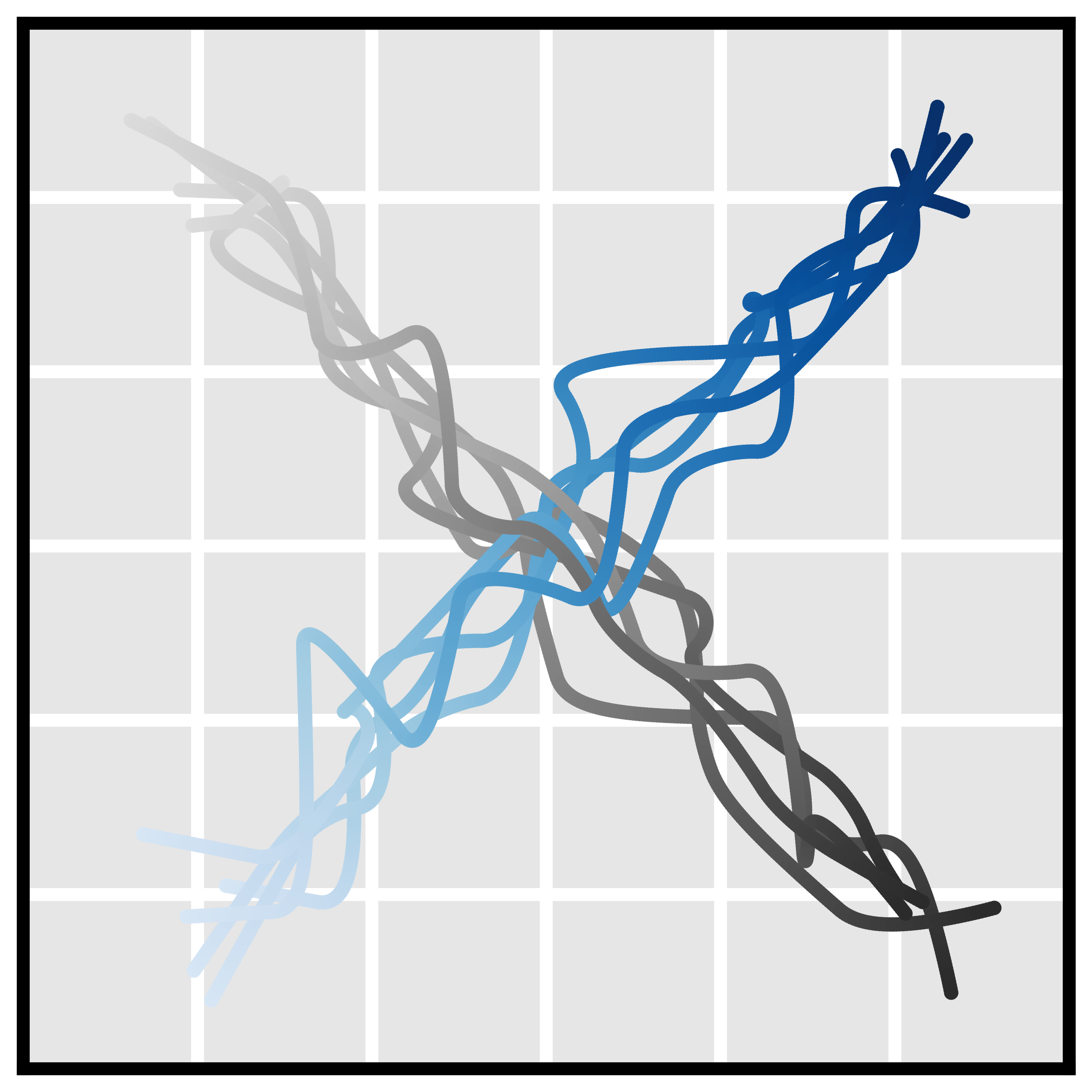}~
    \includegraphics[width=2.6cm]{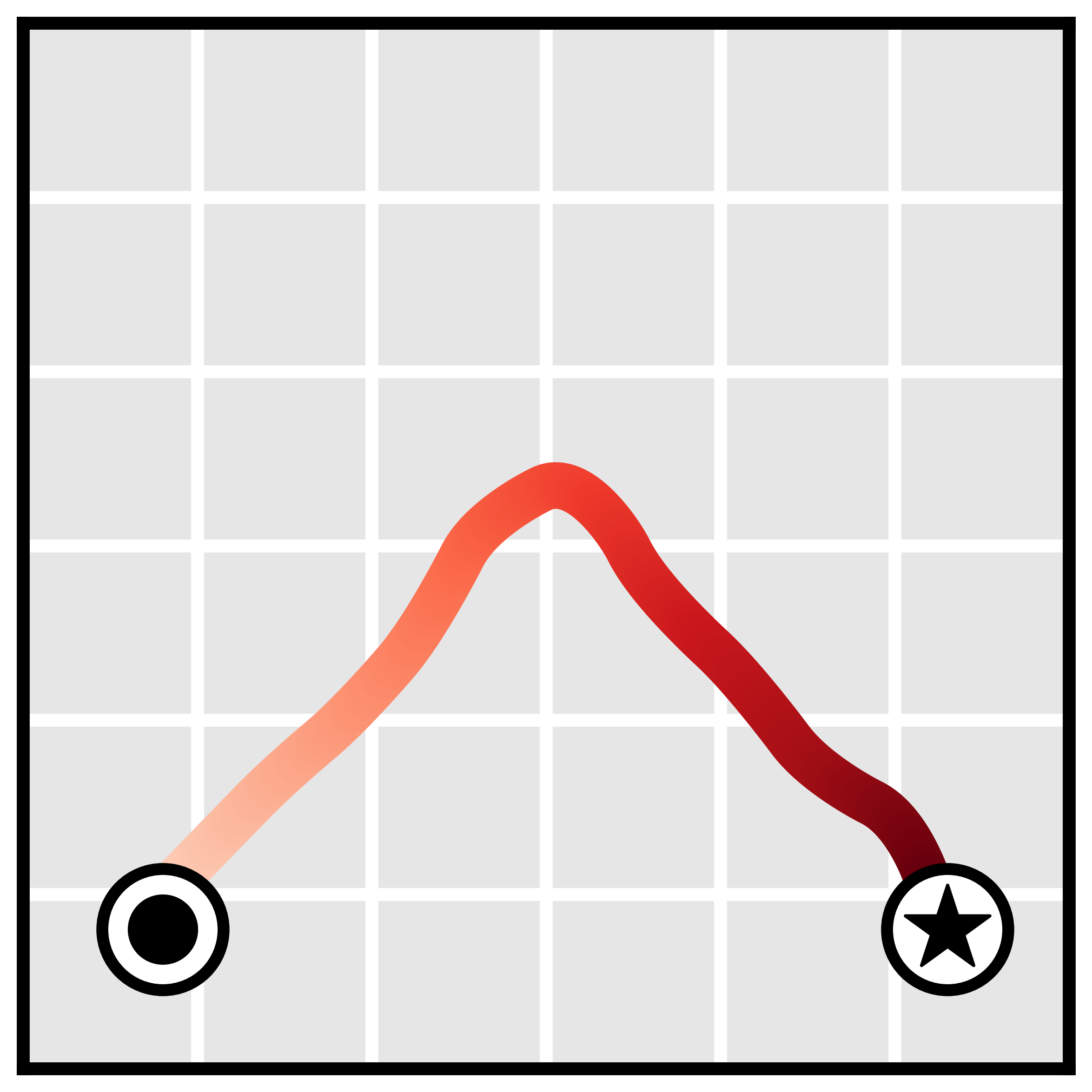} \\
    \vspace{-.5cm}
    \begin{center}
            \hspace{1.0cm} data \hspace{1.5cm} plan \\
    \end{center}
\end{minipage} \\
\vspace{.4cm}
\begin{minipage}[t]{0.65\textwidth}
    \raisebox{2cm}{\textbf{c}}~~~
    \raisebox{0.3cm}{\undersetimage{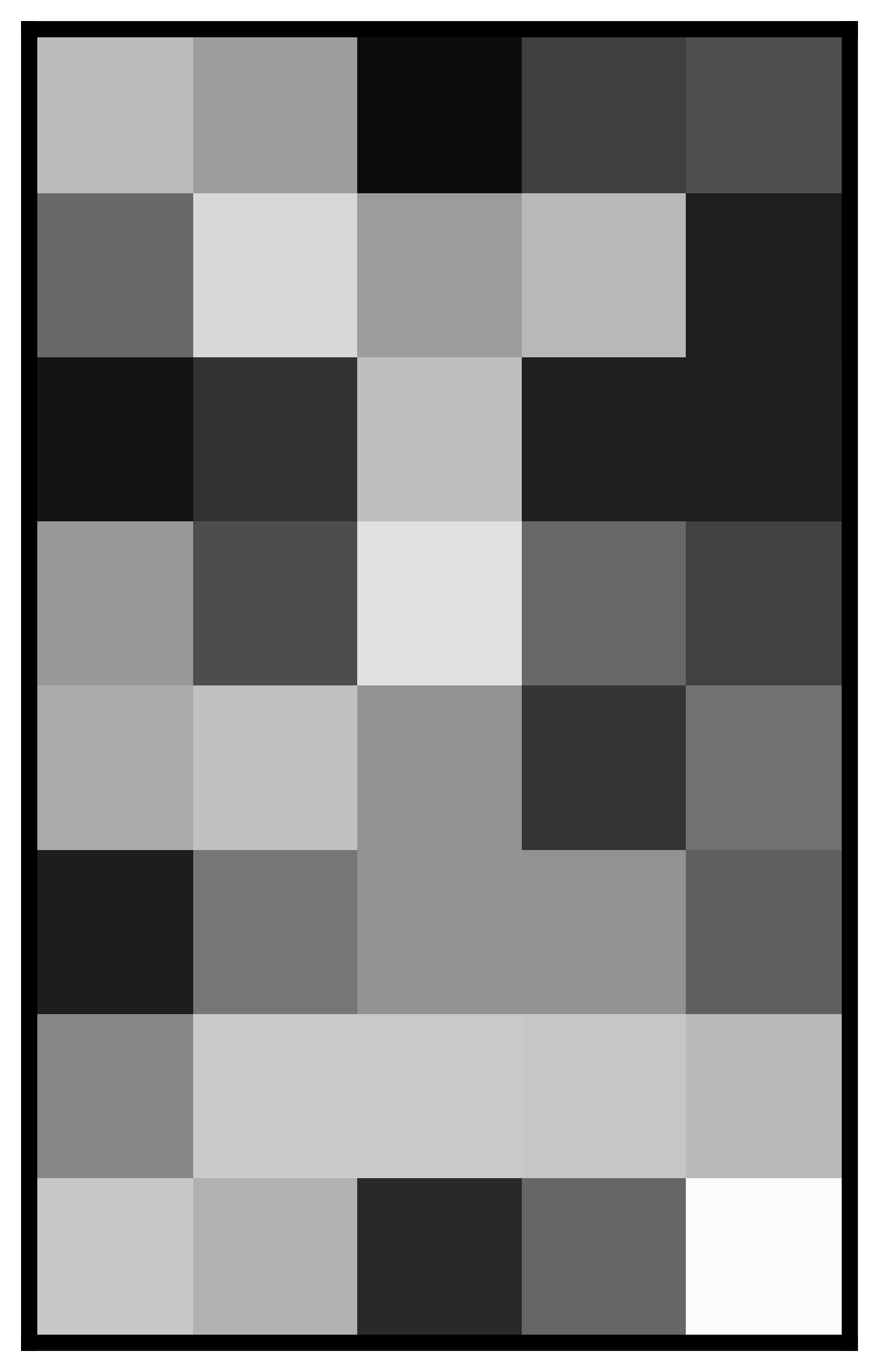}{$\btau{N}$}{-0.35}}
        \hspace{-.4cm}
        \raisebox{1.2cm}{$\rightarrow$}
        \includegraphics[width=2.6cm]{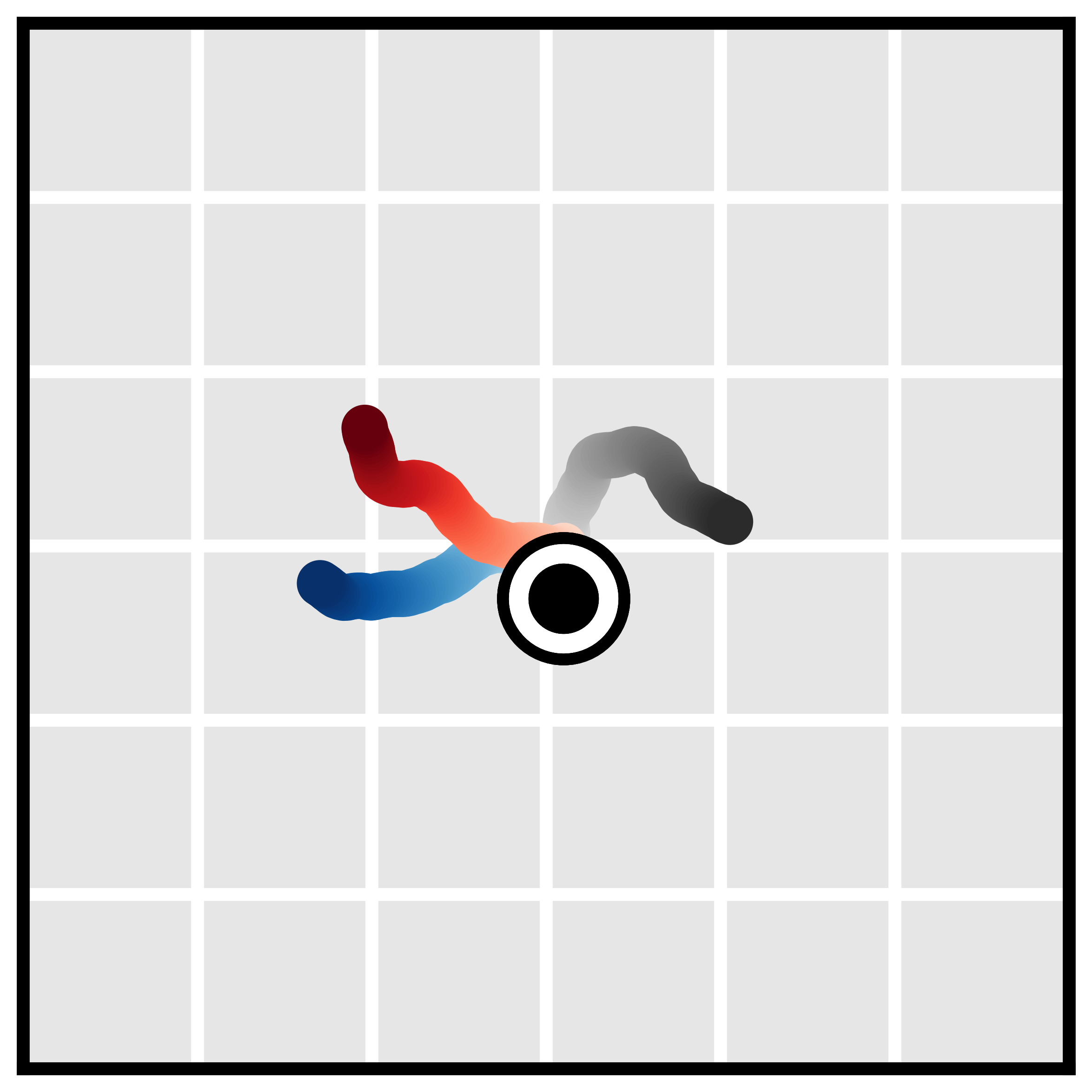}
        ~
    \raisebox{-0.1cm}{\undersetimage{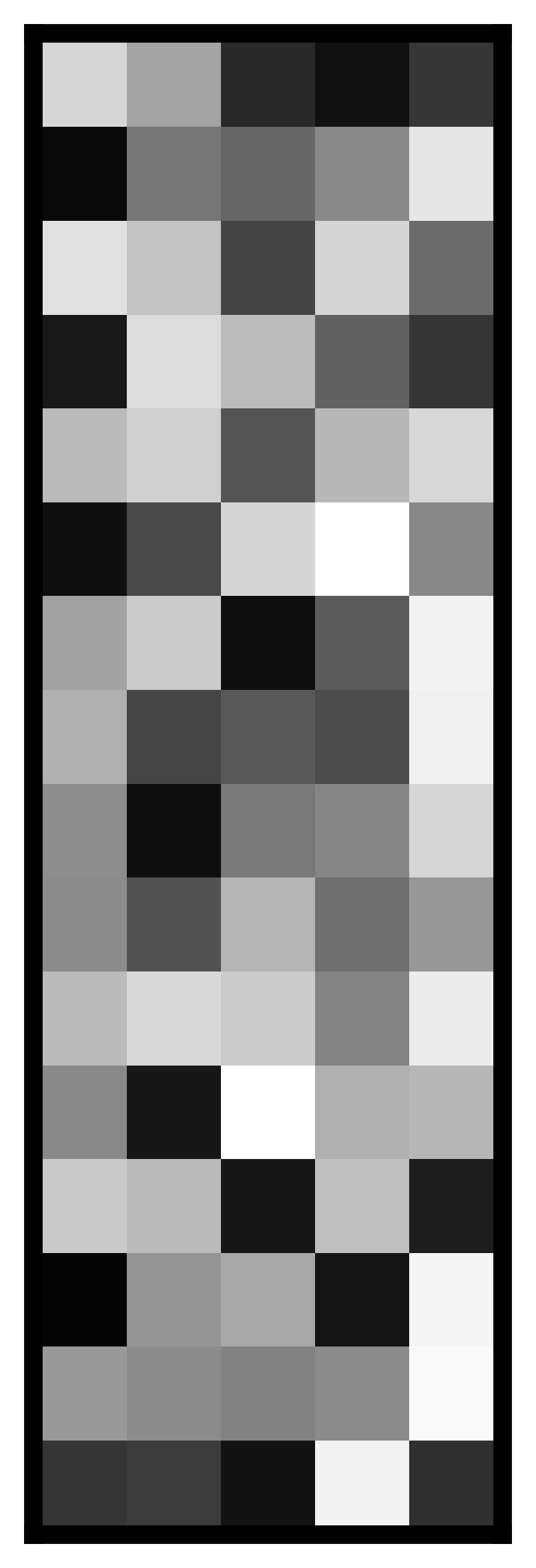}{$\btau{N}$}{-0.25}}
        \hspace{-.4cm}
        \raisebox{1.2cm}{$\rightarrow$}
        \includegraphics[width=2.6cm]{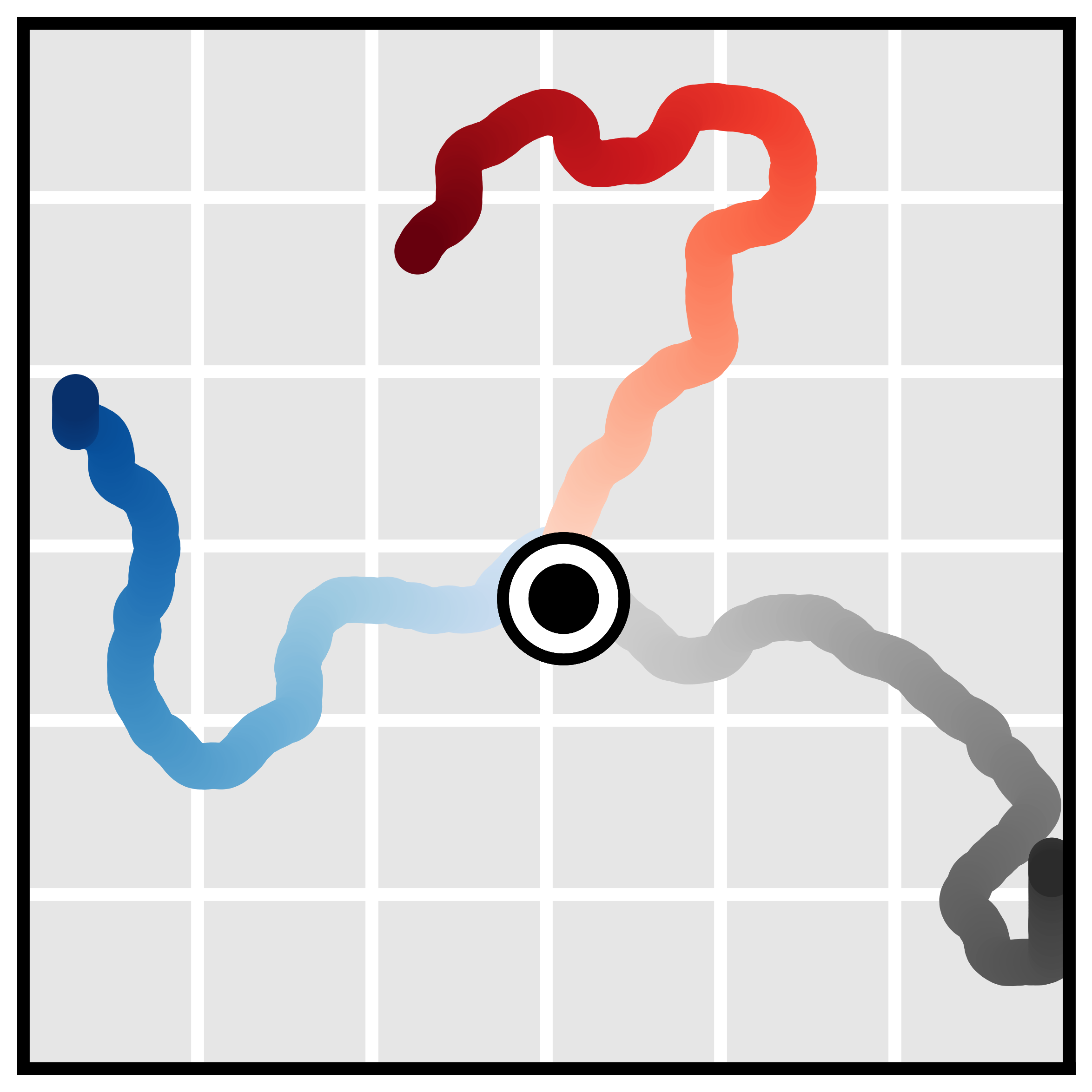} \\
\end{minipage}%
\hfill
\begin{minipage}[t]{0.35\textwidth}
    \raisebox{2cm}{\textbf{d}}~~~
    \includegraphics[width=2.6cm]{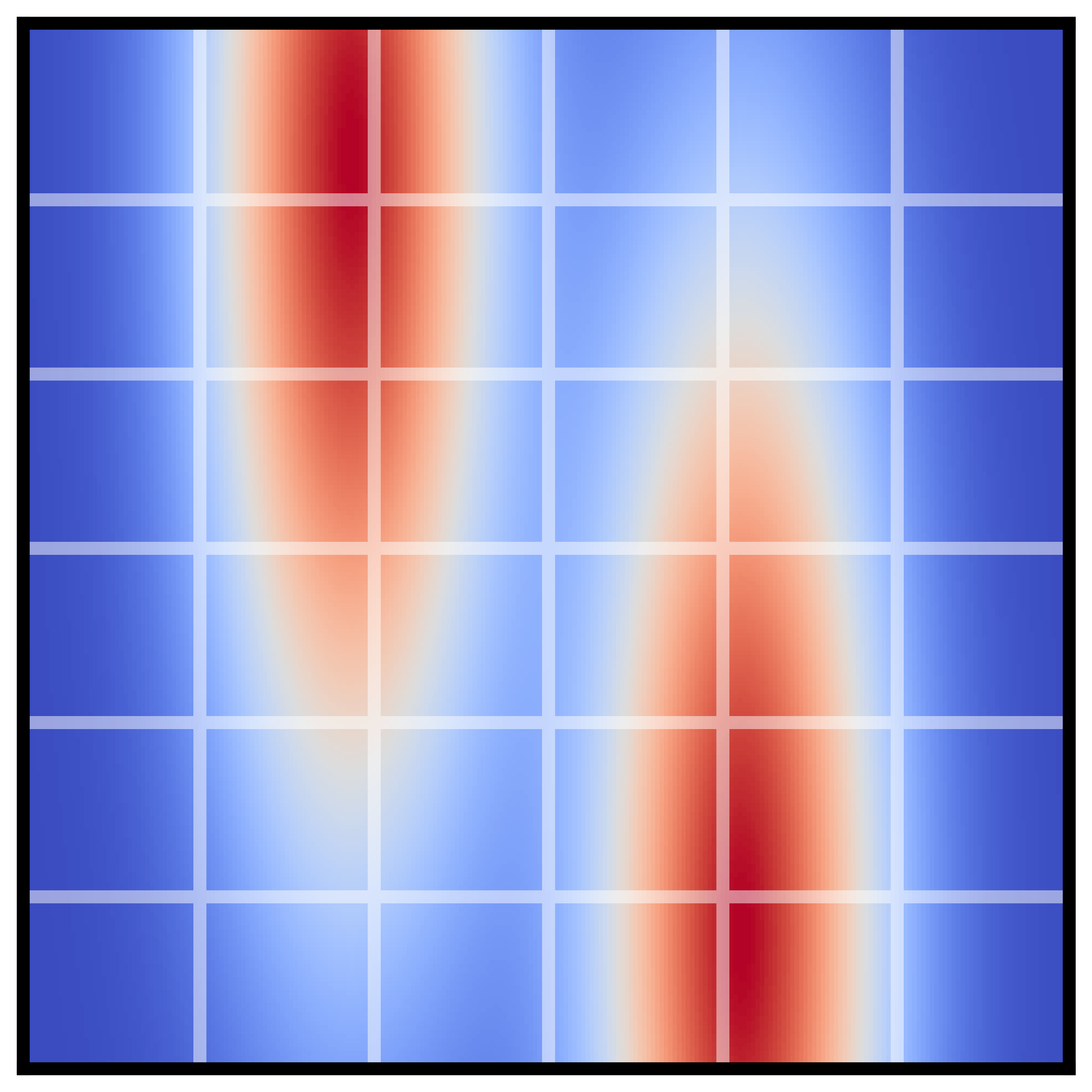}~
    \includegraphics[width=2.6cm]{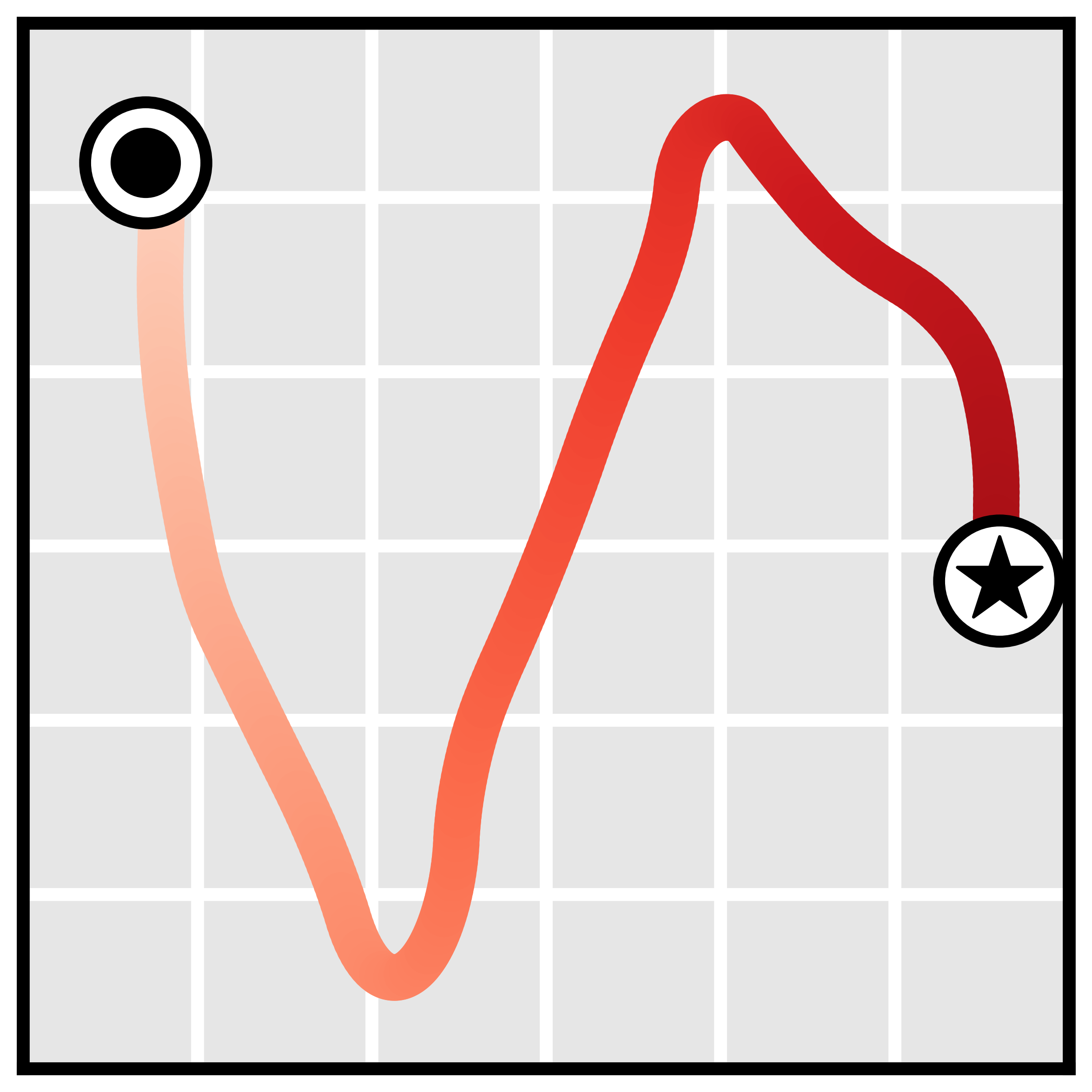} \\
    \vspace{-.5cm}
    \begin{flushleft}
            \hspace{0.43cm} reward
            \hspace{-0.1cm}
            \raisebox{.05cm}{$-$}
            \raisebox{.0cm}{\includegraphics[height=1cm,angle=90]{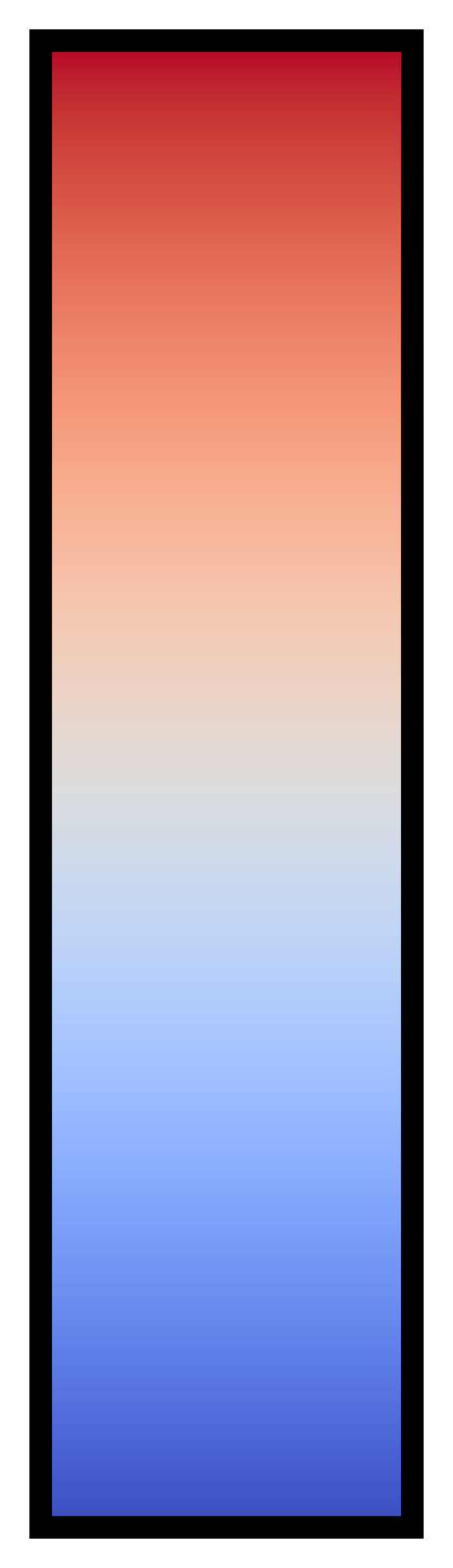}}
            \raisebox{.05cm}{$+$}
            \hspace{0.2cm} plan \\
    \end{flushleft}
\end{minipage}
\caption{
    \textbf{(Properties of diffusion planners)}
    \textbf{(a) Learned long-horizon planning:}
    \method{}'s learned planning procedure does not suffer from the myopic failure modes common to shooting algorithms and is able to plan over long horizons with sparse reward.
    \textbf{(b) Temporal compositionality:} Even though the model is not Markovian, it generates trajectories via iterated refinements to local consistency.
    As a result, it exhibits the types of generalization usually associated with Markovian models, with the ability to stitch together snippets of trajectories from the training data to generate novel plan.
    \textbf{(c) Variable-length plans:}
    Despite being a trajectory-level model, \method's planning horizon is not determined by its architecture.
    The horizon can be updated after training by changing the dimensionality of the input noise.
    \textbf{(d) Task compositionality:} \method{} can be composed with new reward functions to plan for tasks unseen during training.
    In all subfigures,
    \protect{\raisebox{-.05cm}{\includegraphics[height=.35cm]{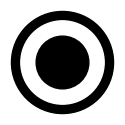}}}
    denotes a starting state and
    \protect{\raisebox{-.05cm}{\includegraphics[height=.35cm]{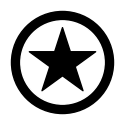}}}
    denotes a goal state.
}
\label{fig:properties}
\end{figure*}

\vspace{.2em}
\textbf{Variable-length plans.}~~
Because our model is fully convolutional in the horizon dimension of its prediction, its planning horizon is not specified by architectural choices.
Instead, it is determined by the size of the input noise $\btau{N} \sim \mathcal{N(\mathbf{0}, \mathbf{I})}$ that initializes the denoising process, allowing for variable-length plans (\textbf{Figure~\ref{fig:properties}c}).

\vspace{.2em}
\textbf{Task compositionality.}~~
While \method contains information about both environment dynamics and behaviors, it is independent of reward function.
Because the model acts as a prior over possible futures, planning can be guided by comparatively lightweight perturbation functions $h(\btau{})$ (or even combinations of multiple perturbations) corresponding to different rewards.
We demonstrate this by planning for a new reward function unseen during training of the diffusion model (\textbf{Figure~\ref{fig:properties}d}).


\begin{table}[t]
\centering
\small
\begin{tabular*}{\columnwidth}{@{\extracolsep{\fill}}lrrrrr}
\toprule
\multicolumn{1}{c}{\textbf{Environment}} &
    \multicolumn{1}{c}{\textbf{MPPI}} &
    \multicolumn{1}{c}{\textbf{CQL}} &
    \multicolumn{1}{c}{\textbf{IQL}} &
    \multicolumn{1}{c}{\textbf{\method}} \\
\midrule
Maze2D~~~~U-Maze &
    $33.2$ &
    $5.7$ &
    $47.4$ &
    \highlight{\color{highlight}113.9} \scriptsize{\raisebox{1pt}{$\pm 3.1$}} \\ 
Maze2D~~~~Medium &
    $10.2$ &
    $5.0$ &
    $34.9$ &
    \highlight{\color{highlight}121.5} \scriptsize{\raisebox{1pt}{$\pm 2.7$}} \\ 
Maze2D~~~~Large &
    $5.1$ &
    $12.5$ &
    $58.6$ &
    \highlight{\color{highlight}123.0} \scriptsize{\raisebox{1pt}{$\pm 6.4$}} \\
\midrule
\multicolumn{1}{c}{\textbf{Single-task Average}} & 16.2 & 7.7 & 47.0 & \highlight{\color{highlight}119.5} \hspace{.58cm} \\
\vspace{-.2cm} \\
\toprule
Multi2D~~~~U-Maze & 
    $41.2$ &
    - &
    $24.8$ &
    \highlight{\color{highlight}128.9} \scriptsize{\raisebox{1pt}{$\pm 1.8$}} \\ 
Multi2D~~~~Medium &
    $15.4$ &
    - &
    $12.1$ &
    \highlight{\color{highlight}127.2} \scriptsize{\raisebox{1pt}{$\pm 3.4$}} \\ 
Multi2D~~~~Large &
    $8.0$ &
    - &
    $13.9$ &
    \highlight{\color{highlight}132.1} \scriptsize{\raisebox{1pt}{$\pm 5.8$}} \\
\midrule
\multicolumn{1}{c}{\textbf{Multi-task Average}} & 21.5 & - & 16.9 & \highlight{\color{highlight}129.4} \hspace{.58cm} \\ 
\bottomrule
\end{tabular*}
\caption{
    \textbf{(Long-horizon planning)}
    The performance of \method{} and prior model-free algorithms in the Maze2D environment, which tests long-horizon planning due to its sparse reward structure.
    The Multi2D setting refers to a multi-task variant with goal locations resampled at the beginning of every episode.
    \method{} substantially outperforms prior approaches in both settings.
    Appendix~\ref{app:baselines} details the sources for the scores of the baseline algorithms.
}
\label{table:maze2d}
\end{table}

\section{Experimental Evaluation}
\label{sec:experiments}

The focus of our experiments is to evaluate \method on the capabilities we would like from a data-driven planner.
In particular, we evaluate
\textbf{(1)} the ability to plan over long horizons without manual reward shaping,
\textbf{(2)} the ability to generalize to new configurations of goals unseen during training, and
\textbf{(3)} the ability to recover an effective controller from heterogeneous data of varying quality.
We conclude by studying practical runtime considerations of diffusion-based planning, including the most effective ways of speeding up the planning procedure while suffering minimally in terms of performance.

\subsection{Long Horizon Multi-Task Planning}
\label{sec:maze2d}

We evaluate long-horizon planning in the Maze2D environments \citep{fu2020d4rl}, which require traversing to a goal location where a reward of 1 is given.
No reward shaping is provided at any other location.
Because it can take hundreds of steps to reach the goal location, 
even the best model-free algorithms struggle to adequately perform credit assignment and reliably reach the goal (Table~\ref{table:maze2d}).

We plan with \method using the inpainting strategy to condition on a start and goal location.
(The goal location is also available to the model-free methods; it is identifiable by being the only state in the dataset with non-zero reward.)
We then use the sampled trajectory as an open-loop plan.
\method achieves scores over 100 in all maze sizes, indicating that it outperforms a reference expert policy.
We visualize the reverse diffusion process generating \method's plans in Figure~\ref{fig:maze2d}.

While the training data in Maze2D is undirected -- consisting of a controller navigating to and from randomly selected locations -- the evaluation is single-task in that the goal is always the same.
In order to test multi-task flexibility, we modify the environment to randomize the goal location at the beginning of each episode.
This setting is denoted as Multi2D in Table~\ref{table:maze2d}.
\method{} is naturally a multi-task planner; we do not need to retrain the model from the single-task experiments and simply change the conditioning goal.
As a result, \method{} performs as well in the multi-task setting as in the single-task setting.
In contrast, there is a substantial performance drop of the best model-free algorithm in the single-task setting (IQL; \citealt{kostrikov2021implicit}) when adapted to the multi-task setting.
Details of our multi-task IQL with hindsight experience relabeling \citep{andrychowicz2017hindsight} are provided in Appendix~\ref{app:multitask}.
MPPI uses the ground-truth dynamics; its poor performance compared to the learned planning algorithm of Diffuser highlights the difficulty posed by long-horizon planning even when there are no prediction inaccuracies.

\begin{figure}[t!]
\centering
\begin{minipage}{.8\columnwidth}
\begin{flushright}
    \raisebox{.35\height}{\rotatebox{90}{\textbf{U-Maze}}}\hspace{0.5cm}
    \includegraphics[width=0.85\columnwidth]{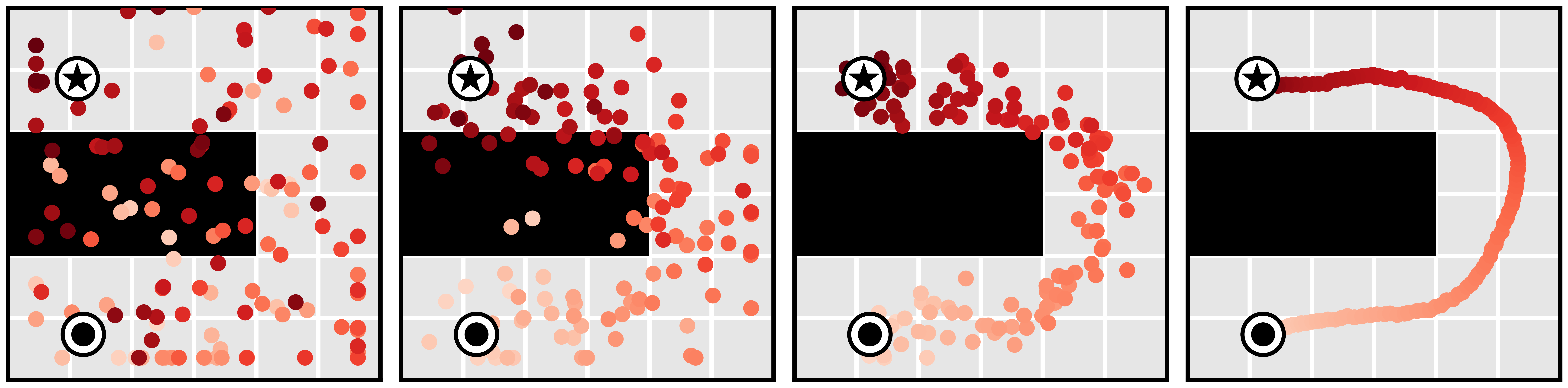} \\
    \raisebox{.3\height}{\rotatebox{90}{\textbf{Medium}}}\hspace{0.5cm}
    \includegraphics[width=0.85\columnwidth]{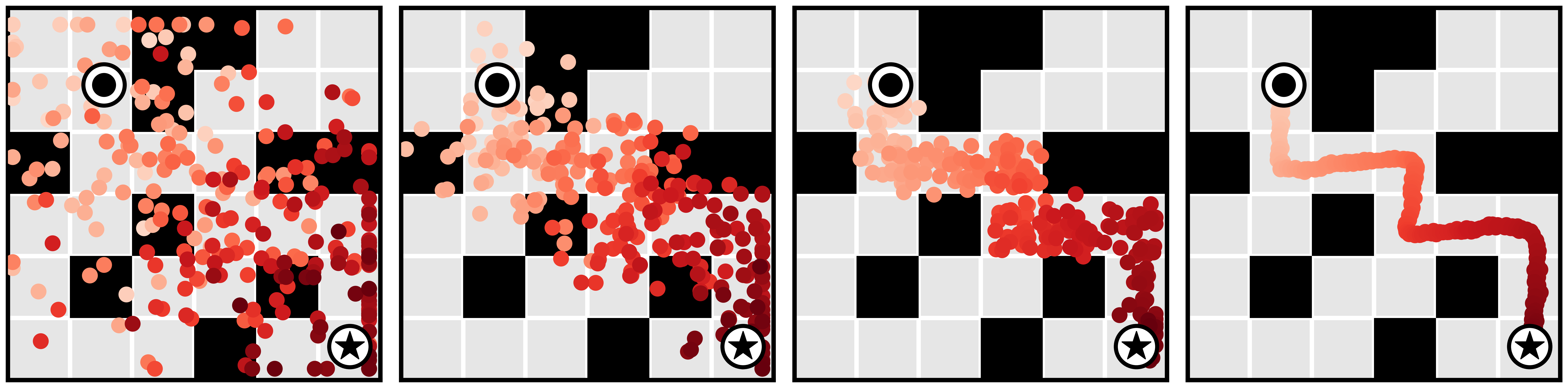} \\
    \raisebox{.6\height}{\rotatebox{90}{\textbf{Large}}}\hspace{0.5cm}
    \includegraphics[width=0.85\columnwidth,trim={.03cm 0 .03cm 0},clip]{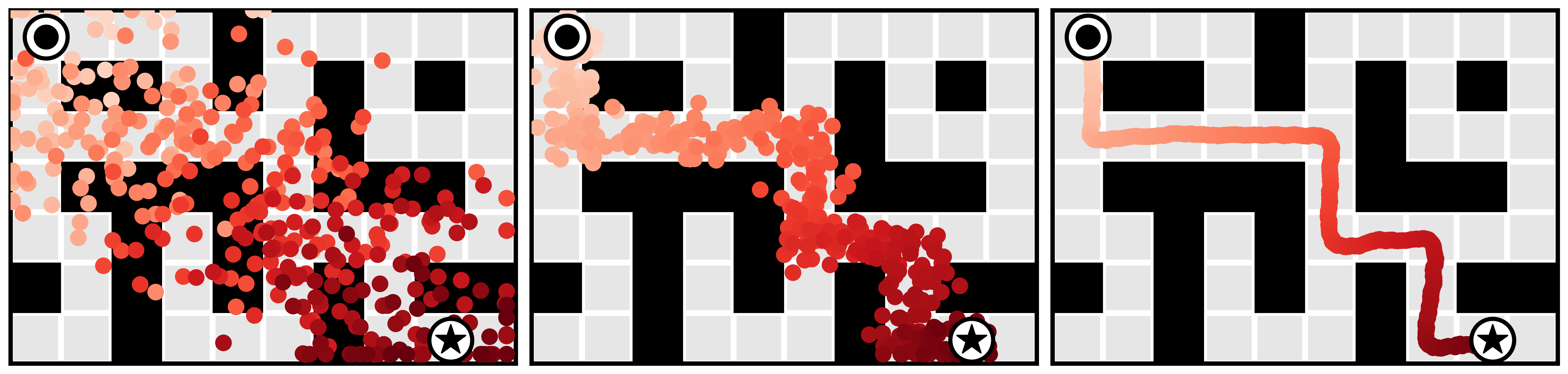} \\
\end{flushright}
\begin{centering}
    \raisebox{.3\height}{denoising}~~
    \tikz \draw [-{Computer Modern Rightarrow[length=2mm, width=3mm]}, line width=.3mm] (0,0) -- (3,0); \\
\end{centering}
\end{minipage}
\caption{
    \textbf{(Planning as inpainting)}
    Plans are generated in the Maze2D environment by sampling trajectories consistent with a specified start 
    \protect{\raisebox{-.05cm}{\includegraphics[height=.35cm]{diffuser/images/maze2d/mark_start_crop.png}}}
    and goal
    \protect{\raisebox{-.05cm}{\includegraphics[height=.35cm]{diffuser/images/maze2d/mark_goal_crop.png}}}
    condition.
    The remaining states are ``inpainted'' by the denoising process.
}
\label{fig:maze2d}
\end{figure}

\subsection{Test-time Flexibility}
\label{sec:blocks}

In order to evaluate the ability to generalize to new test-time goals, we construct a suite of block stacking tasks with three settings: \textbf{(1)} unconditional stacking, for which the task is to build a block tower as tall as possible; \textbf{(2)} conditional stacking, for which the task is to construct a block tower with a specified order of blocks, and \textbf{(3)} rearrangement, for which the task is to match a set of reference blocks' locations in a novel arrangement.
We train all methods on 10000 trajectories from demonstrations generated by PDDLStream \citep{garrett2020pddlstream}; rewards are equal to one upon successful stack placements and zero otherwise.
These block stacking are challenging diagnostics of test-time flexibility;
in the course of executing a partial stack for a randomized goal, a controller will venture into novel states not included in the training configuration.

We use one trained \method for all block-stacking tasks, only modifying the perturbation function $h(\btau{})$ between settings.
In the unconditional stacking task, we directly sample from the unperturbed denoising process $p_\theta(\btau{})$ to emulate the PDDLStream controller.
In the conditional stacking and rearrangement tasks, we compose two perturbation functions $h(\btau{})$ to bias the sampled trajectories: the first maximizes the likelihood of the trajectory's final state matching the goal configuration, and the second enforces a contact constraint between the end effector and a cube during stacking motions.
(See Appendix~\ref{app:stacking_costs} for details.)

\begin{figure}[t]
    \centering
    \small
    \begin{tabular}{@{\extracolsep{\fill}}lrrr}
    \toprule
    \multicolumn{1}{c}{\textbf{Environment}} & \multicolumn{1}{c}{\textbf{~~~~~BCQ~}} &
        \multicolumn{1}{c}{\textbf{~~~~~CQL~}} &
        \multicolumn{1}{c}{\textbf{~~\method}} \\
    \midrule
    {Unconditional Stacking} & $0.0$ & $24.4$ & ~\highlight{\color{highlight}58.7} \scriptsize{\raisebox{1pt}{$\pm 2.5$}}
         \\ 
    {Conditional Stacking} & 0.0  & 0.0 & \highlight{\color{highlight}45.6} \scriptsize{\raisebox{1pt}{$\pm 3.1$}}
         \\ 
    {Rearrangement} & 0.0  & 0.0 & \highlight{\color{highlight}58.9} \scriptsize{\raisebox{1pt}{$\pm 3.4$}} \\ 
    \midrule
    \multicolumn{1}{c}{\textbf{Average}} & 0.0  & 8.1 & \highlight{\color{highlight}54.4~~} \hspace{.44cm} \\ 
    \bottomrule
    \end{tabular}
    \captionof{table}{
    \textbf{(Test-time flexibility)}
    Performance of BCQ, CQL, and \method on block stacking tasks.
    A score of 100 corresponds to a perfectly executed stack; 0 is that of a random policy.
    }
    \label{table:blocks}
\end{figure}

\begin{figure}[t]
    \centering
    \includegraphics[width=0.3\columnwidth]{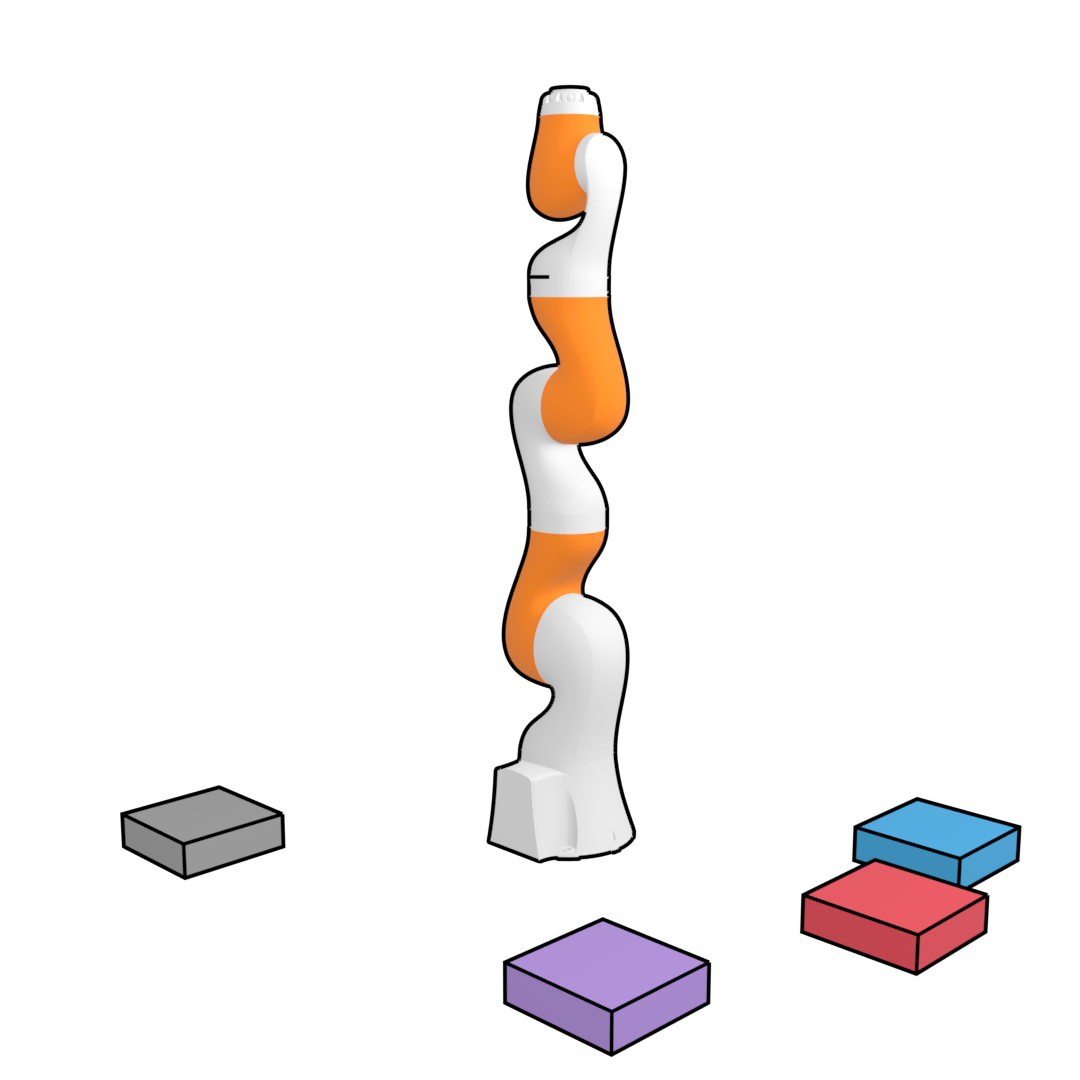}
    \raisebox{1.7cm}{\hspace{0cm}$\rightarrow$}\hspace{.5cm}
    \includegraphics[width=0.3\columnwidth]{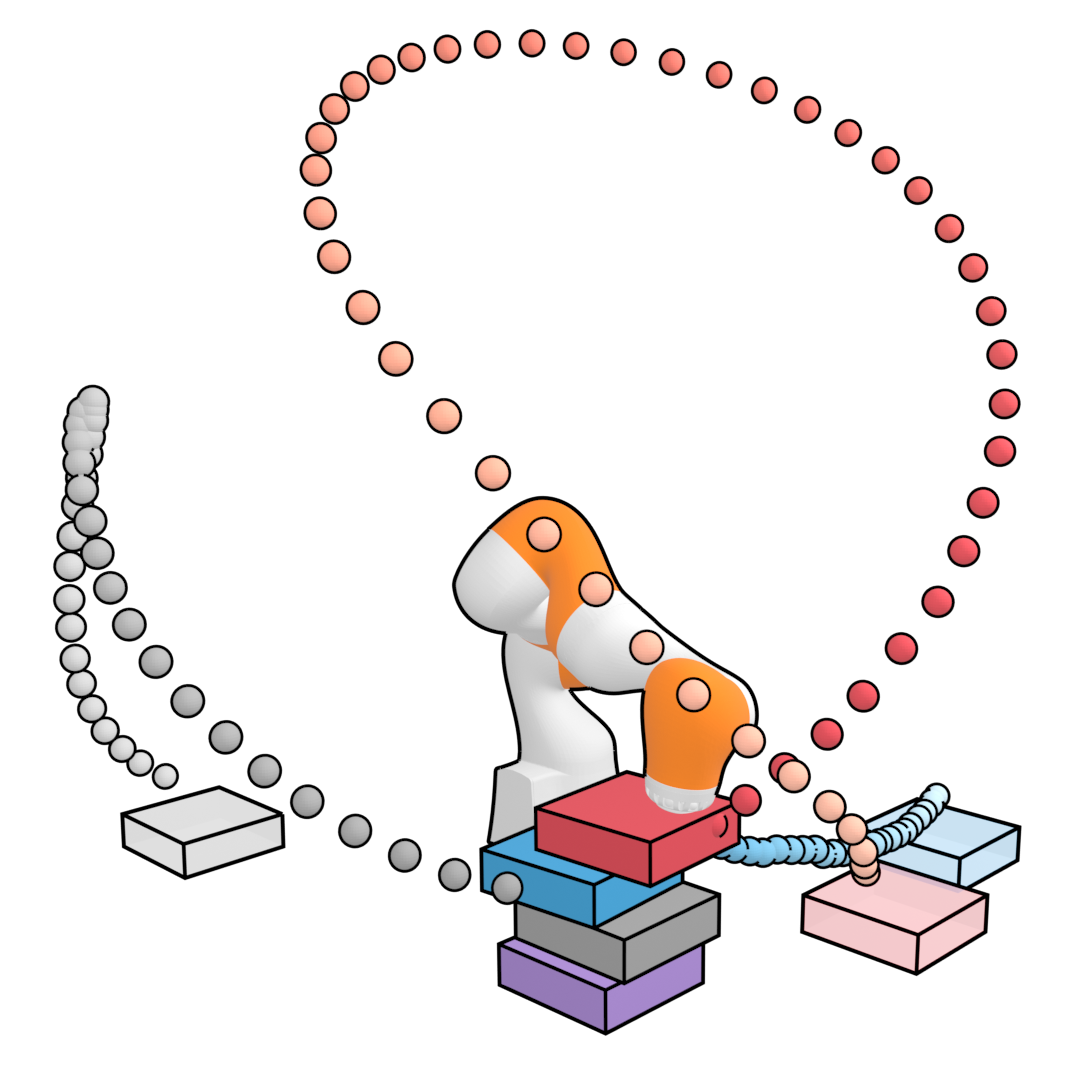}
    \caption{
        \textbf{(Block stacking)}
        A block stacking sequence executed by \method. This task is best illustrated by videos viewable at
        \href{https://diffusion-planning.github.io/}
        {\texttt{diffusion-planning.github.io}}.
    }
    \label{fig:blocks}
\end{figure}

We compare with two prior model-free offline reinforcement learning algorithms: BCQ \citep{fujimoto2019off} and CQL \citep{kumar2020conservative}, training standard variants for the unconditional stacking task and goal-conditioned variants for the conditional stacking and rearrangement tasks.
(Baseline details are provided in Appendix~\ref{app:baselines}.)
Quantitative results are given in Table~\ref{table:blocks}, in which a score of 100 corresponds to a perfect execution of the task.
\method substantially outperforms both prior methods, with the conditional settings requiring flexible behavior generation proving especially difficult for the model-free algorithms.
A visual depiction of an execution by \method is provided in Figure~\ref{fig:blocks}.

\subsection{Offline Reinforcement Learning}
\label{sec:locomotion}

\begin{figure}[t]
\centering
\begin{minipage}{.7\columnwidth}
\begin{minipage}{.05\columnwidth}
\begin{flushleft}
    \raisebox{.2\height}{\rotatebox{90}{denoising}}
    \tikz \draw [-{Computer Modern Rightarrow[length=2mm, width=3mm]}, line width=.3mm] (0.2,2) -- (0.2,0); \\
\end{flushleft}
\end{minipage}~~
\begin{minipage}{.9\columnwidth}
\centering
    \vspace{.15cm}
    \includegraphics[width=\linewidth]{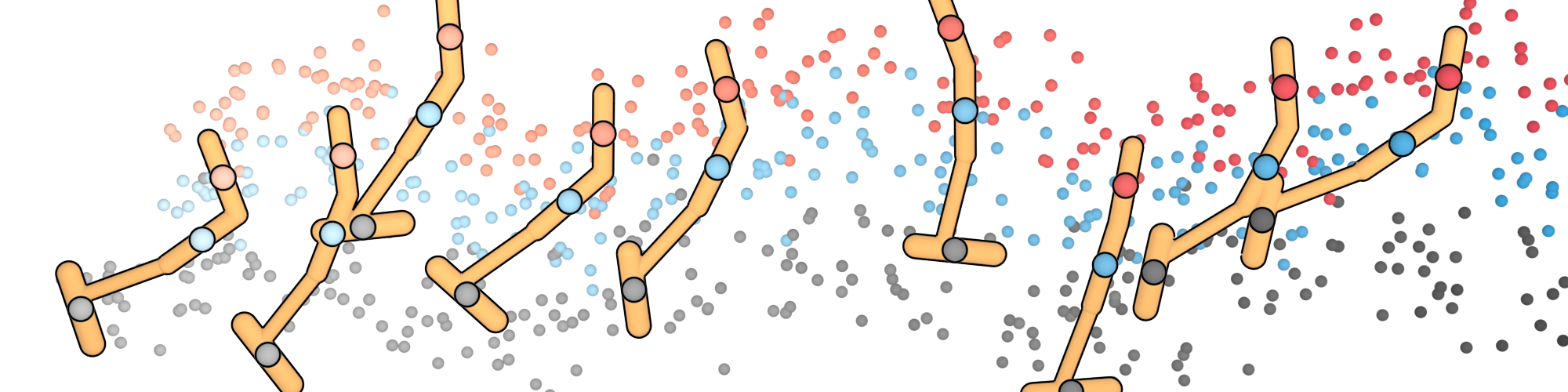} \\
    \includegraphics[width=\linewidth]{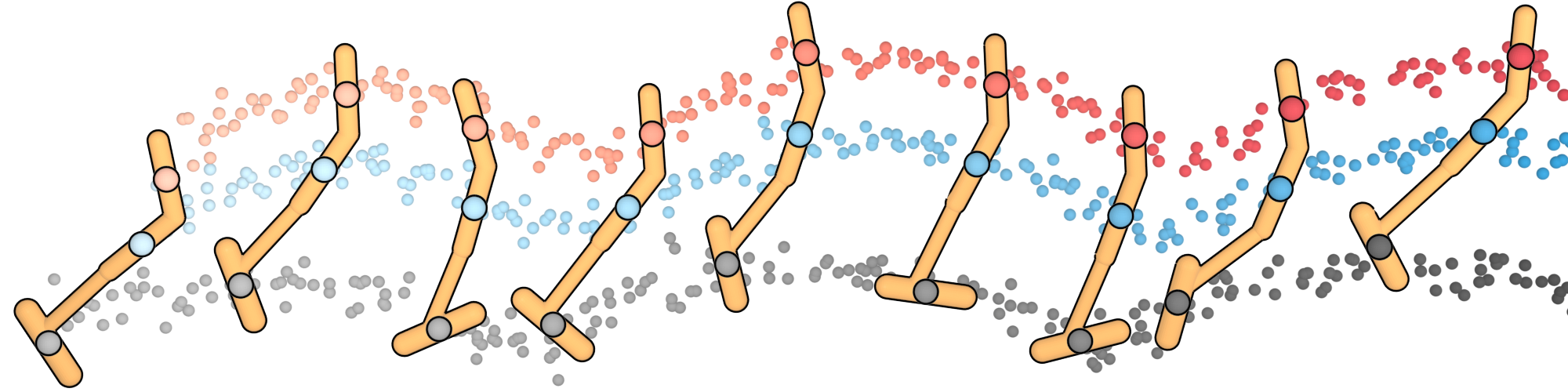} \\
    \includegraphics[width=\linewidth]{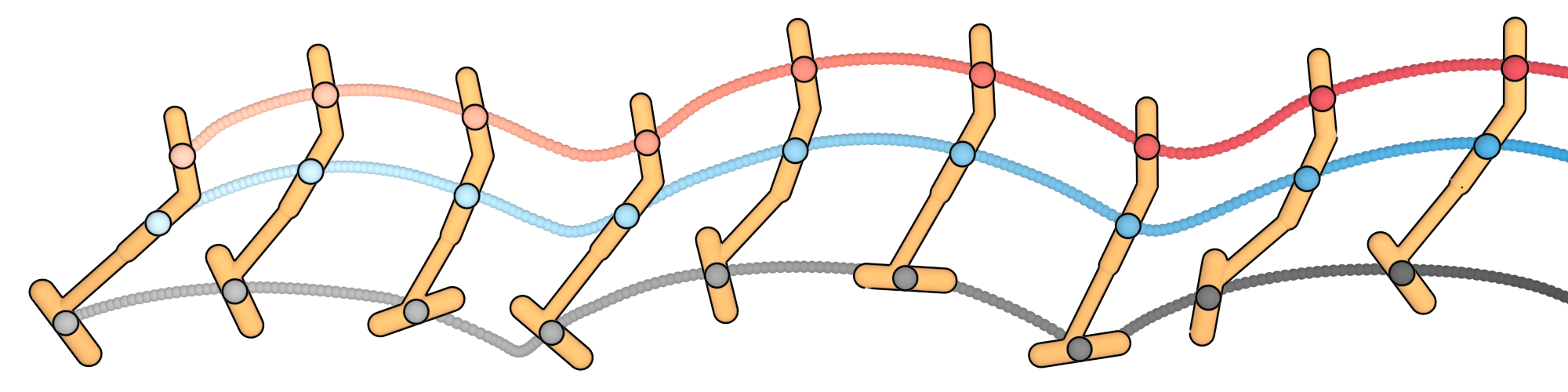} \\
\end{minipage}
\end{minipage}
\colorbox{white}{\textcolor{white}{\rule{\linewidth}{.5cm}}}
\hspace{5cm}
\raisebox{.3\height}{planning horizon}~~
\tikz \draw [-{Computer Modern Rightarrow[length=2mm, width=3mm]}, line width=.3mm] (0,0) -- (3,0); \\
\caption{
    \textbf{(Guided sampling)}
    \method generates all timesteps of a plan concurrently, instead of autoregressively, through the denoising process.
}
\label{fig:locomotion}
\end{figure}

Finally, we evaluate the capacity to recover an effective single-task controller from heterogeneous data of varying quality using the D4RL offline locomotion suite \citep{fu2020d4rl}.
We guide the trajectories generated by \method toward high-reward regions using the sampling procedure described in Algorithm~\ref{alg:rl} and condition the trajectories on the current state using the inpainting procedure described in Section~\ref{sec:inpainting}.
The reward predictor $\mathcal{J}_\phi$ is trained on the same trajectories as the diffusion model.
A visualization of the learned denoising procedure in this setting is shown in Figure~\ref{fig:locomotion}.

We compare to a variety of prior algorithms spanning other approaches to data-driven control, including the model-free reinforcement learning algorithms CQL \citep{kumar2020conservative} and IQL \citep{kostrikov2021implicit}; return-conditioning approaches like Decision Transformer (DT; \citealt{chen2021decision}); and model-based reinforcement learning approaches including Trajectory Transformer (TT; \citealt{janner2021sequence}), MOPO \citep{yu2020mopo}, MOReL \citep{kidambi2020morel}, and MBOP \citep{argenson2020model}.
As shown in Table~\ref{table:locomotion}, \method performs comparably to prior algorithms in the single-task setting: better than the model-based MOReL and MBOP and return-conditioning DT, but worse than the best offline techniques designed specifically for single-task performance.
We also investigated a variant using Diffuser as a dynamics model in conventional trajectory optimizers such as MPPI \citep{grady2015mppi}, but found that this combination performed no better than random, suggesting that the effectiveness of Diffuser stems from coupled modeling and planning, and not from improved open-loop predictive accuracy.

\definecolor{tblue}{HTML}{1F77B4}
\definecolor{tred}{HTML}{FF6961}
\definecolor{tgreen}{HTML}{429E9D}
\definecolor{thighlight}{HTML}{000000}
\newcolumntype{P}{>{\raggedleft\arraybackslash}X}
\begin{table*}[t]
\centering
\footnotesize
\begin{tabularx}{\textwidth}{llPPPPPPPr}
\toprule
\multicolumn{1}{r}{\textbf{\color{black}} Dataset} & \multicolumn{1}{r}{\textbf{\color{black}} Environment} & \multicolumn{1}{r}{\textbf{\color{black}} BC} & \multicolumn{1}{r}{\textbf{\color{black}} CQL} & \multicolumn{1}{r}{\textbf{\color{black}} IQL} & \multicolumn{1}{r}{\textbf{\color{black}} DT} & \multicolumn{1}{r}{\textbf{\color{black}} TT} & \multicolumn{1}{r}{\textbf{\color{black}} MOPO} & \multicolumn{1}{r}{\textbf{\color{black}} MBOP} & \multicolumn{1}{r}{\textbf{\color{black}} Diffuser} \\ 
\midrule
Medium-Expert & HalfCheetah & $55.2$ & $91.6$ & $86.7$ & $86.8$ & $95.0$ & $63.3$ & $\textbf{\color{thighlight}105.9}$ & $88.9$ \scriptsize{\raisebox{1pt}{$\pm 0.3$}} \\ 
Medium-Expert & Hopper & $52.5$ & $\textbf{\color{thighlight}105.4}$ & $91.5$ & $\textbf{\color{thighlight}107.6}$ & $\textbf{\color{thighlight}110.0}$ & $23.7$ & $55.1$ & $103.3$ \scriptsize{\raisebox{1pt}{$\pm 1.3$}} \\ 
Medium-Expert & Walker2d & $\textbf{\color{thighlight}107.5}$ & $\textbf{\color{thighlight}108.8}$ & $\textbf{\color{thighlight}109.6}$ & $\textbf{\color{thighlight}108.1}$ & $101.9$ & $44.6$ & $70.2$ & $\textbf{\color{thighlight}106.9}$ \scriptsize{\raisebox{1pt}{$\pm 0.2$}} \\ 
\midrule
Medium & HalfCheetah & $42.6$ & $44.0$ & $\textbf{\color{thighlight}47.4}$ & $42.6$ & $\textbf{\color{thighlight}46.9}$ & $42.3$ & $44.6$ & $42.8$ \scriptsize{\raisebox{1pt}{$\pm 0.3$}} \\ 
Medium & Hopper & $52.9$ & $58.5$ & $66.3$ & $67.6$ & $61.1$ & $28.0$ & $48.8$ & $\textbf{\color{thighlight}74.3}$ \scriptsize{\raisebox{1pt}{$\pm 1.4$}} \\ 
Medium & Walker2d & $75.3$ & $72.5$ & $\textbf{\color{thighlight}78.3}$ & $74.0$ & $\textbf{\color{thighlight}79.0}$ & $17.8$ & $41.0$ & $\textbf{\color{thighlight}79.6}$ \scriptsize{\raisebox{1pt}{$\pm 0.55$}} \\ 
\midrule
Medium-Replay & HalfCheetah & $36.6$ & $45.5$ & $44.2$ & $36.6$ & $41.9$ & $\textbf{\color{thighlight}53.1}$ & $42.3$ & $37.7$ \scriptsize{\raisebox{1pt}{$\pm 0.5$}} \\ 
Medium-Replay & Hopper & $18.1$ & $\textbf{\color{thighlight}95.0}$ & $\textbf{\color{thighlight}94.7}$ & $82.7$ & $\textbf{\color{thighlight}91.5}$ & $67.5$ & $12.4$ & $\textbf{\color{thighlight}93.6}$ \scriptsize{\raisebox{1pt}{$\pm 0.4$}} \\ 
Medium-Replay & Walker2d & $26.0$ & $77.2$ & $73.9$ & $66.6$ & $\textbf{\color{thighlight}82.6}$ & $39.0$ & $9.7$ & $70.6$ \scriptsize{\raisebox{1pt}{$\pm 1.6$}} \\ 
\midrule
\multicolumn{2}{c}{\textbf{Average}} & 51.9 & \textbf{\color{thighlight}77.6} & \textbf{\color{thighlight}77.0} & 74.7 & \textbf{\color{thighlight}78.9} & 42.1 & 47.8 & \textbf{\color{thighlight}77.5} \hspace{.6cm} \\ 
\bottomrule
\end{tabularx}
\vspace{-.0cm}
\caption{
    \textbf{(Offline reinforcement learning)}
    The performance of \method{} and a variety of prior algorithms on the D4RL locomotion benchmark \citep{fu2020d4rl}.
    Results for \method{} correspond to the mean and standard error over 150 planning seeds. We detail the sources for the performance of prior methods in Appendix~\ref{app:d4rl_sources}. Following \citet{kostrikov2021implicit}, we emphasize in bold scores within 5 percent of the maximum per task ($\ge 0.95 \cdot \text{max}$).
}
\label{table:locomotion}
\end{table*}

\subsection{Warm-Starting Diffusion for Faster Planning}

A limitation of \method{} is that individual plans are slow to generate (due to iterative generation). Na\"ively,  as we execute plans open loop, a new plan must be regenerated at each step of execution. To improve execution speed of \method{}, we may further reuse previously generated plans to warm-start generations of subsequent plans. 

To warm-start planning, we may run a limited number of forward diffusion steps from a previously generated plan and then run a corresponding number of denoising steps from this partially noised trajectory to regenerate an updated plan. In \fig{fig:speed}, we illustrate the trade-off between performance and runtime budget as we vary the underlying number of denoising steps used to regenerate each a new plan from 2 to 100. We find that we may reduce the planning budget of our approach markedly with only modest drop in performance.

\vspace{-.15cm}
\section{Related Work}

\begin{figure}
    \centering
    \includegraphics[width=0.6\columnwidth]{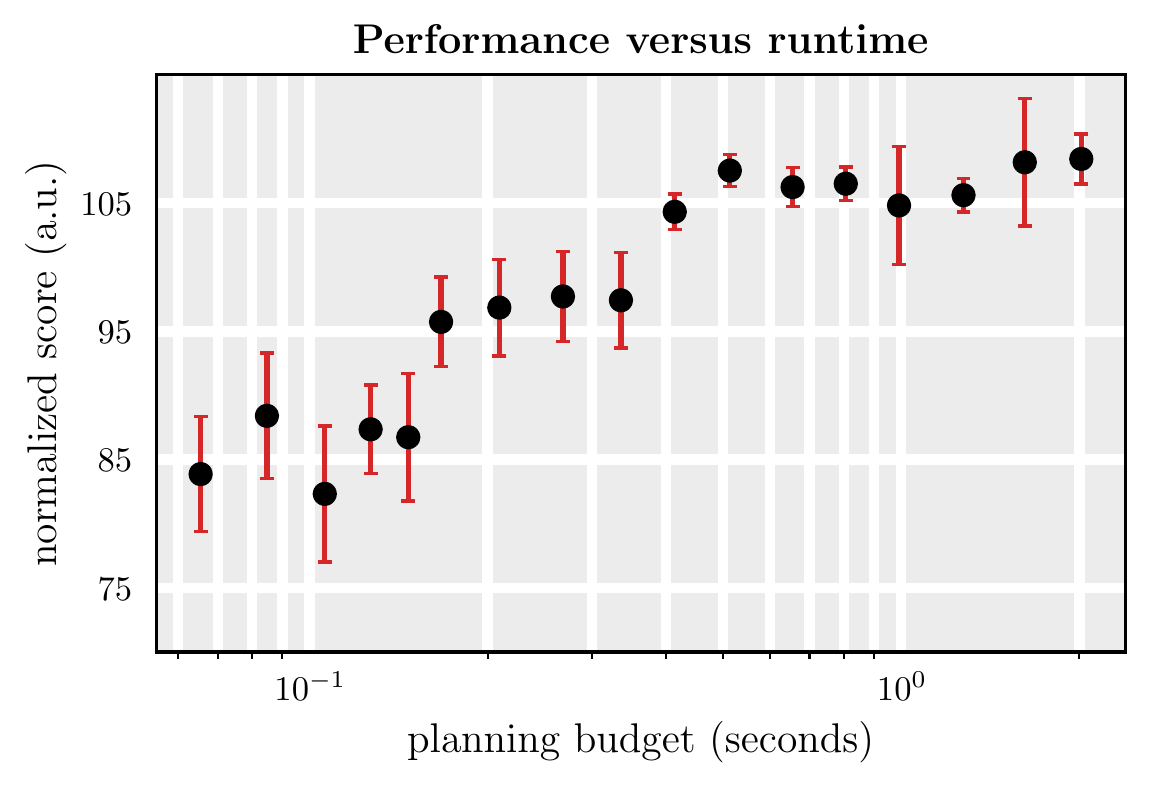}
    \caption{
    \textbf{(Warm-starting planning)}
    Performance of \method on Walker2d Medium-Expert when varying the number of diffusion steps to warm-start planning.
    Performance suffers only minimally even when using one-tenth the number of diffusion steps, as long as plans are initialized from the previous timestep's plan.}
    \label{fig:speed}
\end{figure}

Advances in deep generative modeling have recently made inroads into model-based reinforcement learning, with multiple lines of work exploring dynamics models parameterized as
convolutional U-networks \cite{kaiser2019mbatari},
stochastic recurrent networks \cite{ke2018modeling,hafner2018planet,ha2018worldmodels},
vector-quantized autoencoders \cite{hafner2021mastering,ozair2021vector},
neural ODEs \cite{du2020ode},
normalizing flows \cite{rhinehart2020imitative,janner2020gamma},
generative adversarial networks \cite{eysenbach2021mismatched},
energy-based models (\emph{EBMs}; \citealt{du2019model}),
graph neural networks  \cite{sanchez-gonzalez2018graph},
neural radiance fields \cite{li2021scene},
and Transformers \citep{janner2021sequence,chen2021transdreamer}.
Further, \citealt{lambert2020learning} have studied non-autoregressive trajectory-level dynamics models for long-horizon prediction.
These investigations generally assume an abstraction barrier between the model and planner.
Specifically, the role of learning is relegated to approximating environment dynamics; once learning is complete the model may be inserted into any of a variety of planning \cite{botev2013cem,grady2015mppi} or policy optimization \cite{sutton1990dyna,wang2019benchmarking} algorithms because the form of the planner does not depend strongly on the form of the model.
Our goal is to break this abstraction barrier by designing a model and planning algorithm that are trained alongside one another,
resulting in a non-autoregressive trajectory-level model for which sampling and planning are nearly identical.

A number of parallel lines of work have studied how to break the abstraction barrier between model learning and planning in different ways.
Approaches include training an autoregressive latent-space model for reward prediction \citep{tamar2016vin, oh2017vpn,schrittwieser2019muzero}; weighing model training objectives by state values \citep{farahmand2017valueaware}; and applying collocation techniques to learned single-step energies \citep{du2019model, rybkin2021model}.
In contrast, our method plans by modeling and generating all timesteps of a trajectory concurrently, instead of autoregressively, and conditioning the sampled trajectories with auxiliary guidance functions.


Diffusion models have emerged as a promising class of generative model that formulates the data-generating process as an iterative denoising procedure \citep{sohldickstein2015nonequilibrium, ho2020denoising}.
The denoising procedure can be seen as parameterizing the gradients of the data distribution \citep{song2019generative}, connecting diffusion models to score matching \cite{hyvarinen2005score} and EBMs \citep{lecun06atutorial, du2019implicit, nijkamp2019learning,grathwohl2020stein}.
Iterative, gradient-based sampling lends itself towards flexible conditioning \citep{dhariwal2021diffusion} and compositionality \citep{du2020compositional}, which we use to recover effective behaviors from heterogeneous datasets and plan for reward functions unseen during training.
While diffusion models have been developed for the generation of images \cite{song2020denoising}, waveforms \cite{chen2020wavegrad}, 3D shapes \cite{zhou2021shape}, and text \cite{austin2021structured}, to the best of our knowledge they have not previously been used in the context of reinforcement learning or decision-making.
\vspace{-.15cm}
\section{Discussion}
\label{sec:discussion}
This chapter has proposed \method, a denoising diffusion model for trajectory data.
Unlike conventional dynamics models, \method lends itself to a learned, non-greedy planning procedure that allows it so solve long-horizon tasks.
Planning with \method is almost identical to sampling from it, differing only in the addition of auxiliary perturbation functions that serve to guide samples.
The learned diffusion-based planning procedure has a number of useful properties, including graceful handling of sparse rewards, the ability to plan for new rewards without retraining, and a temporal compositionality that allows it to produce out-of-distribution trajectories by stitching together in-distribution subsequences.

Our results point to a new class of diffusion-based planning procedures for deep model-based reinforcement learning.
However, there are genuine tradeoffs between Diffuser and the Transformer-based approach from Chapter~\ref{ch:transformer}.
While Diffuser comes equipped with a much more efficient learned planner than beam search, it comes at the cost of reduced open-loop predictive accuracy.
This discrepancy arises due to differences in the underlying generative models.
Diffusion models' success in image generation has been attributed to an increased importance placed on low-frequency information, a consequence of the denoising schedule and parameter-sharing across noise levels \citep{song2019ncsn,dieleman2022diffusion}.
In contrast, Transformers are more readily able to represent the types of high-frequency information that can disproportionately affect trajectories, such as joint angles near a contact point, but are limited in planning contexts due to impoverished search procedures.
It remains an open question, both in the context of decision-making and broadly in generative modeling, how to best combine the effectiveness of the Transformer architecture with the strengths of non-autoregressive denoising strategies.

\chapter{Conclusion}
\renewcommand{\chaptertitle}{Conclusion}
\label{ch:conclusion}

In this thesis, we considered the role of generative models in the reinforcement learning problem.
We began by asking what prediction problem the generative model should be trained to solve, and in the process developed a continuous-state dynamics model that can predict over probabilistic, infinite horizons.
However, when generalizing the standard single-step prediction problem to that of infinite-horizon prediction, we found that the bottleneck to scaling lay with the quality of the generative model itself.
As a result, we studied the extent to which improved generative modeling capabilities could be used to enhance the performance of a planning algorithm, designing a planner around the toolbox of large-scale language modeling.
While this approach markedly increased predictive accuracy, planning capabilities lagged behind due to the inefficiency of language model decoding approaches, making the new bottleneck the planning algorithm in which the generative model was embedded.

The main difficulty with such an approach was that the quality of the planner was constant; because the model itself could improve with data but the planner could not, the model would continue to become more accurate until the planner became the performance-limiting component.
To address this, we described a method of incorporating the entire planning algorithm into the generative model itself, as opposed to embedding a generative model into a conventional planner.
This design allowed the plans, and not just the predictive quality, to improve with more data and experience, and we demonstrated that this approach was effective in long-horizon and sparse-reward settings where conventional planning algorithms struggle.

As with many research programs, the work in this thesis raises more questions than it answers.
We conclude by discussing some of the most promising directions for future work in the intersection of reinforcement learning and generative modeling:
\begin{itemize}
    \item \textbf{Planning in abstracted time.}
        The $\gamma$-model allows us to ask questions of the form ``will I reach a particular state at some point in the future'' without needing to worry about precisely when that state will be reached.
        This affordance is a natural fit for long-horizon decision-making.
        For example, a plan for a mobile robot would plausibly need to include battery recharging after some amount of time.
        The ability to incorporate such a step in its plan, without needing to precommit to the exact second in which this step must occur, would allow the planner to hone in on higher-level decisions to be made and their ordering as opposed to unnecessary (and inevitably misleading) precision.
        This idea shares much of its inspiration with the line of work describing options \citep{sutton1999semimdps}; scaling such methods to high-dimensional control problems poses an exciting, and likely fruitful, challenge.
    \item \textbf{Planning in abstracted states.}
        The methods described in this thesis plan in raw state space.
        In large observation spaces, this becomes not only unnecessary  -- as most details in an observation are unlikely to be relevant to the task at hand -- but also possibly infeasible in the event that the space is so large that high-quality generative modeling once again becomes the bottleneck.
        Instead, it would be much more efficient to plan in a compressed space pruned of unnecessary details.
        There has been much work on learned latent spaces, with the usual spectrum  ranging from those trained with reconstructive objectives \citep{hafner2018planet,janner2019mbpo} to those that contain only value-relevant information \citep{grimm2020value} and options that interpolate between these two extremes \citep{farahmand2017valueaware,eysenbach2021mismatched}.
        However, it remains an open problem how to combine these works with learned planners, contemporary high-capacity generative models, and deploy them in the types of real-world problems that are often used as the motivating examples for learned latent spaces.
    \item \textbf{Improved learned planning algorithms.}
        One of the main contributions of this thesis is a learned planning algorithm based on iterative denoising.
        The main advantage of this method compared to conventional planners is its ability to improve with experience.
        However, just as there exist many classical planning algorithms with different strengths and weaknesses, there is no reason to believe that Diffuser is the end of the story for data-driven planning.
        Improved variants could incorporate dynamics Jacobians, more sophisticated constrained optimization techniques, or make use of tree-based search techniques for better coverage of the space of possibilities.
    \item \textbf{Amortized and adaptive planning.}
        As discussed in Chapter~\ref{ch:introduction}, model-based planners are appealing for their test-time flexibility and generalization properties.
        However, a fundamental limitation of planning-based approaches is their runtime costs, which can become prohibitively large for any problem requiring a high planning frequency.
        In comparison, model-free policies tend to be more efficient, often requiring only a single feedforward pass of a neural network as opposed to many such computations for the purpose of an iterative planner.
        To benefit from the respective strengths of both approaches, the cost of planning could be amortized over time by storing the result of the planning method into a lightweight policy and falling back on the planner only in novel situations.
        The open challenge, of course, is to design a reliable mechanism for determining when a situation calls for the slower planning approach and when it is safe to use the faster policy.
    \item \textbf{Foundation dynamics models.}
        The dynamics models and planners described in this thesis were trained and evaluated on individual tasks (\emph{e.g.}, in locomotion settings) or in small sets of tasks (\emph{e.g.}, in maze-solving or block-stacking domains).
        While common in reinforcement learning, this workflow is becoming increasingly rare in the rest of machine learning, in which models pretrained on large-scale datasets are finetuned (or used directly) on a wide variety of tasks.
        There is an opportunity to build something similar for dynamics models, especially in real-world robotics settings where the regularity of the physical world could enable the types of generalization present in large language modeling \citep{brown2020gpt3}.
        In the short term, this problem is largely one of dataset collection.
        In the longer term, this endeavour will open up new research problems we are not yet able to anticipate.
        A common refrain in deep learning is that something new breaks with every order-of-magnitude scale increase, and it is likely that the same will be true in robotic control.
\end{itemize}

\newpage
\subsection*{Open-Source Implementations}
\renewcommand{\chaptertitle}{Open-Source}
\addcontentsline{toc}{chapter}{Open-Source Implementations}

Code to reproduce the results in this thesis is available at the following webpages:
\begin{itemize}
    \item Chapter~\ref{ch:gamma}: \href{https://gammamodels.github.io/}
    {\texttt{gamma-models.github.io}}
    \item Chapter~\ref{ch:transformer}: \href{https://trajectory-transformer.github.io/}
    {\texttt{trajectory-transformer.github.io}}
    \item Chapter~\ref{ch:diffuser}: \href{https://diffusion-planning.github.io/}
    {\texttt{diffusion-planning.github.io}}
\end{itemize}

\subsection*{Code References}
\addcontentsline{toc}{chapter}{Code References}

The following libraries were used in this work:
NumPy \citep{harris2020numpy},
PyTorch \citep{paszke2019pytorch},
JAX \citep{jax2018github},
Flax \citep{flax2020github},
einops \citep{rogozhnikov2022einops}, 
MuJoCo \citep{todorov2012mujoco}, 
mujoco-py \citep{openai2016mujocopy},
Gym \citep{openai2016gym},
minGPT \citep{karpathy2020mingpt},
and Diffusion Models in PyTorch \citep{wang2020ddpm}.

\newpage
\titleformat{\chapter}[display]
{\Huge\bfseries\centering}
{}
{}
{\titlerule[3pt]\vspace{0.5em}}
[\vspace{0.65em}{\titlerule[1pt]}\vspace{-.25em}]

\renewcommand{\chaptertitle}{Bibliography}
\bibliography{references}
\bibliographystyle{setup/icml}

\appendix

\begin{appendices}

\titleformat{\chapter}[display]
{\Huge\bfseries\centering}
{
    \raisebox{0.175em}{\rule{0.325\linewidth}{3pt}}
    \parbox[t]{0.3\linewidth}{\centering\Huge Appendix \thechapter}
    \raisebox{0.175em}{\rule{0.335\linewidth}{3pt}}
}
{-0.95em}
{}
[\vspace{0.65em}{\titlerule[1pt]}\vspace{-.25em}]

\newpage

\chapter{$\boldgamma$-Model Details}
\renewcommand{\chaptertitle}{Appendix A: $\boldgamma$-Model}

\setcounter{thm}{0}

\section{Geometric weighting lemma}
\label{app:geometric_lemma}
\begin{lem}
\label{lem:gamma_mve_weights}
Let $\alpha_n$ be importance weights as described in Theorem~\ref{thm:weights}:
$\alpha_n = \frac{(1-\gammav)(\gammav - \gammam)^{n-1}}{(1-\gammam)^n}$.
Then:
\begin{equation*}
1 - \sum_{n=1}^{H} \alpha_n = \left( \frac{\gammav-\gammam}{1-\gammam} \right)^{\!H}
\end{equation*}
\end{lem}
\vspace{-.3cm}
\begin{proof}
\begin{align*}
1 - \sum_{n=1}^{H} \alpha_n
&= 1 - \left( \frac{1-\gammav}{\gammav-\gammam} \right) \sum_{n=1}^{H}  \left( \frac{\gammav - \gammam}{1-\gammam} \right)^{\!n} \\
&= 1 -  \left( \frac{1-\gammav}{\gammav-\gammam} \right)
\frac{ \left( \frac{\gammav - \gammam}{1-\gammam} \right) - \left(\frac{\gammav - \gammam}{1-\gammam}\right)^{\!H+1}}{ \frac{1-\gammav}{1-\gammam} } \\
&= 1 -  \left( \frac{1-\gammam}{\gammav-\gammam} \right) \left(
\left( \frac{\gammav - \gammam}{1-\gammam} \right) - \left(\frac{\gammav - \gammam}{1-\gammam}\right)^{\!H+1} \right) \\
&= \left(\frac{\gammav - \gammam}{1-\gammam}\right)^{\!H}
\end{align*}
\end{proof}




\section{Implementation Details}
\label{app:implementation}

\paragraph{$\boldgamma$-MVE algorithmic description.}
The $\gamma$-MVE estimator may be used for value estimation in any actor-critic algorithm.
We describe the variant used in our control experiments, in which it is used in the soft actor critic algorithm (SAC; \citealt{haarnoja18sac}), in Algorithm~\ref{alg:gamma_mve}.
The $\gamma$-model update is unique to $\gamma$-MVE; the objectives for the value function and policy are identical to those in SAC.
The objective for the $Q$-function differs only by replacing $V(\stp)$ with $V_{\gamma-\text{MVE}}(\stp)$.
For a detailed description of how the gradients of these objectives may be estimated, and for hyperparameters related to the training of the $Q$-function, value function, and policy, we refer to \cite{haarnoja18sac}. 

\paragraph{Network architectures.}
For all GAN experiments, the $\gamma$-model generator $\model$ and discriminator $\disc$ are instantiated as two-layer MLPs with hidden dimensions of 256 and leaky ReLU activations. For all normalizing flow experiments, we use a six-layer neural spline flow \citep{durkan2019spline} with 16 knots defined in the interval $[-10, 10]$.
The rational-quadratic coupling transform uses a three-layer MLP with hidden dimensions of 256.

\paragraph{Hyperparameter settings.}
We include the hyperparameters used for training the GAN $\gamma$-model in Table~\ref{tbl:hyperparameters_gan} and the flow $\gamma$-model in Table~\ref{tbl:hyperparameters_flow}.

{
\renewcommand{\arraystretch}{1.4}
\begin{table}
\centering
\caption{GAN $\gamma$-model hyperparameters (Algorithm~\ref{alg:practical_samples}).}
\label{tbl:hyperparameters_gan}
\begin{tabular}{p{8cm}|p{1.75cm}}
\hline
    \textbf{Parameter} & \textbf{Value}\\
\hline
    Batch size & 128 \\
    Number of $\se$ samples per $(\st, \at)$ pair & 512 \\
    Delay parameter $\tau$ & $5 \cdot 10^{-3}$ \\
    Step size $\lambda$ & $1 \cdot 10^{-4}$ \\
    Replay buffer size (off-policy prediction experiments) & $2 \cdot 10^{5}$ \\
\hline
\end{tabular} \\
\vspace{.2cm}
\end{table}
}

{
\renewcommand{\arraystretch}{1.4}
\begin{table}
\centering
\caption{Flow $\gamma$-model hyperparameters (Algorithm~\ref{alg:practical_logp})}
\label{tbl:hyperparameters_flow}
\begin{tabular}{p{8cm}|p{1.75cm}}
\hline
    \textbf{Parameter} & \textbf{Value}\\
\hline
    Batch size & 1024 \\
    Number of $\se$ samples per $(\st, \at)$ pair & 1 \\
    Delay parameter $\tau$ & $5 \cdot 10^{-3}$ \\
    Step size $\lambda$ & $1 \cdot 10^{-4}$ \\
    Replay buffer size (off-policy prediction experiments) & $2 \cdot 10^{5}$ \\
   Single-step Gaussian variance $\sigma^2$ & $1 \cdot 10^{-2}$ \\
\hline
\end{tabular} \\
\vspace{.2cm}
\end{table}
}

We found the original GAN \citep{goodfellow2014gan} and the least-squares GAN \citep{mao2016lsgan} formulation to be equally effective for training $\gamma$-models as GANs.

\section{Environment Details}
\label{app:environments}

\textbf{Acrobot-v1} is a two-link system \citep{sutton1996generalization}. The goal is to swing the lower link above a threshold height. The eight-dimensional observation is given by $[\cos \theta_0, \sin \theta_0, \cos \theta_1, \sin \theta_1, \frac{\mathrm{d}}{\mathrm{d}t}\theta_0, \frac{\mathrm{d}}{\mathrm{d}t}\theta_1]$. We modify it to have a one-dimensional continuous action space instead of the standard three-dimensional discrete action space. We provide reward shaping in the form of $r_\text{shaped}=-\cos \theta_0 - \cos (\theta_0 + \theta_1)$.

\textbf{MountainCarContinuous-v0} is a car on a track \citep{moore1990efficient}.
The goal is to drive the car up a high too high to summit without built-up momentum.
The two-dimmensional observation space is $[x, \frac{\mathrm{d}}{\mathrm{d}t}x]$. We provide reward shaping in the form of $r_\text{shaped}=x$.

\textbf{Pendulum-v0} is a single-link system. The link starts in a random position and the goal is to swing it upright. The three-dimensional observation space is given by $[\cos\theta, \sin\theta, \frac{\mathrm{d}}{\mathrm{d}t}\theta]$.

{
\newcommand{\be}{\mathbf{e}}
\renewcommand{\bg}{\mathbf{g}}
\textbf{Reacher-v2} is a two-link arm.
The objective is to move the end effector $\be$ of the arm to a randomly sampled goal position $\bg$. The 11-dimensional observation space is given by 
$[\cos\theta_0, \cos\theta_1, \sin\theta_0, \sin\theta_1, \bg_{x},\bg_{y}, \frac{\mathrm{d}}{\mathrm{d}t}\theta_0, \frac{\mathrm{d}}{\mathrm{d}t}\theta_1, \be_x-\bg_x, \be_y-\bg_y, \be_z-\bg_z]$.
}

Model-based methods often make use of shaped reward functions during model-based rollouts \citep{chua2018pets}. For fair comparison, when using shaped rewards we also make the same shaping available to model-free methods.


\section{Adversarial \texorpdfstring{$\boldsymbol{\gammam}$}{Gamma}-Model Predictions}
\label{app:gan_predictions}
\begin{figure}[H]
    \centering
    \includegraphics[width=1.0\linewidth]{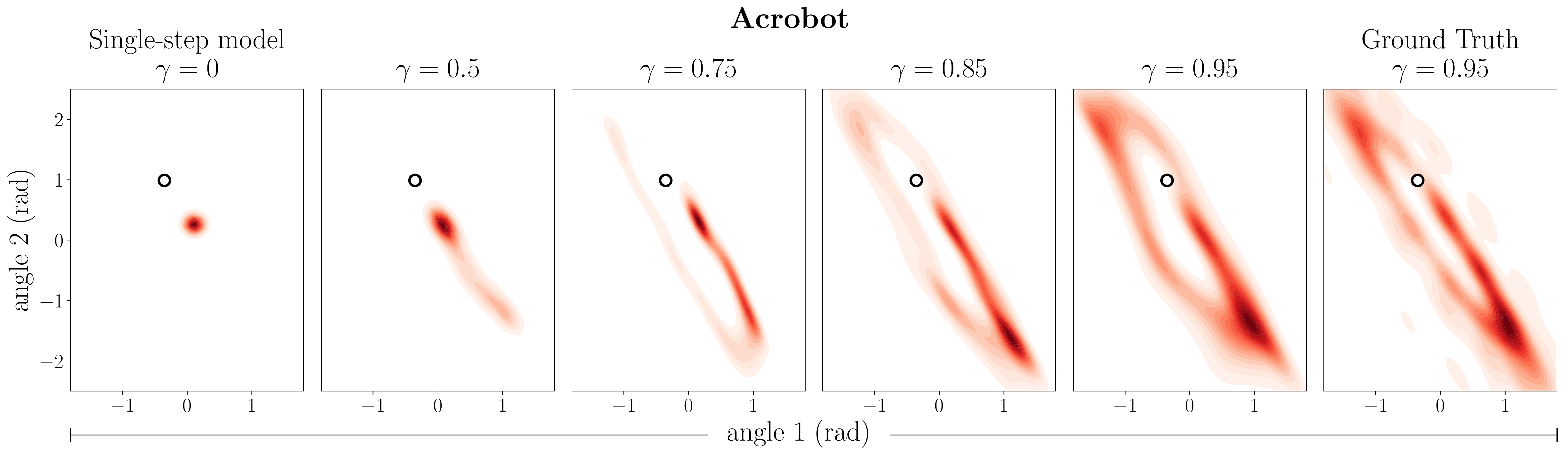} \\
    \vspace{0.2cm}
    \includegraphics[width=1.0\linewidth]{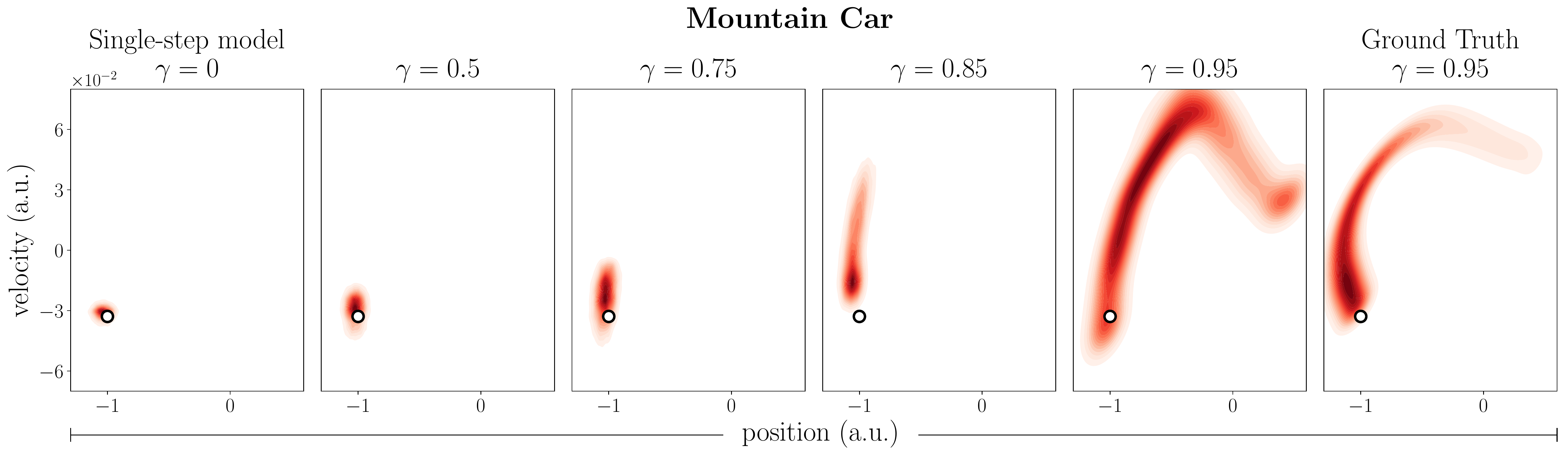}
    \caption{
    \textbf{(Adversarial $\boldgamma$-model predictions)}
    Visualization of the distribution from a single feedforward pass of $\gamma$-models trained as GANs according to Algorithm~\ref{alg:practical_samples}.
    GAN-based $\gamma$-models tend to be more unstable than normalizing flow $\gamma$-models, especially at higher discounts.
    }
    \label{fig:density_gan}
\end{figure}


\newpage
\chapter{Trajectory Transformer Details}
\renewcommand{\chaptertitle}{Appendix B: Trajectory Transformer}
\section{Model and Training Specification}
\label{sec:details}

\paragraph{Architecture and optimization details.}
In all environments, we use a Transformer architecture with four layers and four self-attention heads.
The total input vocabulary of the model is $V \times (N + M + 2)$ to account for states, actions, rewards, and rewards-to-go, but the output linear layer produces logits only over a vocabulary of size $V$; output tokens can be interpreted unambiguously because their offset is uniquely determined by that of the previous input.
The dimension of each token embedding is 128.
Dropout is applied at the end of each block with probability $0.1$.

We follow the learning rate scheduling of \citep{radford2018improving}, increasing linearly from 0 to $\num{2.5e-4}$ over the course of $2000$ updates.
We use a batch size of 256. 

\paragraph{Hardware.}
Model training took place on NVIDIA Tesla V100 GPUs (NCv3 instances on Microsoft Azure) for 80 epochs, taking approximately 6-12 hours (varying with dataset size) per model on one GPU.

\section{Discrete Oracle}
\label{app:oracle}
The discrete oracle in Figure~\ref{fig:model_error} is the maximum log-likelihood attainable by a model under the uniform discretization granularity.
For a single state dimension $i$, this maximum is achieved by a model that places all probability mass on the correct token, corresponding to a uniform distribution over an interval of size
\[
\frac{r_i - \ell_i}{V}.
\]
The total log-likelihood over the entire state is then given by:
\[
\sum_{i=1}^{N} \log\frac{V}{r_i - \ell_i}.
\]

\section{Baseline performance sources}
\label{app:baselines}

\paragraph{Offline reinforcement learning}
The results for CQL, IQL, and DT are from Table 1 in \citet{kostrikov2021implicit}.
The results for MBOP are from Table 1 in \citet{argenson2020model}.
The results for BRAC are from Table 2 in \cite{fu2020d4rl}.
The results for BC are from Table 1 in \citet{kumar2020conservative}.

\section{Datasets}
\label{app:data}
The D4RL dataset~\citep{fu2020d4rl} used in our experiments is under the Creative Commons Attribution 4.0 License (CC BY). The license information can be found at

\begin{centering}
    \url{https://github.com/rail-berkeley/d4rl/blob/master/README.md} \\
\end{centering}
\vspace{.2cm}

under the ``Licenses'' section. 

\section{Beam Search Hyperparameters}

\begin{center}
\def\arraystretch{1.35}
\begin{tabular}{|l|l|c|} 
\hline
\textbf{Beam width} & maximum number of hypotheses retained during beam search & 256 \\
\textbf{Planning horizon} & number of transitions predicted by the model during & 15 \\
\textbf{Vocabulary size} & number of bins used for autoregressive discretization & 100 \\
\textbf{Context size} & number of input $(\st, \at, \rt, R_t)$ transitions & 5 \\
$\bm{k_\textbf{obs}}$ & top-$k$ tokens from which observations are sampled & 1 \\
$\bm{k_\textbf{act}}$ & top-$k$ tokens from which actions & 20 \\
\hline
\end{tabular}
\end{center}

Beam width and context size are standard hyperparameters for decoding Transformer language models.
Planning horizon is a standard trajectory optimization hyperparameter.
The hyperparameters $k_\text{obs}$ and  $k_\text{act}$ indicate that actions are sampled from the most likely $20\%$ of action tokens and next observations are decoded greedily conditioned on previous observations and actions.

In many environments, the beam width and horizon may be reduced to speed up planning without affecting performance.
Examples of these configurations are provided in the reference implementation: \href{https://github.com/JannerM/trajectory-transformer}{\texttt{github.com/jannerm/trajectory-transformer}}.

\section{Goal-Reaching on Procedurally-Generated Maps}
\label{app:minigrid}

The method evaluated here and the experimental setup is identical to that described in Section 3.2 (Goal-conditioned reinforcement learning), with one distinction: because the map changes each episode, the Transformer model has an additional context embedding that is a function of the current observation image. This embedding is the output of a small convolutional neural network and is added to the token embeddings analogously to the treatment of position embeddings. The agent position and goal state are not included in the map; these are provided as input tokens as described in Section 3.2.

The action space of this environment is discrete. There are seven actions, but only four are required to complete the tasks: turning left, turning right, moving forward, and opening a door. The training data is a mixture of trajectories from a pre-trained goal-reaching policy and a uniform random policy.

94\% of testing goals are reached by the model on held-out maps. Example paths are shown in Figure~\ref{fig:minigrid}.

\begin{figure}
    \centering
    \includegraphics[width=0.25\linewidth]{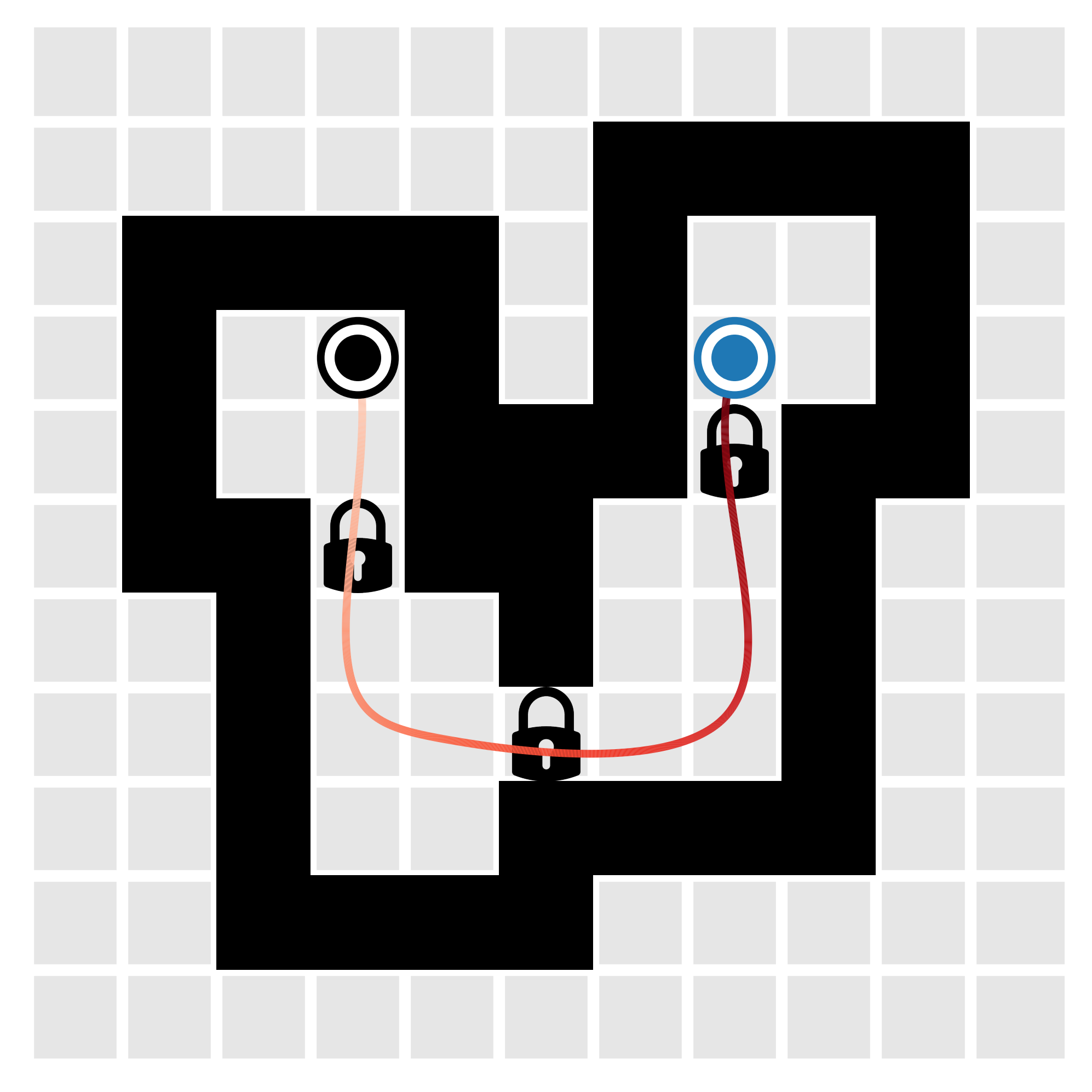}
    \includegraphics[width=0.25\linewidth]{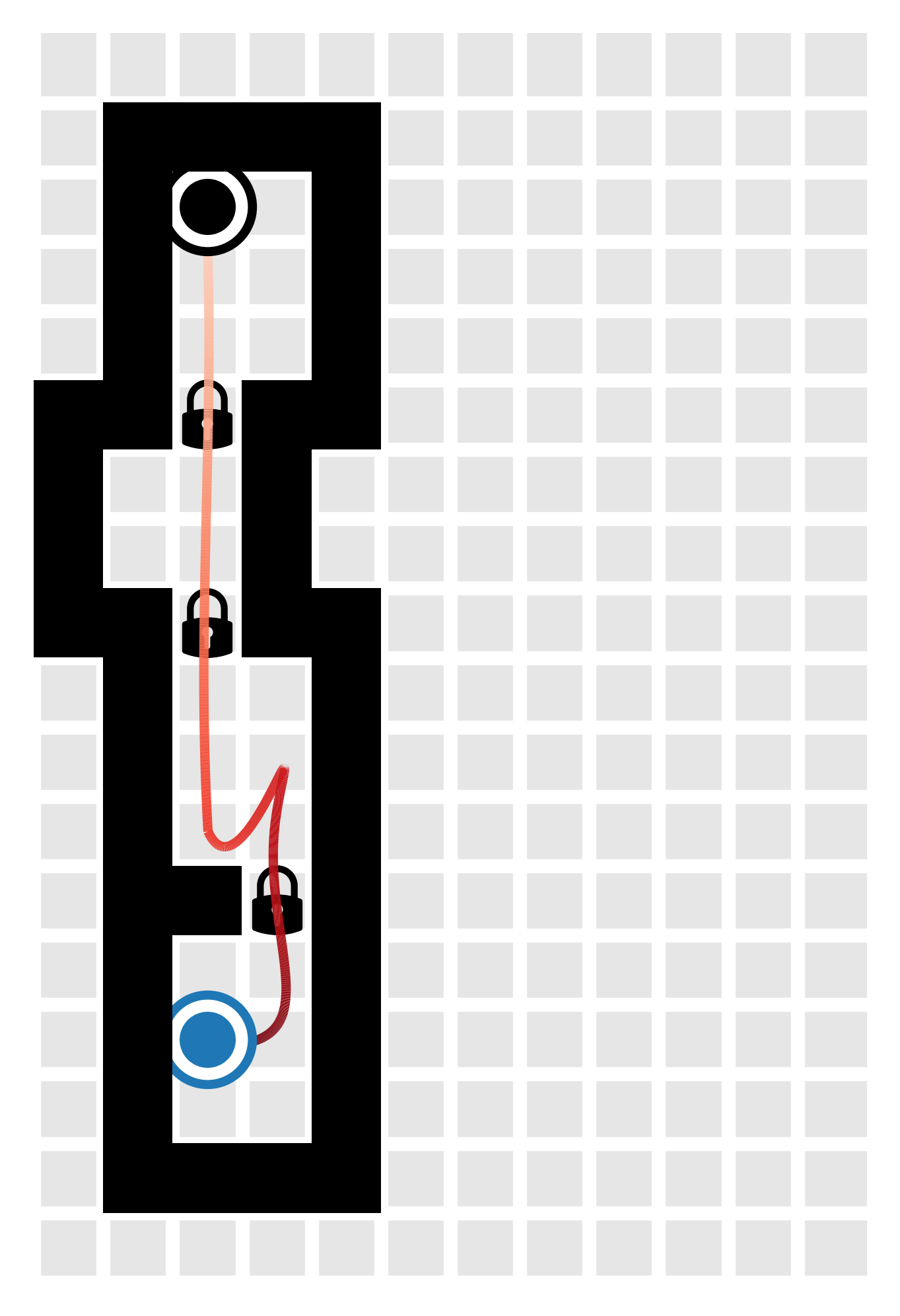}
    \includegraphics[width=0.25\linewidth]{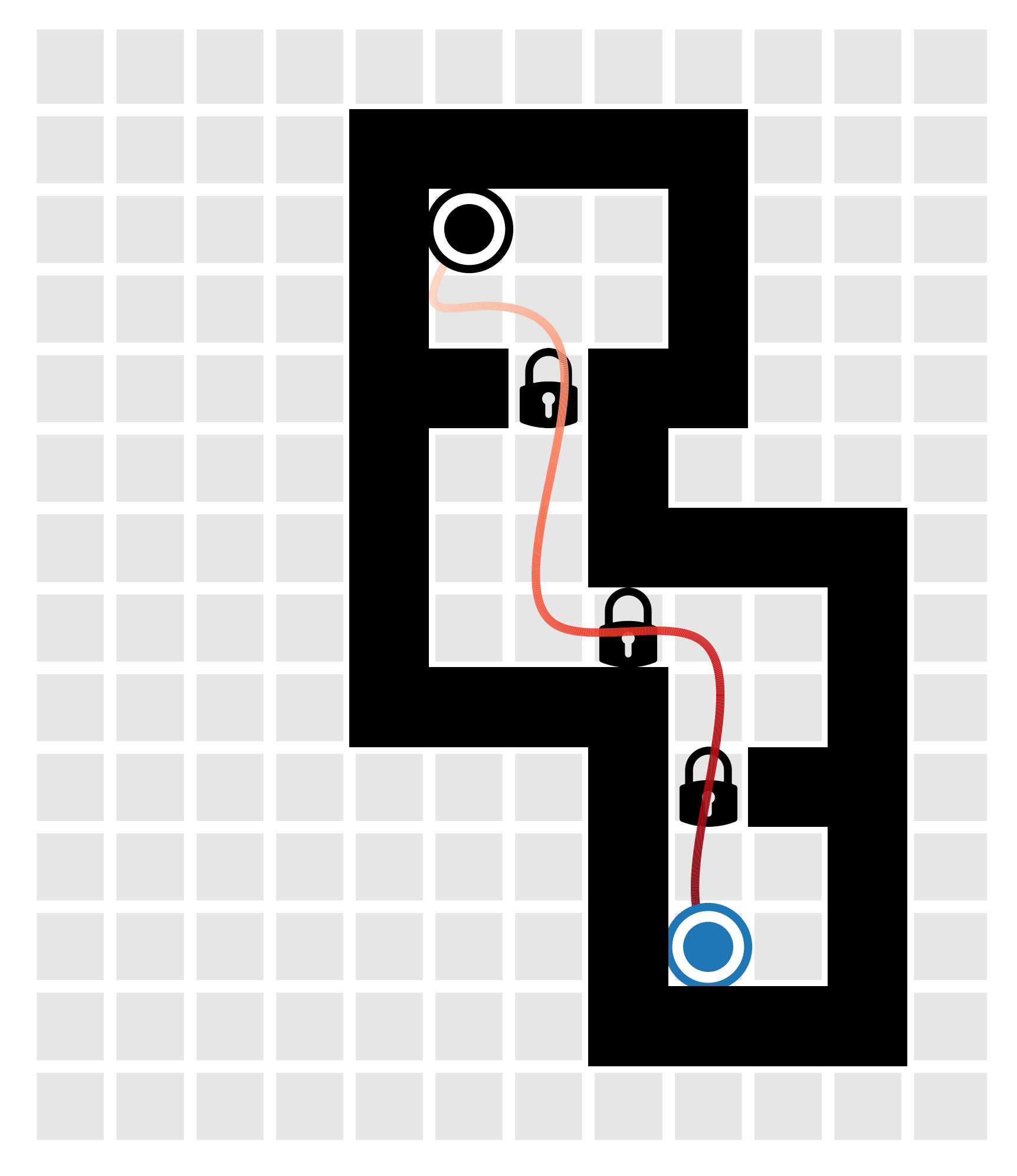} \\
    \includegraphics[width=0.25\linewidth]{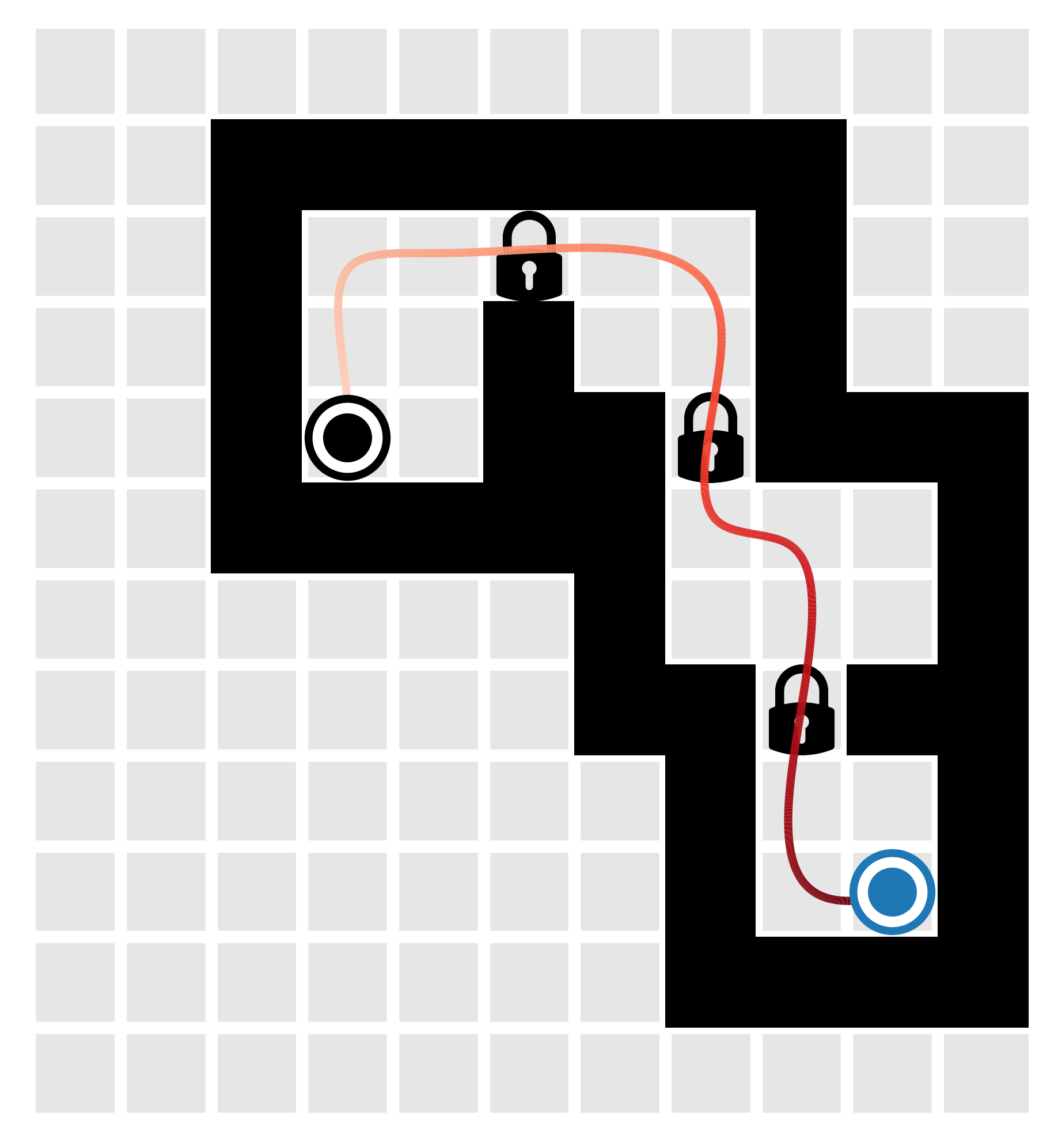}
    \includegraphics[width=0.25\linewidth]{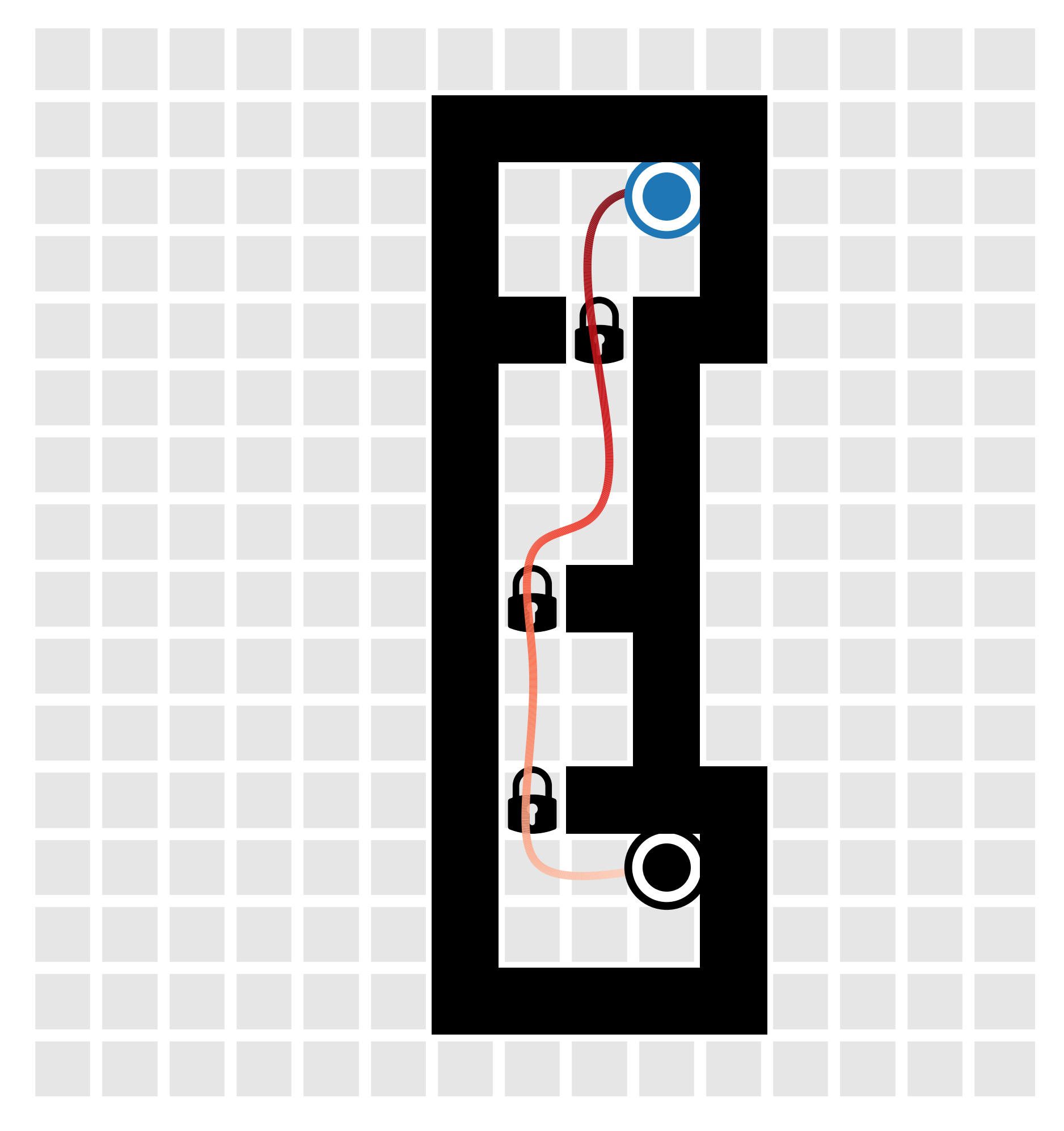}
    \includegraphics[width=0.25\linewidth]{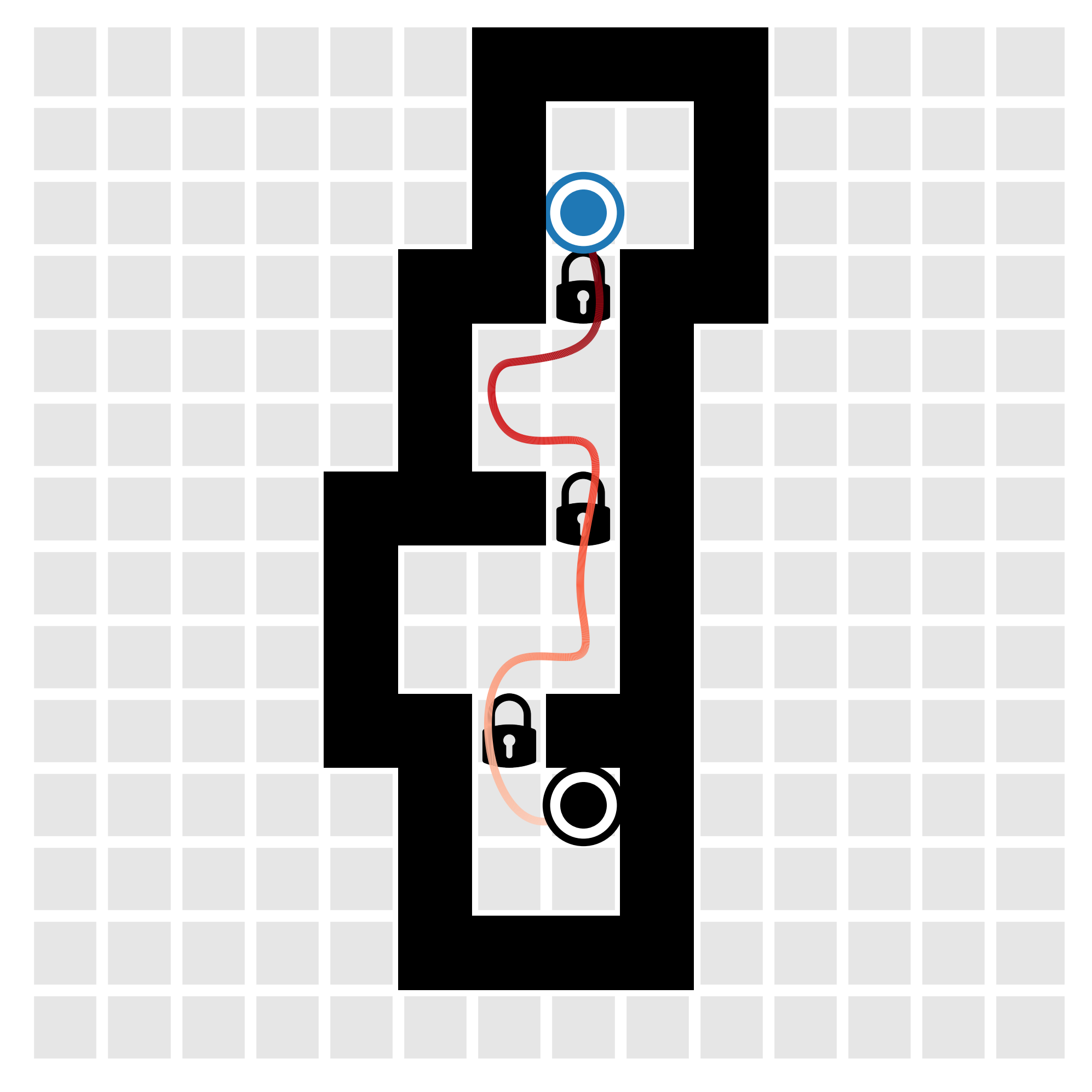} \\
    \includegraphics[width=0.25\linewidth]{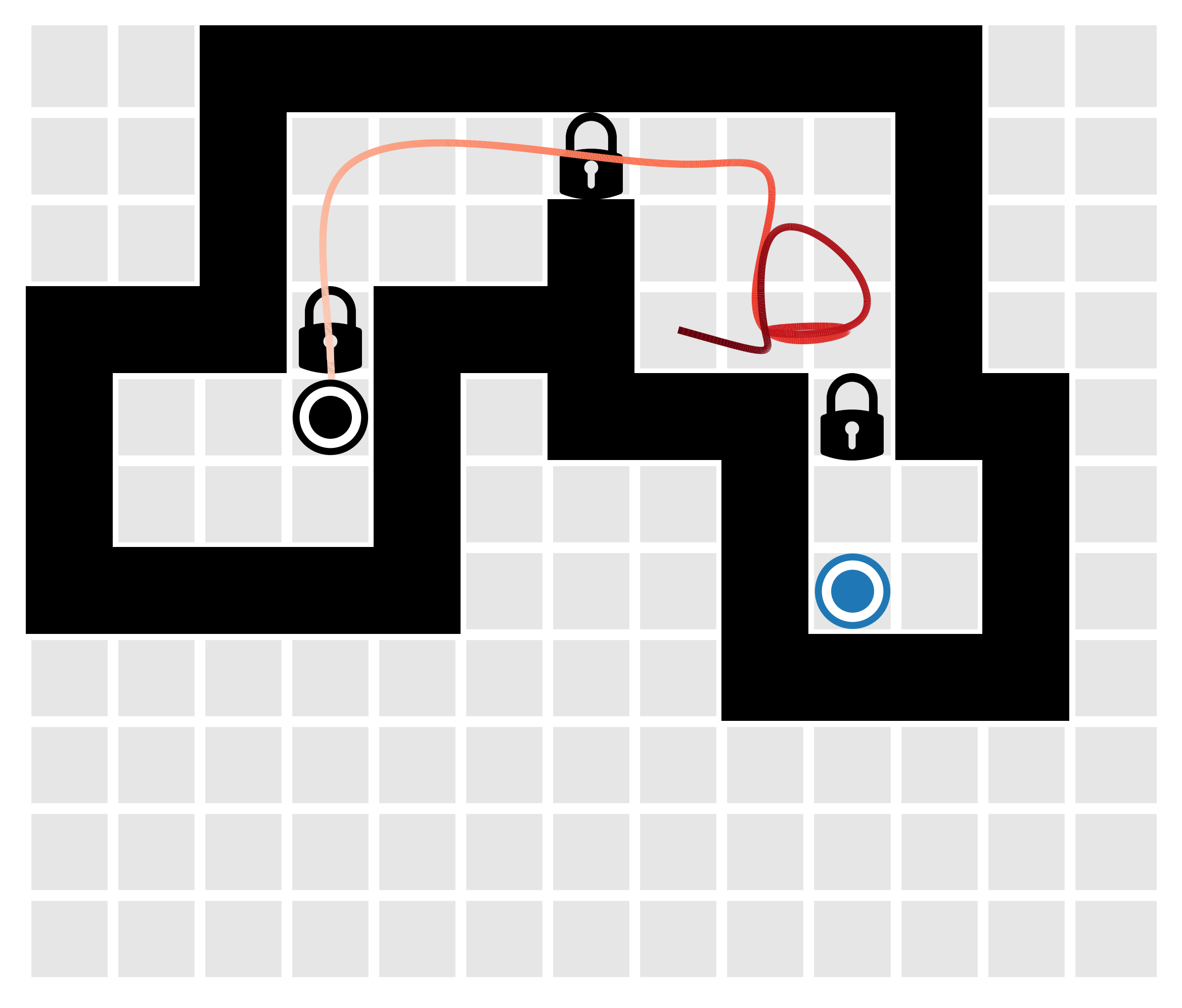}
    \includegraphics[width=0.25\linewidth]{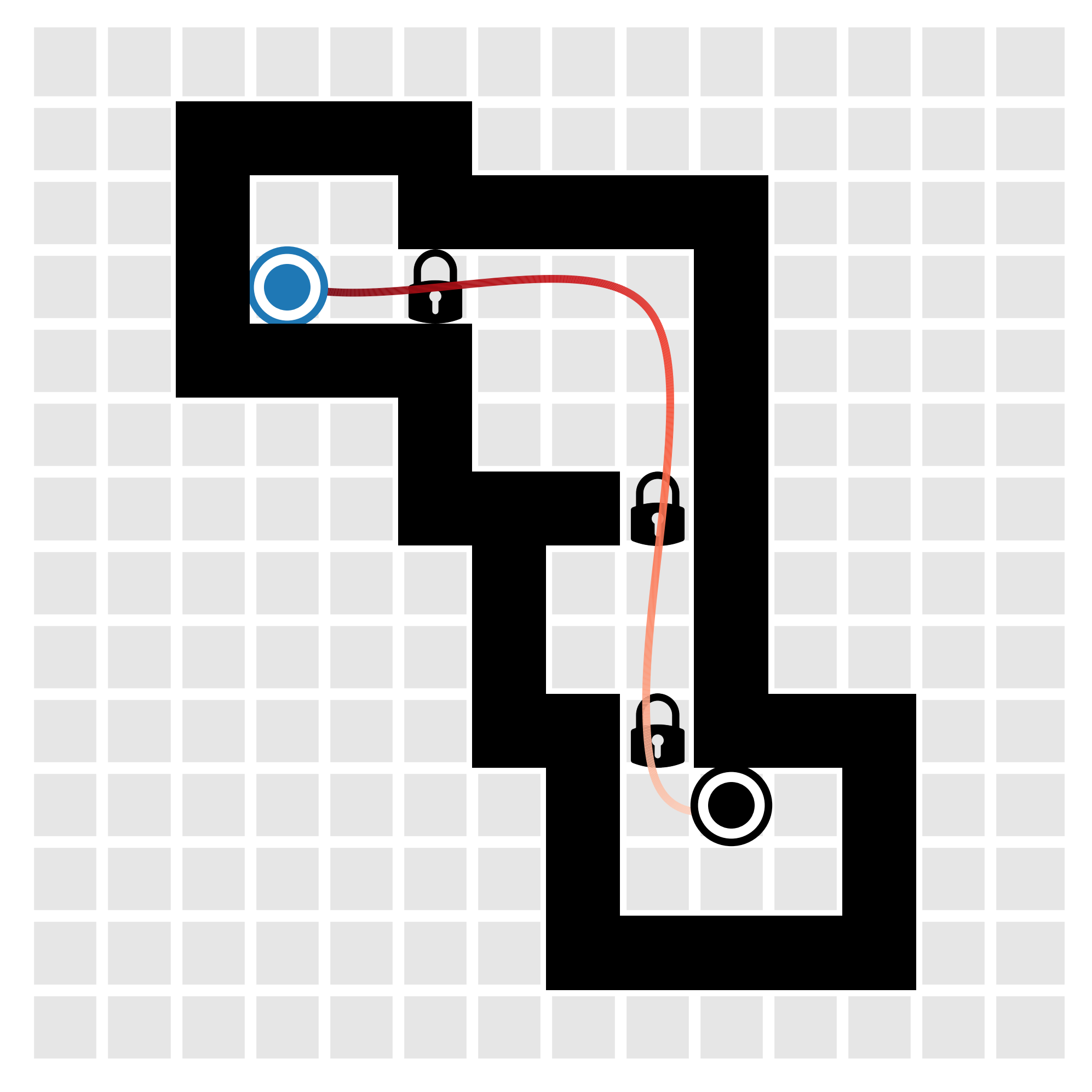}
    \includegraphics[width=0.25\linewidth]{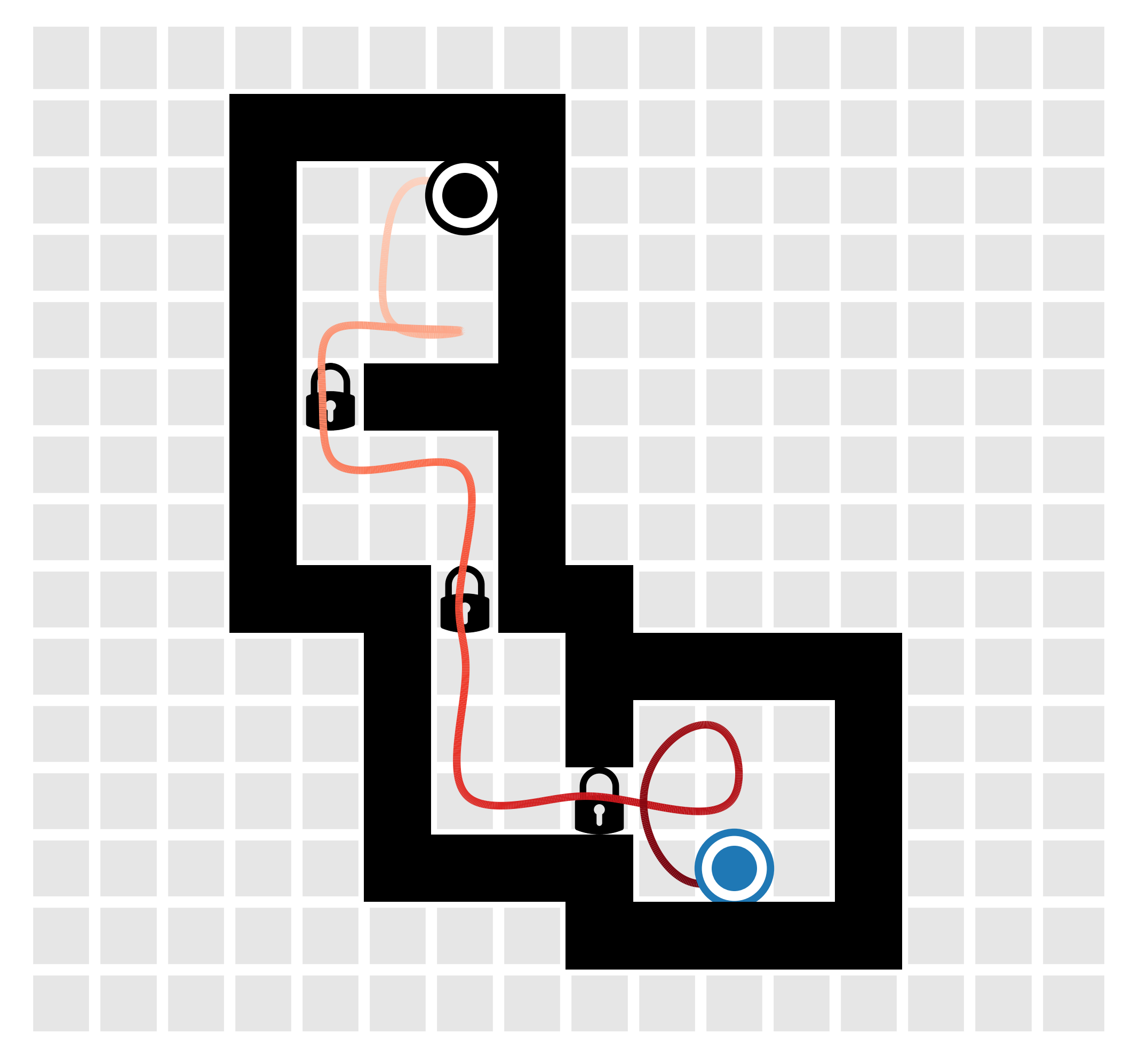} \\
    \caption{
        \textbf{(Minigrid rollouts)}
        Example paths of the Trajectory Transformer planner in the \texttt{MiniGrid-MultiRoom-N4-S5}. Lock symbols indicate doors.
    }
    \label{fig:minigrid}
\end{figure}


\newpage
\chapter{Diffuser Details}
\renewcommand{\chaptertitle}{Appendix C: Diffuser}
\section{Baseline details and sources}
\label{app:baselines}
\label{app:multitask}

In this section, we provide details about baselines we ran ourselves.
For scores of baselines previously evaluated on standardized tasks, we provide the source of the listed score.

\subsection{Maze2D experiments}

\paragraph{Single-task.}

The performance of CQL and IQL on the standard Maze2D environments is reported in the D4RL whitepaper \citep{fu2020d4rl} in Table~2.

We ran IQL using the offical implementation from the authors: 
\begin{center}
{
    \small
    \href{https://github.com/ikostrikov/implicit_q_learning}
    {\texttt{github.com/ikostrikov/implicit\_q\_learning}}.
}
\end{center}

We tuned over two hyperparameters:
\begin{enumerate}
    \item temperature $\in [3, 10]$
    \item expectile $\in [0.65, 0.95]$
\end{enumerate}

\paragraph{Multi-task.}

We only evaluated IQL on the Multi2D environments because it is the strongest baseline in the single-task Maze2D environments by a sizeable margin.
To adapt IQL to the multi-task setting, we modified the $Q$-functions, value function, and policy to be goal-conditioned.
To select goals during training, we employed a strategy based on hindsight experience replay, in which we sampled a goal from among those states encountered in the future of a trajectory.
For a training backup $(\st, \at, \stp)$, we sampled goals according to a geometric distribution over the future
\[
\Delta \sim \text{Geom}(1-\gamma) ~~~~~~~~ \mathbf{g} = \bs_{t + \Delta},
\]
recalculated rewards based on the sampled goal, and conditioned all relevant models on the goal during updating.
During testing, we conditioned the policy on the ground-truth goal.

We tuned over the same IQL parameters as in the single-task setting.

\subsection{Block stacking experiments}

\paragraph{Single-task.}

We ran CQL using the following implementation
\begin{center}
{
    \small
    \href{https://github.com/young-geng/cql}
    {\texttt{https://github.com/young-geng/cql}}.
}
\end{center}
and used default hyperparameters in the code.
We ran BCQ using the author's original implementation
\begin{center}
{
    \small
    \href{https://github.com/sfujim/BCQ}
    {\texttt{https://github.com/sfujim/BCQ}}.
}
\end{center}

For BCQ, we tuned over two hyperparameters:
\begin{enumerate}
    \item discount factor $\in [0.9, 0.999]$
    \item tau $\in [0.001, 0.01]$
\end{enumerate}

\paragraph{Multi-task.}

To evaluate BCQ and  CQL in the multi-task setting, we modified the $Q$-functions, value function and policy to be goal-conditioned.
We trained using goal relabeling as in the Multi2D environments.
We tuned over the same hyperparameters described in the single-task block stacking experiments.

\subsection{Offline Locomotion}
\label{app:d4rl_sources}

The scores for BC, CQL, IQL, and AWAC are from Table~1 in \citet{kostrikov2021implicit}.
The scores for DT are from Table~2 in \citet{chen2021decision}.
The scores for TT are from Table~1 in \citet{janner2021sequence}.
The scores for MOReL are from Table~2 in \citet{kidambi2020morel}.
The scores for MBOP are from Table~1 in \citet{argenson2020model}.

\begin{figure}[t]
\centering
\includegraphics[width=\columnwidth]{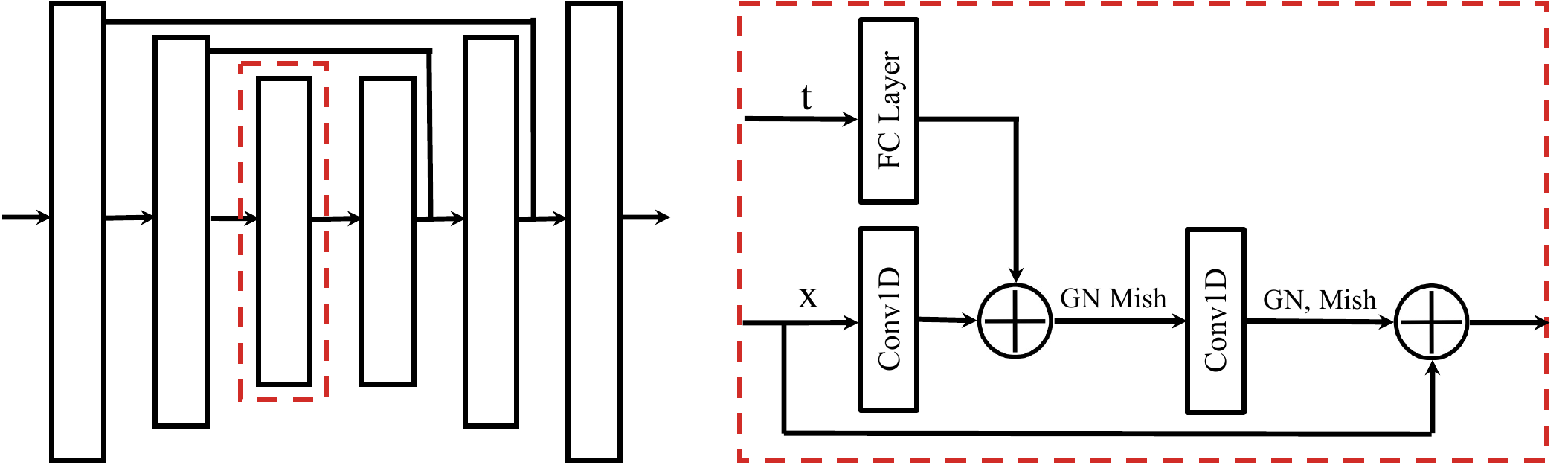}
\caption{
    \textbf{(Diffuser U-Net architecture)}
    \method has a U-Net architecture with residual blocks consisting of temporal convolutions, group normalization, and Mish nonlinearities.
}
\label{fig:app_architecture}
\end{figure}

\section{Test-time Flexibility}
\label{app:stacking_costs}

To guide \method{} to stack blocks in specified configurations, we used two separate perturbation functions $h(\btau{})$ to specify a given stack of block A on top of block B, which we detail below.

\myparagraph{Final State Matching} To enforce a final state consisting of block A on top of block B, we trained a perturbation function $h_{\text{match}}(\btau{})$ as a per-timestep classifier determining whether a a state $\vs$ exhibits a stack of block A on top of block B. We train the classifier on the demonstration data as the diffusion model.

\myparagraph{Contact Constraint} To guide the Kuka arm to stack block A on top of block B, we construct a perturbation function $h_{\text{contact}}(\btau{}) = \sum_{i=0}^{64} -1 * \|\btau{}_{c_i} - 1\|^2$, where $\btau{}_{c_i}$ corresponds to the underlying dimension in state $\btau{}_{s_i}$ that specifies the presence or absence of contact between the Kuka arm and block A.  We apply the contact constraint between the Kuka arm and block A for the first 64 timesteps in a trajectory, corresponding to initial contact with block A in a plan.

\section{Implementation Details}

In this section we describe the architecture and record hyperparameters.

\begin{enumerate}
    \item The architecture of Diffuser (Figure~\ref{fig:app_architecture}) consists of a U-Net structure with 6 repeated residual blocks. Each block consisted of two temporal convolutions, each followed by group norm \citep{wu2018groupnorm}, and a final Mish nonlinearity \citep{misra2019mish}.
    Timestep embeddings are produced by a single fully-connected layer and added to the activations of the first temporal convolution within each block.
    \item We train the model using the Adam optimizer \citep{ba2015adam} with a learning rate of $4\mathrm{e}{-05}$ and batch size of $32$.
    We train the models for 500k steps.
    \item The return predictor $\mathcal{J}$ has the structure of the first half of the U-Net used for the diffusion model, with a final linear layer to produce a scalar output.
    \item We use a planning horizon $T$ of 32 in all locomotion tasks, $128$ for block-stacking, $128$ in \texttt{Maze2D / Multi2D U-Maze}, 265 in \texttt{Maze2D / Multi2D Medium}, and 384 in \texttt{Maze2D / Multi2D Large}.
    \item We found that we could reduce the planning horizon for many tasks, but that the guide scale would need to be lowered (\emph{e.g.}, to 0.001 for a horizon of $4$ in the \texttt{halfcheetah} tasks) to accommodate.
    The \href{https://github.com/jannerm/diffuser/blob/34d0e93296c6d8649187e6790ee41cf0c59e3631/config/locomotion.py#L163-L178}{configuration file} in the open-source code demonstrates how to run with a modified scale and horizon.
    \item We use $N=20$ diffusion steps for locomotion tasks and $N=100$ for block-stacking.
    \item We use a guide scale of $\alpha=0.1$ for all tasks except \texttt{hopper-medium-expert}, in which we use a smaller scale of $0.0001$.
    \item We used a discount factor of $0.997$ for the return prediction $\mathcal{J}_\phi$, though found that above $\gamma=0.99$ planning was fairly insensitive to changes in discount factor.
    \item We found that control performance was not substantially affected by the choice of predicting noise $\epsilon$ versus uncorrupted data $\btau{0}$ with the diffusion model.
\end{enumerate}

\newpage
\let\clearpage\relax


\end{appendices}

\end{document}